\documentclass{article}

\usepackage{microtype}
\usepackage{subfigure}
\usepackage{booktabs} %

\usepackage{hyperref}

\usepackage[accepted]{icml2024}

\usepackage[textsize=tiny]{todonotes}
\usepackage{comment}
\usepackage{caption}
\usepackage[tbtags]{amsmath}
\usepackage{amsthm}
\usepackage{bm}
\usepackage{amssymb,mathrsfs}
\usepackage{amsfonts}
\usepackage{upgreek}
\usepackage{graphicx}
\usepackage{float}
\usepackage{wrapfig}
\usepackage{booktabs}       %
\usepackage{todonotes}
\usepackage{multirow}
\usepackage{color}
\usepackage{algorithm,algorithmic}

\usepackage{stmaryrd}

\usepackage{array}
\newcolumntype{?}{!{\vrule width 1.5pt}}

\usepackage{esvect}
\usepackage[inline]{enumitem}

\usepackage{graphicx} %
\usepackage{tikz}
\usepackage{pgfplots}
\usepackage{xcolor}
\usepackage{colortbl}
\colorlet{linkcolor}{blue!70!black}
\definecolor{pearDark}{HTML}{2980B9}
\usepackage{bbm}

\usepackage{ifthen}
\usepackage{xargs}

\usepackage{wrapfig}

\usepackage{aliascnt}
\usepackage{cleveref}
\usepackage{autonum}
\makeatletter
\newtheorem{theorem}{Theorem}
\crefname{theorem}{theorem}{Theorems}
\Crefname{Theorem}{Theorem}{Theorems}

\newtheorem*{lemma_nonumber*}{Lemma}

\newaliascnt{lemma}{theorem}
\newtheorem{lemma}[lemma]{Lemma}
\aliascntresetthe{lemma}
\crefname{lemma}{lemma}{lemmas}
\Crefname{Lemma}{Lemma}{Lemmas}

\newaliascnt{corollary}{theorem}
\newtheorem{corollary}[corollary]{Corollary}
\aliascntresetthe{corollary}
\crefname{corollary}{corollary}{corollaries}
\Crefname{Corollary}{Corollary}{Corollaries}

\newaliascnt{proposition}{theorem}
\newtheorem{proposition}[proposition]{Proposition}
\aliascntresetthe{proposition}
\crefname{proposition}{proposition}{propositions}
\Crefname{Proposition}{Proposition}{Propositions}

\newaliascnt{definition}{theorem}

\aliascntresetthe{definition}
\crefname{definition}{definition}{definitions}
\Crefname{Definition}{Definition}{Definitions}

\newaliascnt{remark}{theorem}

\aliascntresetthe{remark}
\crefname{remark}{remark}{remarks}
\Crefname{Remark}{Remark}{Remarks}

\crefname{figure}{figure}{figures}
\Crefname{Figure}{Figure}{Figures}

\newtheorem{assumption}{\textbf{A}\hspace{-3pt}} 
\Crefname{assumption}{\textbf{A}\hspace{-3pt}}{\textbf{A}\hspace{-3pt}}
\crefname{assumption}{\textbf{A}}{\textbf{A}}
\crefformat{assumption}{{\textbf{A}}#2#1#3}
\crefname{assumption}{assumption}{assumptions}
\Crefname{Assumption}{Assumption}{Assumptions}

\newtheorem{assumptionF}{\textbf{F}\hspace{-3pt}}
\crefformat{assumptionF}{{\textbf{F}}#2#1#3}

\Crefname{assumptionB}{\textbf{B}\hspace{-3pt}}{\textbf{B}\hspace{-3pt}}
\crefname{assumptionB}{\textbf{B}}{\textbf{B}}

\Crefname{assumptionC}{\textbf{C}\hspace{-3pt}}{\textbf{C}\hspace{-3pt}}
\crefname{assumptionC}{\textbf{C}}{\textbf{C}}

\Crefname{assumptionH}{\textbf{H}\hspace{-3pt}}{\textbf{H}\hspace{-3pt}}
\crefname{assumptionH}{\textbf{H}}{\textbf{H}}

\Crefname{assumptionT}{\textbf{T}\hspace{-3pt}}{\textbf{T}\hspace{-3pt}}
\crefname{assumptionT}{\textbf{T}}{\textbf{T}}

\Crefname{assumptionT}{\textbf{T}\hspace{-3pt}}{\textbf{T}\hspace{-3pt}}
\crefname{assumptionT}{\textbf{T}}{\textbf{T}}

\Crefname{assumptionL}{\textbf{L}\hspace{-3pt}}{\textbf{L}\hspace{-3pt}}
\crefname{assumptionL}{\textbf{L}}{\textbf{L}}

\Crefname{assumptionQ}{\textbf{Q}\hspace{-3pt}}{\textbf{Q}\hspace{-3pt}}
\crefname{assumptionQ}{\textbf{Q}}{\textbf{Q}}

\Crefname{assumptionAR}{\textbf{AR}\hspace{-3pt}}{\textbf{AR}\hspace{-3pt}}
\crefname{assumptionAR}{\textbf{AR}}{\textbf{AR}}

\makeatletter
\newcommand*{\centerfloat}{%
  \parindent \z@
  \leftskip \z@ \@plus 1fil \@minus \textwidth
  \rightskip\leftskip
  \parfillskip \z@skip}
\makeatother

\usetikzlibrary{spy}
\usetikzlibrary[patterns] %
\usetikzlibrary{calc,decorations.pathreplacing} %
\usetikzlibrary{shadows.blur} %
\usetikzlibrary{shapes.symbols}

\newcommand{\tcmb}[1]{\textcolor{blue}{#1}}

\newcommand{\tcmv}[1]{\textcolor{green!60!black}{#1}}
\newcommand{\tcmr}[1]{\textcolor{red}{#1}}

\def\msn{\mathsf{N}}

\def\mcbb{\mathcal{B}}  %

\def\rset{\mathbb{R}}

\def\rmd{\mathrm{d}}

\def\rmC{\mathrm{C}}

\newcommand{\abs}[1]{\left\vert #1 \right\vert}

\newcommandx{\psr}[3][3=]{\left\langle#1,#2 \right\rangle_{#3}}
\newcommandx{\normr}[2][2=]{ \left\Vert#1 \right\Vert_{#2}}
\newcommandx{\psrLigne}[3][3=]{\langle#1,#2 \rangle_{#3}}
\newcommandx{\normrLigne}[2][2=]{ \Vert#1 \Vert_{#2}}

\newcommandx{\norm}[2][1=]{\ifthenelse{\equal{#1}{}}{\left\Vert #2 \right\Vert}{\left\Vert #2 \right\Vert^{#1}}}

\newcommand\probaMarkovTilde[2][2=]
{\ifthenelse{\equal{#2}{}}{{\widetilde{\mathbb{P}}_{#1}}}{\widetilde{\mathbb{P}}_{#1}\left[ #2\right]}}

\newcommand{\PE}{\mathbb{E}} %

\def\ie{\textit{i.e.}}

\def\eqsp{\;}

\renewcommand{\iint}[2]{\{ #1,\ldots,#2\}}

\newcommand\sequence[3][2=,3=]
{\ifthenelse{\equal{#3}{}}{\ensuremath{\{ #1_{#2}\}}}{\ensuremath{\{ #1_{#2}, \eqsp #2 \in #3 \}}}}
\newcommand\sequenceD[3][2=,3=]
{\ifthenelse{\equal{#3}{}}{\ensuremath{\{ #1_{#2}\}}}{\ensuremath{( #1)_{ #2 \in #3} }}}

\newcommand\sequenceDouble[4][3=,4=]
{\ifthenelse{\equal{#3}{}}{\ensuremath{\{ (#1_{#3},#2_{#3}) \}}}{\ensuremath{\{  (#1_{#3},#2_{#3}), \eqsp #3 \in #4 \}}}}

\def\iid{i.i.d.}

\def\Idd{\mathrm{I}_d}

\def\path{\operatorname{path}}

\newcommand{\beq}{\begin{equation}}
\newcommand{\eeq}{\end{equation}}

\newcommand{\floor}[1]{\left\lfloor #1 \right\rfloor}

\def\densityGaussian{\msn}

\def\SLIPS{\texttt{SLIPS}}

\def\tY{\tilde{Y}}
\def\tU{\tilde{U}}

\def\rmq{\mathrm{q}}

\newcommand{\lsnr}{\log\operatorname{SNR}}
\newcommand{\snr}{\operatorname{SNR}}
\newcommand{\locr}{\operatorname{LocR}}
\usepackage{pifont}%

\icmltitlerunning{
Stochastic Localization via Iterative Posterior Sampling
}

\begin{document}

\twocolumn[
	\icmltitle{Stochastic Localization via Iterative Posterior Sampling
	}

	\icmlsetsymbol{equal}{*}

	\begin{icmlauthorlist}
		\icmlauthor{Louis Grenioux}{equal,x}
		\icmlauthor{Maxence Noble}{equal,x}
		\icmlauthor{Marylou Gabrié}{x}
		\icmlauthor{Alain Durmus}{x}
	\end{icmlauthorlist}

	\icmlaffiliation{x}{CMAP, CNRS, École polytechnique, Institut Polytechnique de Paris, 91120 Palaiseau, France}

	\icmlcorrespondingauthor{Louis Grenioux}{louis.grenioux@polytechnique.edu}
	\icmlcorrespondingauthor{Maxence Noble}{maxence.noble-bourillot@polytechnique.edu}

	\icmlkeywords{Stochastic localization, Generative modeling, Diffusion Models, Monte Carlo Markov Chain methods, Multi-modal Sampling}

	\vskip 0.3in
]

\printAffiliationsAndNotice{\icmlEqualContribution} %

\begin{abstract}
	Building upon score-based learning,
	new interest in stochastic localization techniques has recently emerged.
	In these models, one seeks to noise a sample from the data distribution through a stochastic process, called observation process, and progressively learns a denoiser associated to this dynamics.
	Apart from specific applications, the use of
	stochastic localization
	for the problem of
	sampling from an unnormalized target density has not been explored extensively.
	This work contributes to fill this gap.
	We consider a general stochastic localization framework and introduce an explicit class of observation processes, associated with flexible denoising
	schedules. We provide a complete methodology,
	\emph{Stochastic Localization via Iterative Posterior Sampling} ($\SLIPS$), to obtain approximate samples of this dynamics, and as a by-product,
	samples from the target distribution. %
	Our scheme is based on a Markov chain Monte Carlo estimation of the denoiser and comes with detailed practical guidelines.
	We illustrate the benefits and applicability of $\SLIPS$ on several benchmarks of multi-modal distributions, including Gaussian mixtures in increasing dimensions, Bayesian logistic regression and a high-dimensional field system from statistical-mechanics.

\end{abstract}

\section{Introduction}

We consider in this paper the problem of sampling from a probability density known up to a normalization constant. This problem finds its origin in various tasks, ranging from Bayesian statistics \cite{kroeseHandbookMonteCarlo2011} to statistical mechanics \cite{Krauth2006}, including now generative modeling \cite{turner2019metropolis, greniouxBalancedTrainingEnergyBased2023}.

Markov chain Monte Carlo (MCMC) samplers are among the most common approaches for this task with a wide span of applicability. In addition, under appropriate conditions on the target, theoretical guarantees can be derived \cite{dalalyan2017theoretical,durmus2017nonasymptotic}.
However, for complex distributions, simple MCMC algorithms have some limitations.
As a workaround, it has been suggested to solve the problem progressively, targeting intermediate smoothed distributions. The resulting algorithms include Replica Exchange \cite{swendsenReplicaMonteCarlo1986}, Annealed Importance Sampling \cite{neal2001annealed} and Sequential Monte Carlo \cite{del2006sequential}. Nevertheless, these methods can still largely struggle in high-dimensional settings. To address this issue, approximate inference methods such as Variational Inference (VI) \cite{Wainwright2008} have emerged, in connection with deep generative models such as Variational Auto-Encoders \cite{Rezende2014,KW2014} or Normalizing Flows \cite{rezende2015variational}.

In contrast to ``density-based sampling", which is the setting considered in the present work, generative modeling assumes the availability of training samples and aims to generate similar realizations. Henceforth, we refer to generative modeling as ``data-based sampling" as opposed to ``density-based sampling".
Diffusion-based generative models \cite{sohl2015deep,ho2020denoising,song2020score} now constitute the state-of-the-art of data-based sampling. In these models, noise is progressively added to the training samples via a forward stochastic process. The derivatives of the logarithm of the forward marginal densities, called \emph{scores}, are learned via score matching techniques \cite{hyvarinen2005estimation, vincent2011connection}, and new samples are obtained by simulating the backward process using the learned scores. This approach to generative modeling has been found to scale well with dimension in a large variety of applications \cite{rombach2022high, chen2020wavegrad, nichol2021glide} and comes with theoretical guarantees \cite{chen2022sampling}.

Extending diffusion-based model techniques to density-based sampling proves challenging as it requires an efficient estimator of the score without samples of the target. Relying on the link made with stochastic optimal control \cite{tzen2019theoretical, de2021diffusion, holdijk2022path, pavon2022local}, several VI methods using deep neural networks have recently been proposed for this task \cite{zhang2021path,berner2022optimal, vargas2023denoising, vargas2023expressiveness,vargas2023transport, richter2023improved}.
On the other hand, \cite{huang2023monte} take a non-parametric approach and proposed a scheme based on a MCMC estimation of the scores and therefore can potentially mitigate the numerical bias intrinsic in using neural networks or more broadly any parametric estimation method.
This is also our approach.

Closely related to denoising diffusion models, Stochastic Localization (SL) techniques have been first employed as a tool to establish results in geometric measure theory \cite{eldan2013thin, eldan2020taming, eldan2022analysis, chen2022localization}, and more recently have been proposed as a density-based sampling approach. \cite{alaoui2022sampling, montanari2023posterior} demonstrate the potential of SL for specific challenging distributions and \cite{ghioSamplingFlowsDiffusion2023} provide a conjecture on the type of distributions which can be efficiently sampled with these techniques.
\cite{montanari2023sampling}
provides a blueprint on using SL for density-based sampling but does not discuss a practical strategy for an arbitrary target distribution.

\paragraph{Contributions.}
Building on these prior works, we bring the following contributions.
\begin{itemize}
	\vspace{-0.1cm}
	\item We investigate a general framework of SL that allows to reflect on the role of the signal-to-noise scheduling when using SL for sampling.
	      \vspace{-0.1cm}
	\item We identify the challenges in its practical implementation and propose a learning-free sampling methodology using MCMC estimation, which requires few tuning. We elucidate our algorithmic design with theoretical and numerical considerations applicable to a certain class of non log-concave distributions.
	      \vspace{-0.1cm}
	\item We provide numerical evidence of the robustness of the proposed approach in large dimension, beyond the class of distributions amenable to theoretical guarantees. Those results show that the proposed algorithm is on par or superior to modern sampling methods in a wide variety of settings.
\end{itemize}
\vspace{-0.1cm}
The code to reproduce our experiments is available at \url{https://github.com/h2o64/slips}.

\paragraph{Notation.} For any distribution $\mu$ with finite first moment, its expectation is denoted by $\mathbf{m}_\mu=\int_{\rset^d}x \rmd \mu(x)$. The density of the Gaussian distribution with mean $\mathbf{m}\in \rset^d$ and covariance $\Sigma \in\rset^{d \times d}$ is denoted by $x \mapsto \densityGaussian(x; \mathbf{m}, \Sigma)$.

\section{Background on SL for sampling} \label{sec:background}

Consider a target distribution $\pi$ defined on $(\rset^d,\mcbb(\rset^d))$, with $\mathcal{B}(\rset^d)$ Borel sets of $\rset^d$ endowed with the Euclidean norm. Using SL to sample from $\pi$ consists in identifying and simulating a stochastic process that converges almost surely to a random variable distributed as $\pi$ \cite{alaoui2022sampling, montanari2023posterior, montanari2023sampling}. A specific example is to consider, given $X\sim \pi$, the stochastic process $(Y_t)_{t \geq 0}$, called the \emph{observation process}, defined by
\begin{equation}\label{eq:def_y_t}
	Y_t = t X + \sigma W_t \eqsp,
\end{equation}
where $(W_t)_{t\geq 0}$ is a standard Brownian motion on $\rset^d$, independent from $X$, and $\sigma>0$\footnote{Note that the authors originally considered $\sigma=1$.}.
The time-rescaled $Y_t/t$ converges almost surely to $X$ as $t\to \infty$.
Moreover,
if $\pi$ admits a finite second order moment,
we can show that the $2$-Wasserstein distance between the distribution of $Y_t/t$ and $\pi$ is bounded by $\sigma \sqrt{d/t}$ \footnote{This bound will be referred to as the \emph{localization rate}.}, see \Cref{subapp:sto-loc-th}. Thus, for $T$ large, $Y_T/T$ is approximately distributed according to $\pi$.

Sampling from $(Y_t)_{t \geq 0}$ using directly \eqref{eq:def_y_t} cannot be done in practice, precisely since it requires to first sample from $\pi$. Nevertheless, this issue can be overcome under the assumption that $\pi$ admits a finite first moment. Indeed, in this case, $(Y_t)_{t\geq0}$ solves the Stochastic Differential Equation (SDE),
\begin{equation}\label{eq:sde_sto_loc_with_denoiser}
	\rmd Y_t = u_t(Y_t) \rmd t + \sigma \rmd B_t \eqsp ,  \eqsp Y_0 = 0 \eqsp,
\end{equation}
where $(B_t)_{t\geq 0}$ is a standard Brownian motion on $\rset^d$ and $u_t(y)=\int_{\rset^d}x q_t(x|y) \rmd x$ for any $y\in \rset^d$, see \citep[Corollary 3.7.]{Brunick2013mimicking}.
The drift function $u_t$ involves the conditional density of $X$ given $Y_t=y \in \rset^d$ defined for $x \in \rset^d$ as
\begin{align}\label{eq:standard_posterior}
	 & q_0(x|y) \propto\pi(x) \eqsp ,                                                                             \\
	 & \textstyle q_t (x|y) \propto \pi(x) \densityGaussian(x; y/t, \sigma^2/t \, \Idd) \eqsp , \quad  t>0 \eqsp.
\end{align}
In the literature, the \emph{random} conditional expectation $u_t(Y_t)=\PE[X|Y_t]$ is often referred to as
the \emph{Bayes estimator} of $X$ given $Y_t$ \cite{robbins1992empirical,saremi2019neural} or the \emph{optimal denoiser} of $Y_t$ \cite{montanari2023sampling, ghioSamplingFlowsDiffusion2023}. Moreover, if $\pi$ has finite second moment, one can show that the SDE \eqref{eq:sde_sto_loc_with_denoiser} admits unique weak solutions \citep[Theorem 7.6. \& Remark 7.2.7.]{liptser1977statistics}, see \Cref{subapp:sto-loc-th}.
As a result, one can simulate $(Y_t)_{t\in [0,T]}$ for any $T>0$ by integrating the SDE \eqref{eq:sde_sto_loc_with_denoiser}, while avoiding to first sample from $\pi$, if the drift function $(u_t)_{t\in[0,T]}$ (or an approximation of it) is known.

Yet, the challenge of this sampling approach lies in the estimation of $(u_t)_{t\in[0,T]}$ without samples from $\pi$ available \textit{a priori}. At time $t=0$, we have $u_0(y)=\mathbf{m}_\pi$. For any $t>0$, $u_t$ is linked to the \emph{score}
of the marginal distribution of $Y_t$, given by $\textstyle p_t(y)=\int_{\rset^d} \densityGaussian(y; tx, \sigma^2 t \, \Idd) \rmd \pi(x)$, following Tweedie's formula given in \Cref{app:preliminaries},
\begin{align}\label{eq:link_score}
	\textstyle u_t(y)=y/t+ \sigma^2 \nabla \log p_t(y) \eqsp.
\end{align}
Therefore, sampling via SL can be reduced to a score estimation task, as noted by \cite{montanari2023sampling}, who also establishes a direct connection between the SL induced by the process \eqref{eq:def_y_t} and variance-preserving diffusions \cite{song2020score}, see \Cref{app:details_sl}.
In data-based sampling, \ie, when samples $\{X^i\}_{i=1}^N$ from $\pi$ are at the disposal of the user, $u_t$ can be estimated using a parameterized model through a least-squares regression \cite{saremi2019neural, montanari2023sampling} following score matching techniques.

A few recent works started to discuss the density-based sampling task using SL or diffusions %
\cite{montanari2023posterior, huang2023monte, vargas2023denoising}. Here, we focus on the SL formalism and propose in \Cref{sec:algos} a sampling algorithm for arbitrary distributions based on a learning-free estimator of the denoiser. Before, we provide in \Cref{sec:gen-sto-loc} a working definition of SL with flexible denoising schedules.

\section{SL with flexible denoising scheduling} \label{sec:gen-sto-loc}
Inspired by recent developments studying optimal noise scheduling
in diffusion models \cite{nichol2021improved, kingma2021variational, karras2022elucidating, kingma2023understanding}, we consider a new class of explicit observation processes associated with flexible denoising schedules.

\subsection{A more general observation process}\label{subsec:gen-framework}

For a possibly finite time horizon $T_{\text{gen}} \in (0, \infty]$, we now consider the
observation process $(Y^\alpha_t)_{t \in [0,T_{\text{gen}})}$ defined by
\begin{align}\label{eq:def_y_t-gen}
	Y_t^\alpha = \alpha(t) X + \sigma W_t \eqsp ,
\end{align}
where $(W_t)_{t\geq 0}$ is a standard Brownian motion on $\rset^d$, independent from $X\sim \pi$, and $\alpha(t)=t^{1/2}g(t)$ with the following assumptions on $g$:
\begin{enumerate}[wide, labelindent=0pt, label=(\alph*)]
	\vspace{-0.3cm}
	\item \label{item:gsl-fun} $g \in \rmC^0([0,T_{\text{gen}}),\rset_+) \cap \rmC^1((0,T_{\text{gen}}),\rset_+)$;
	      \vspace{-0.1cm}
	\item \label{item:gsl-0} $g(t) \sim C t^{\beta/2}$ as $t \to 0$, for some $\beta\geq 1$, $C>0$;
	      \vspace{-0.1cm}
	\item \label{item:gsl-T} $g$ is strictly increasing and $\lim_{t \to T_{\text{gen}}} g(t) = \infty$.
	      \vspace{-0.1cm}
\end{enumerate}
Under these assumptions, the observation process \eqref{eq:def_y_t-gen} verifies $Y^\alpha_0=0$ and $Y^\alpha_t/\alpha(t) \to X$ almost surely as $t\to T_{\text{gen}}$.
Moreover, by denoting $\pi^\alpha_t$ the distribution of $Y^\alpha_t/\alpha(t)$, one can show that for any $t\in (0,T_{\text{gen}})$, the localization rate is now function of $g(t)$:
\begin{align} \label{prop:standard_gen_sto_loc_cv}
	\textstyle W_2(\pi, \pi^\alpha_t) \leq
	\sigma \sqrt{d}/g(t) \eqsp,
\end{align}
where $W_2$ denotes the $2$-Wasserstein distance.
Proofs and refinements in the case where $\pi$  is a Gaussian distribution are postponed to \Cref{subapp:sto-loc-th}.
We recover the process \eqref{eq:def_y_t} discussed in \Cref{sec:background} with $g(t) = t^{1/2}$ and $T_{\text{gen}}=\infty$.

As in the previous section, sampling from the observation process via \eqref{eq:def_y_t-gen} requires to first sample from $\pi$. We again bypass this issue by introducing an equivalent SDE. In \Cref{subapp:sto-loc-th}, we indeed show that $(Y^\alpha_t)_{t \in [0,T_{\text{gen}})}$ solves
\begin{align}\label{eq:SDE-gen}
	\rmd Y^\alpha_t = \dot{\alpha}(t) u^\alpha_t(Y^\alpha_t) \rmd t + \sigma \rmd B_t \eqsp, \eqsp Y^\alpha_0 = 0 \eqsp ,
\end{align}
where $(B_t)_{t\geq 0}$ is a standard Brownian motion on $\rset^d$ and the drift function $u_t^\alpha$ is defined for any $y\in \rset^d$ by
\begin{equation}
	\textstyle u^\alpha_t(y)=\int_{\rset^d}x q^\alpha_t(x|y) \rmd x \eqsp .
	\label{eq:def_u_alpha}
\end{equation}
Here, the conditional density of $X$ given
$Y_t=y$ is
\begin{align}\label{eq:general_posterior}
	 & q^\alpha_0(x|y) \propto\pi(x) \eqsp ,                                                                                                            \\
	 & \textstyle q_t^\alpha (x|y) \propto \pi(x) \densityGaussian(x; y/\alpha(t), \sigma^2/g(t)^2\, \Idd) \eqsp, \eqsp t\in (0,T_{\text{gen}}) \eqsp .
\end{align}
Assuming further that $\pi$ has finite second moment and that $g$ satisfies technical conditions, \eqref{eq:SDE-gen} admits unique weak solutions, see \Cref{subapp:sto-loc-th}. Then again, if $u_t^\alpha$ is known, integrating \eqref{eq:SDE-gen} up to time $T\in (0,T_{\text{gen}})$ is equivalent to sampling from $(Y^\alpha_t)_{t \in [0,T]}$. Then, $Y^\alpha_T/\alpha(T)$ is approximately distributed according to $\pi$ for any $T$ close enough to $T_{\text{gen}}$.

The challenge still lies in the estimation of the drift function $u^\alpha_t$ which verifies $u^\alpha_0(y)=\mathbf{m}_\pi$ at time $t=0$, and which can be expressed, for any $t\in (0,T_{\text{gen}})$, using the score of the marginal distribution $p^\alpha_t(y)$ given by $\textstyle p^\alpha_t(y)=\int_{\rset^d}\densityGaussian(y;\alpha(t)x, \sigma^2 t\, \Idd)  \rmd \pi(x)$. In \Cref{subapp:sto-loc-th}, we indeed show that
\begin{align}\label{eq:link_score_general}
	\textstyle u_t^\alpha(y)=y/\alpha(t)+ \{\sigma^2 t/\alpha(t)\} \nabla \log p_t^\alpha(y) \eqsp ,
\end{align}
recovering the previously mentioned link between score estimation and denoiser estimation.

Moreover, the score of $p_t^\alpha$ can also be written as an expectation over the same posterior distribution $q_t^\alpha$ but with a different observable
\begin{align} \label{eq:target_sm}
	\textstyle \nabla \log p_t^\alpha(y)= \frac{1}{\alpha(t)}\int_{\rset^d} \nabla \log \pi(x) q_t^\alpha (x|y) \rmd x  \eqsp .
\end{align}
Hence, this suggests to consider an alternative expression of the SDE defined in \eqref{eq:SDE-gen}, involving the expectation given in \eqref{eq:target_sm} instead of the denoiser function $u_t^\alpha$ \footnote{Equation \eqref{eq:target_sm} and its consequence have been independently derived in the concurrent works \cite{debortoli2024target, akhoundsadegh2024iterated}. Both have been publicly available almost at the same time than the present paper.}. We refer to \Cref{subsec:alternative_SLIPS} for more details. However, we experimentally observed this approach was less stable than our original method, which is why we decided to consider the SDE \eqref{eq:SDE-gen}.

\begin{table}[t!]
	\caption{Stochastic Localization schemes.}
	\label{table:sto-loc-gen}
	\vspace{-0.3cm}
	\begin{center}
		\begin{small}
			\begin{tabular}{lccr}
				\toprule
				SL scheme                  & Hyper-parameters              & $g(t)$                              \\
				\midrule
				Geom-$\infty(\alpha_1)$    & $\alpha_1 \geq 1$             & $t^{\alpha_1/2}$                    \\
				Geom$(\alpha_1, \alpha_2)$ & $\alpha_1\geq1$, $\alpha_2>0$ & $t^{\alpha_1/2}(1-t)^{-\alpha_2/2}$ \\
				\bottomrule
			\end{tabular}
		\end{small}
	\end{center}
	\vspace{-0.5cm}
\end{table}

\begin{figure}[h!]
	\centering
	\includegraphics[width=\linewidth]{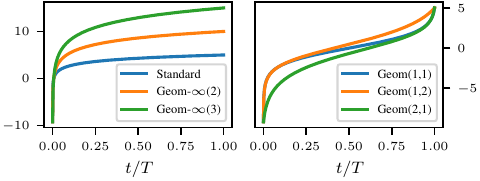}
	\vspace{-0.6cm}
	\caption{Display of $\lsnr$ for the localization schemes Geom-$\infty$ (\textit{left}) and Geom (\textit{right}).}
	\label{fig:log-SNR}
	\vspace{-0.4cm}
\end{figure}

\subsection{Hyper-parameter selection and SNR profile}\label{subsec:guideline}

At first sight, determining appropriate choices for $g$, $\sigma$ and $T_{\text{gen}}$, without introducing further complexity, can be challenging. Below, we propose interpretable choices, relying on the analysis of the \emph{signal-to-noise ratio} of the observation process, which determines the speed at which the denoising process localizes on the target distribution.

An intuitive definition of the Signal-to-Noise Ratio (SNR) of the observation process \eqref{eq:def_y_t-gen} is, for any $t\in [0,T_{\text{gen}})$,
\begin{align}\label{eq:SNR}
	\snr (t) =  \frac{\mathbb{E}\left[\norm{\alpha(t) (X - \mathbb{E}[X])}^2\right]}{\mathbb{E}\left[\norm{\sigma W_t}^2\right]}= \frac{R_\pi^2 g(t)^2}{\sigma^2 d} \eqsp,
\end{align}
where $\textstyle R_\pi^2=\int_{\rset^d} \norm{x -\mathbf{m}_\pi}^2\rmd \pi(x)$ is the scalar variance of $\pi$.
The monotonicity of $g$ ensures that the SNR is strictly increasing on $[0,T_{\text{gen}})$, with $\lim_{t \to 0 } \snr(t) = 0$ (no information on $\pi$ at all) and $\lim_{t \to T_{\text{gen}} } \snr(t) = \infty$ (no noise at all). Further, the scaling in time of the SNR schedule is solely determined by the function $g$, which also controls the time dependence in the localization rate \eqref{prop:standard_gen_sto_loc_cv}. This observation leads us to fix $\sigma=R_\pi/\sqrt{d}$, assuming that $R_\pi$ is known, to obtain a SNR profile
$\snr (t)=g^2(t)$
independent of the target distribution or the dimension\footnote{Note that our framework would be unchanged by taking $\sigma=1$ and defining $g$ up to a multiplicative constant.}. In cases where the scalar variance of the target $\pi$ is unknown, an estimator $\hat{R}_\pi$ can be used instead.

The practical implementation of SL-based sampling requires also the choice of an effective time horizon $T<T_{\text{gen}}$ up to which we integrate the SDE \eqref{eq:SDE-gen}, similarly to the maximum denoising time in diffusion models. It needs to be close enough to $T_{\text{gen}}$ to consider $Y^\alpha_T/\alpha(T)$ as approximately distributed according to $\pi$ and as small as possible to ensure computational efficiency. This trade-off is accounted for by selecting $T$ based on reaching a predefined level of the logarithm of the SNR (log-SNR), which is preferable to use in practice.
Denoting by $\eta$ a positive log-SNR threshold, we set $T=T_\eta$, where $T_\eta$ is such that $\lsnr(T_{\eta})= 2 \log g(T_{\eta}) = \eta$.
In practice, the choice of $\eta$ does not have a significant impact on performance (see \Cref{subsec:gaussian_init} \Cref{fig:app:impact_t_0}).

\begin{figure}[h!]
	\centering
	\includegraphics[width=\linewidth]{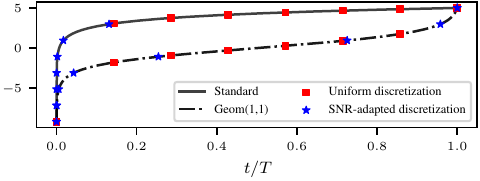}
	\vspace{-0.6cm}
	\caption{SNR-adapted vs uniform time discretization for the schemes Standard and Geom$(1,1)$. The uniform discretization leads to larger log-SNR differences between timesteps where the SNR increases rapidly.}
	\label{fig:adap_dist}
	\vspace{-0.4cm}
\end{figure}

Having reduced the set of hyper-parameters to tune, we investigate two main settings for the denoising schedule $g(t)$. Their main characteristics are summarized in \Cref{table:sto-loc-gen} and examples of corresponding log-SNR profiles are plotted in \Cref{fig:log-SNR}.
\begin{enumerate}[wide, labelindent=0pt, label=(\alph*)]
	\vspace{-0.3cm}
	\item \textbf{Asymptotic geometric schedule (Geom-$\infty$)}, with $T_{\text{gen}}=\infty$ and $g(t)=t^{\alpha_1/2}$ for $\alpha_1 \geq 1$. This is a natural extension of the observation process discussed in \Cref{sec:background} for which $\alpha_1=1$. We refer to it as the \emph{standard} scheme.
	      \vspace{-0.1cm}
	\item \textbf{Non-asymptotic geometric schedule (Geom)}, with $T_{\text{gen}}<\infty$ and $g(t)=(t/T_{\text{gen}})^{\alpha_1/2}(1-t/T_{\text{gen}})^{-\alpha_2/2}$, for $\alpha_1\geq1$ and $\alpha_2>0$. In practice, we only consider $T_{\text{gen}}=1$. If $\alpha_1=\alpha_2$, the log-SNR profile is similar to the one obtained in diffusion models via cosine noise scheduling \cite{nichol2021improved}. Moreover, by taking $\alpha_2 <\alpha_1$, the profile becomes flatter near $t=1$, similarly to the scheduling profiles presented in \cite{kingma2021variational, kingma2023understanding}. Following the framework of \cite{albergo2023stochastic}, $\{(1-t)^{\alpha_2/2}Y_t\}_{t\in [0,1]}$ is a stochastic interpolant between the Dirac mass at $0$ and $\pi$.
\end{enumerate}

\vspace{-0.1cm}
The possibility to adapt the denoising schedule has not yet received attention in the application of SL for sampling. In the following section, we detail the design of our algorithm which adapts to any denoising  schedule.

\vspace{-0.1cm}
\section{$\SLIPS$ sampling algorithm}\label{sec:algos}

In this section, we consider hyper-parameters $\sigma$, $T$ and a denoising schedule $g$ as given in \Cref{subsec:guideline} and expose our strategy for SL-based sampling. As already mentioned, sampling from $\pi$ via SL consists in integrating the SDE \eqref{eq:SDE-gen} to obtain a realization $\{Y_t^\alpha\}_{t\in [0,T]}$ and set $Y_T^\alpha/\alpha(T)$ as an approximate sample from $\pi$. However, the SDE integration faces two main hindrances. First, the drift function $(u_t^{\alpha})_{t \in [0,T]}$ defined in \eqref{eq:def_u_alpha} is not tractable in general.
Second, even if the drift function were known exactly, numerical integration incurs unavoidable discretization errors.
We first tackle the latter issue in \Cref{subsec:discretization}.
We then present a Monte Carlo-based approach to tackle the estimation of $u_t^\alpha$ in \Cref{subsec:estimation}. Finally, we present the key idea of our algorithm $\SLIPS$ in \Cref{subsection:duality}.

\subsection{Handling the discretization error}
\label{subsec:discretization}

Consider a time discretization of the interval $[0,T]$ defined by an increasing sequence of timesteps $(t_k)_{k=0}^K$ where $t_0\geq 0$, $t_K=T$ and $K\geq 1$. We define a sequence of random variables $\{\tY^\alpha_{t_k}\}_{k=0}^{K}$ approximating \eqref{eq:SDE-gen}, defined for any $k \in \iint{0}{K-1}$ by the recursion
\begin{align}\label{eq:recursion-EM}
	\tY^\alpha_{t_{k+1}} = \tY^\alpha_{t_{k}} + w_k u^\alpha_{t_k}(\tY^\alpha_{t_k}) + \sigma \sqrt{\delta_k} Z_{k+1} \eqsp ,
\end{align}
where $\tY^\alpha_{t_0}=0$, $\delta_k=t_{k+1}-t_k$, $w_k= \alpha(t_{k+1})-\alpha(t_{k}) $ and  $(Z_{k})_{k=1}^K$ is distributed according to the centered standard Gaussian distribution. Note that \eqref{eq:recursion-EM} results from an \emph{Euler Maruyama} (EM) discretization scheme applied to \eqref{eq:SDE-gen}, that corresponds for any $k \in \iint{0}{K-1}$ to the exact integration of the SDE
\begin{align}\label{eq:SDE-gen-disc}
	\rmd \tY^\alpha_t = \dot{\alpha}(t) u^\alpha_{t_k}(\tY^\alpha_{t_k}) \rmd t + \sigma \rmd B_t \eqsp, \eqsp t \in[t_k, t_{k+1}] \eqsp.
\end{align}

In practice, we rather start the integration from a timestep $t_0>0$ for reasons that we detail in the next section.
We also adopt a time discretization adapted to the variations of the log-SNR on $[t_0,T]$. Define $\Delta_{\text{SNR}} = \{\lsnr(T)-\lsnr(t_0)\}/K$, and for any $k\in\{0, \hdots, K\}$, choose $t_k$ such that
\begin{align}\label{eq:SNR-disc}
	\lsnr(t_k) = \lsnr(t_0) + \Delta_{\text{SNR}} k \eqsp.
\end{align}
This discretization defines smaller step-sizes when the variation of $\lsnr$ is high, \ie, where the denoising process ``accelerates'' (see \Cref{fig:adap_dist}). We observed that using uniform discretization leads to significantly greater numerical errors in our experiments, see \Cref{app:perfect_score}.

As described in the next section, we suggest to estimate $u^\alpha_{t_k}(\tY^\alpha_{t_k})$ by Monte Carlo methods. Denoting by $\texttt{MC-Est}(u^\alpha_{t}(Y))$ an estimator of the drift function at time $t$ evaluated at $Y$,
we come to consider the joint sequence $\{(\tU^\alpha_{t_k},\tY^\alpha_{t_k})\}_{k=0}^K$ given by
\begin{align}\label{eq:recursion-EM-est}
	 & \tU^\alpha_{t_k} = \texttt{MC-Est}(u^\alpha_{t_k}(\tY^\alpha_{t_k})) \eqsp ,                             \\
	 & \tY^\alpha_{t_{k+1}}= \tY^\alpha_{t_{k}} + w_k \tU^\alpha_{t_k} + \sigma \sqrt{\delta_k} Z_{k+1} \eqsp ,
\end{align}
where
$(t_k)_{k=0}^K$ is the SNR-adapted time discretization of $[t_0,T]$ defined in \eqref{eq:SNR-disc}.
The initialization of $\tY^\alpha_{t_0}$ is crucial and will be discussed in \Cref{subsection:duality}.
We now detail our Monte Carlo estimation scheme defining $\tU^\alpha_{t_k}$.

\vspace{0.1cm}
\subsection{Estimating the denoiser via MCMC} \label{subsec:estimation}

\paragraph{Monte Carlo estimation with posterior sampling.}
We aim to build accurate estimators of the quantities
\begin{align}
	\textstyle u^\alpha_{t_k}(\tY^\alpha_{t_k})=\int_{\rset^d}x \, q_{t_k}^\alpha(x|\tY^\alpha_{t_k})\rmd x \eqsp.
\end{align}
For ease of notation, we will denote the \emph{random} posterior density $q_{t_k}(\cdot|\tY^\alpha_{t_k})$ by $\mu_k$ in this section and
recall that for any $k\in \{0, \hdots, K\}$
\begin{align} \label{eq:mu_k}
	\textstyle \mu_k(x) \propto \pi(x) \densityGaussian(x; \tY^\alpha_{t_k}/\alpha(t_k), \sigma^2/g(t_k)^2\, \Idd)  \eqsp.
\end{align}
One possible approach is to estimate $\tU^\alpha_{t_k}$ with Importance Sampling (IS), by choosing as instrumental distribution $\densityGaussian(\tY^\alpha_{t_k}/\alpha(t_k),  \sigma^2/g(t_k)^2 \, \Idd)$. %
However, IS is known to suffer from large variance in high-dimension unless the proposal is closely adapted to $\pi$ \cite{ceriotti2012inefficiency}. This variance issue translates into sample-size requirements which may grow exponentially with $d$, see e.g., \cite{chatterjee2018sample}. %
\vspace{0.1cm}

Instead, we propose to approximately sample from $\mu_k$ with a MCMC algorithm, in the same fashion as \cite{huang2023monte}, and compute $\tU^\alpha_{t_k}$ as the empirical mean of the obtained samples. More precisely, given a   Markov chain $\{X^{j}_k\}_{j=1}^{M}$  targeting $\mu_k$, we define $\textstyle\tU^\alpha_{t_k} = (1/M)\sum_{j=1}^M  X^{j}_k$.

In practice, we turn to the largely used \emph{Metropolis-Adjusted Langevin Algorithm} (MALA) \cite{roberts1996exponential} leveraging gradient-information of the target log-density. Nonetheless, to justify the use of MCMC here, sampling from $\mu_k$  has to be easier than the original problem of sampling from $\pi$. In particular, we establish in the next paragraph, conditions ensuring that $\mu_k$ is log-concave, implying a reasonable mixing time for MALA \cite{dwivedi2019log}. %

\begin{figure*}[h!]
	\centering
	\includegraphics[width=0.49\linewidth]{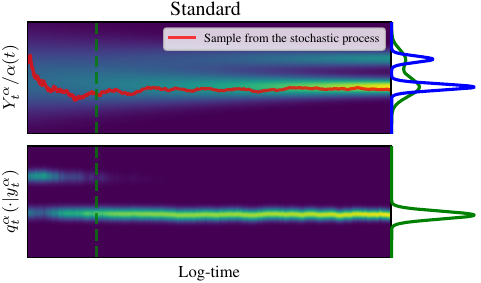} \hfill
	\includegraphics[width=0.49\linewidth]{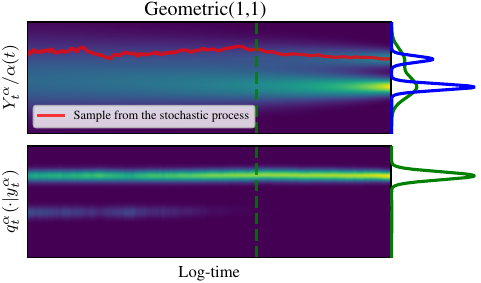}
	\vspace{-0.2cm}
	\caption{\textbf{Duality of log-concavity}: distribution of $Y^{\alpha}_t / \alpha(t)$ (\textit{up}) and $q^{\alpha}_t(\cdot | y^{\alpha}_t)$ (\textit{down}) for $t\in(0, T_{\text{gen}})$, where $y^{\alpha}_t$ is a realisation of the observation process (\tcmr{red} line), for the standard scheme (\textit{left}) and the Geom$(1,1)$ scheme (\textit{right}). The target distribution $\pi$ is a mixture of two 1D-Gaussian distributions $\densityGaussian(-2/3, (0.05)^2)$ and $\densityGaussian(4/3, (0.05)^2)$ with respective weights $2/3$ and $1/3$, which density is given by the \tcmb{blue} line.
		The heat map represents the likelihood of the distributions and the \tcmv{green} line on the right edge stands for the distributions taken at the time given by the dotted green line.}
	\label{fig:fig1}
	\vspace{-0.4cm}
\end{figure*}

\paragraph{Main issue in posterior sampling.}
To illustrate the main challenge in the sampling of $\mu_k$, we focus on a particular setting. In the same spirit as \citep[Theorem 1]{saremi2023chain}, we consider the following assumption on $\pi$.
\vspace{0.2cm}
\begin{assumption}[Log-concavity outside a compact]
	\label{ass:target} There exist $R>0$ and $\tau>0$ such that $\pi$ is the convolution of $\mu$ and $\mathsf{N}(0,\tau^2\,\Idd)$, where $\mu$ is a distribution whose scalar variance is bounded by $dR^2$, \ie, $\int_{\rset^d}\norm{x-\mathbf{m}_\pi}^2\rmd \mu(x) \leq dR^2$.
\end{assumption}
This assumption can be equivalently formulated as follows: for any random vector $X \sim \pi$, we have $X= U + G$, where $\mathbb{E}[\norm{U-\mathbf{m}_\pi}^2]\leq dR^2$ and $G\sim \densityGaussian(0, \tau^2 \,\Idd)$ is independent of $U$. In this setting, the scalar variance of $\pi$, defined in \Cref{subsec:guideline}, verifies $R_\pi^2= \mathbb{E}[\norm{U-\mathbf{m}_\pi}^2] + d\tau^2$; therefore, we have the upper bound $R_\pi^2 \leq d(R^2 + \tau^2)$. In particular, this setting includes non log-concave distributions such as Gaussian mixtures, see \Cref{app:gaussian_mixture} for more details.
Under Assumption \Cref{ass:target}, we obtain the following result whose formal statement is given in \Cref{app:theory}.
\vspace{0.2cm}
\begin{theorem}
	\label{th:q-log-concave}
	Assume \Cref{ass:target}. There exists $t_{\mathrm{q}} >0$ (explicit in \Cref{app:theory}), depending on $g,d,R$ and $\tau$ such that if $t_k > t_{\mathrm{q}}$, $\mu_k$ is strongly log-concave. In addition, $\mu_k$ is more log-concave as $k$ increases.
\end{theorem}
This result can be explained by the fact that, for large $t_k$, the Gaussian term dominates the target term in the posterior \eqref{eq:mu_k}, since
$\sigma^2/g(t_k)^2 \approx 0$ - which intuitively ensures increasing log-concavity. On the other hand, if $t_k$ is close to 0, then $\mu_k$ is expected to be close to $\pi$, recalling that $q_0^\alpha=\pi$. In this case, for non log-concave target distributions satisfying \Cref{ass:target}, $\mu_k$ is non log-concave for small $t_k$. We illustrate this behaviour in \Cref{fig:fig1} for a bimodal Gaussian mixture: as depicted by the bottom row, the posterior distribution starts multi-modal (non log-concave) and eventually becomes unimodal (log-concave).

Therefore, for a wide variety of target distributions, our MALA-based posterior sampling approach on the time interval $[t_0, T]$ will fail at the very first steps if $t_0$ is chosen too small. We now explain how to bypass this issue in $\SLIPS$.

\subsection{Duality of log-concavity is all you need}
\label{subsection:duality}

\paragraph{Initialisation of the integration.}
Following the previous discussion, a reliable computation of our MCMC-based estimator given $\tU_{t_k}^\alpha$ requires to have
$t_k > t_{\mathrm{q}}$ in order to ensure log-concavity of the random posterior $\mu_k$.

Thus, we aim to start the integration from a $t_0 > t_{\rmq}$. To initialize the recursion \eqref{eq:recursion-EM-est} from $t_0$, one needs to sample the first iterate $\tY^{\alpha}_{t_0}$ distributed according to $p^\alpha_{t_0}$.
Given that $\tU_{t_0}^\alpha$ is a reliable estimator of the denoiser since $t_0 > t_{\rmq}$, we can derive an approximation of the score $\nabla \log p^{\alpha}_{t_0}$, via Tweedie's formula \eqref{eq:link_score_general}. We use this estimate to sample from $p^{\alpha}_{t_0}$ via  the \emph{Unadjusted Langevin Algorithm} (ULA) \footnote{We do not use MALA since only the score of $p^{\alpha}_{t_0}$ is available.} \cite{roberts1996exponential}. In this initialization procedure, the successive evaluations of the score $\nabla \log p^{\alpha}_{t_0}$ in ULA can be computed by applying an inner loop of MALA on $q_{t_0}^\alpha$. This results in a Langevin-within-Langevin scheme, which is presented in \Cref{alg:langevin_within_langevin}, and detailed in \Cref{app:slips}.

For efficient sampling, ULA however requires a condition of log-concavity on $p^{\alpha}_{t_0}$ \cite{dalalyan2017theoretical,durmus2017nonasymptotic}. In \Cref{th:p-log-concave}, we prove that this property is actually ensured for small values of $t_0$ under Assumption \Cref{ass:target}.

\begin{theorem}
	\label{th:p-log-concave}
	Assume \Cref{ass:target}. There exists $t_{\mathrm{p}} >0$ (explicit in \Cref{app:theory}), depending on $g,d,R$ and $\tau$ such that if $t < t_{\mathrm{p}}$, $p^\alpha_{t}$ is strongly log-concave.
\end{theorem}
We provide a formal version of this result in \Cref{app:theory}. Intuitively, when $t$ is small, the Gaussian noise in the observation process  \eqref{eq:def_y_t-gen} overwhelms $\alpha(t)X$, which makes $p^\alpha_t$ log-concave. Note however that this log-concavity property is not verified for large $t$ as $p_t^\alpha$ becomes as log-concave as $\pi$ by the localization property, see \eqref{prop:standard_gen_sto_loc_cv}. We illustrate this behaviour in the upper row of \Cref{fig:fig1}.

\paragraph{Duality of log-concavity.}Ideally, we would like to have (i) $t_0>t_{\mathrm{q}}$ (to ensure that the estimation from $\tU_{t_k}^\alpha$ is correct) and (ii) $t_0<t_{\mathrm{p}}$ (to ensure that $p^{\alpha}_{t_0}$ is log-concave). Under an additional assumption to \Cref{ass:target}, \Cref{th:duality} demonstrates that such $t_0$ can exist. A proof of this result is given in \Cref{app:theory}.
\begin{theorem}\label{th:duality}
	Assume \Cref{ass:target}, with $dR^2/\tau^2< 2$. Then, $t_{\mathrm{q}}<t_{\mathrm{p}}$, where $t_{\mathrm{q}}$ and $t_{\mathrm{p}}$ are given in \Cref{th:q-log-concave} and \Cref{th:p-log-concave}.
\end{theorem}
Note however that this extra assumption is restrictive and could be slightly improved, as done for Gaussian mixtures in \Cref{app:theory}. %
Then, choosing $t_0$ boils down to finding a sweet spot where $p^{\alpha}_{t_0}$ and $q_{t_0}^\alpha$ are both approximately log-concave. Together, these requirements form what we call the \emph{`duality of log-concavity'}. In \Cref{fig:fig1}, we show that $t_0$ can be found for two different localization schemes when considering a Gaussian mixture that does not fit the assumption made in \Cref{th:duality}. Furthermore, we highlight in \Cref{app:mcmc_ablation_study} that a sweet spot generally exists for a wider variety of target distributions independently of the localization scheme.

\vspace{0.3cm}
In practice, we ensure the log-concavity of $q_{t_0}^\alpha$ first, as the Monte Carlo estimation of $u^\alpha_{t_0}$ is the starting point of the $\SLIPS$ algorithm.

\subsection{Implementation of $\SLIPS$.}

By combining all the previous observations made above, we propose the \emph{Stochastic Localization via Iterative Posterior Sampling} ($\SLIPS$) algorithm, summarized in \Cref{alg:slips}, which we detail now.

The $\SLIPS$ algorithm has five inputs :
\begin{enumerate}[wide, labelindent=0pt, label=(\alph*)]
	\vspace{-0.3cm}
	\item the noising schedule in the observation process $t \mapsto\alpha(t)$ defined on a time interval $[0, T_{\text{gen}})$,
	      \vspace{-0.1cm}
	\item the \emph{initial} integration time in \eqref{eq:SDE-gen}: $t_0\in (0, T_{\text{gen}})$,
	      \vspace{-0.1cm}
	\item the \emph{final} integration time in \eqref{eq:SDE-gen}: $T\in (t_0, T_{\text{gen}})$, or equivalently $\eta>0$ by log-SNR correspondence,
	      \vspace{-0.1cm}
	\item the number of discretization steps in \eqref{eq:SDE-gen}: $K\geq 1$,
	      \vspace{-0.1cm}
	\item the number of samples for posterior sampling $M \geq 1$.
\end{enumerate}
The algorithm can be decomposed in two main parts : \ref{item:enum_slips_a} its initialization (summarized in \Cref{alg:langevin_within_langevin}) and \ref{item:enum_slips_b} the run of the discretized SDE \eqref{eq:recursion-EM-est} via posterior sampling.
\begin{enumerate}[wide, labelindent=0pt, label=(\roman*)]
	\item \label{item:enum_slips_a} The initialization of $\SLIPS$ consists in approximately sampling from $p^{\alpha}_{t_0}$. Since its score can be expressed via the posterior distribution $q^{\alpha}_{t_0}$, see \eqref{eq:def_u_alpha} and \eqref{eq:link_score_general}, and $q^{\alpha}_{t_0}$ can itself be sampled from with MCMC methods (see \Cref{subsec:estimation}), we propose to use a Langevin-within-Langevin procedure : each step of the outer loop corresponds to an iteration of the Langevin algorithm to sample from $p^{\alpha}_{t_0}$, and requires local MALA steps to estimate the score by sampling from the posterior $q^{\alpha}_{t_0}$. The exact computations are detailed in \Cref{app:slips}. The final iterate of this algorithm is used as the initialization of the second part of $\SLIPS$.
	\item \label{item:enum_slips_b} Once the initialization of $\SLIPS$ is done, we turn to the actual core part of $\SLIPS$, which consists in running the discretized SDE given in \eqref{eq:recursion-EM-est}. At step $k+1$ of this recursion (corresponding to the time $t_{k+1}$ in the initial SDE), the denoiser term is approximated by running the MALA algorithm on the posterior distribution $q^{\alpha}_{t_k}$ conditioned on the $k$-th iterate of the recursion. Finally, the last iterate of this recursion is considered as an approximate sample from the target distribution.
\end{enumerate}

\paragraph{Complexity of $\SLIPS$.} Note that running this procedure is meant to produce one sample from the target distribution. In practice, one may want to produce several samples. In this case, $\SLIPS$ can be easily parallelised by running simultaneously independent realizations of the procedure described above. Moreover, we emphasize that the initialization of $\tY^\alpha_{t_0}$ is the most challenging step in $\SLIPS$ as posterior sampling only gets easier afterwards, by \Cref{th:q-log-concave}. The computational cost of this initialization, involving a Langevin-within-Langevin procedure, remains however reasonable thanks to a persistent initialization of the ULA and MALA chains (see \Cref{app:slips}). As such, only a few steps of each is needed in practice (see \Cref{app:tde}).
\vspace{-0.2cm}
\paragraph{Limitation of $\SLIPS$.} We highlight that $t_0$ is the only hyper-parameter in our algorithm that requires careful tuning. However, \Cref{fig:app:w2_with_t_0} and \Cref{fig:app:w2_with_t_0_sigmas} in \Cref{app:mcmc_ablation_study} highlight moderate sensitivity to this hyper-parameter.

\vspace{-0.2cm}
\begin{algorithm}[h!]
	\caption{SLIPS}
	\label{alg:slips}
	\begin{algorithmic}
		\STATE {\bfseries Input:} $\alpha$, $t_0$, $\eta$, $K$, $M$
		\STATE Set $T=T_\eta$ and $\sigma=\hat{R}_\pi/\sqrt{d}$, see \Cref{subsec:guideline}
		\STATE Set $(t_k)_{k=0}^K$ as the SNR-adapted disc. of $[t_0, T]$, see \eqref{eq:SNR-disc}
		\STATE Initialize $\tY^\alpha_{t_0}$ with \Cref{alg:langevin_within_langevin}
		\FOR{$k=0$ {\bfseries to} $K-1$}
		\STATE Define $\delta_k=t_{k+1}-t_k$, $w_k= \alpha(t_{k+1})-\alpha(t_{k})$
		\STATE Simulate $\{X^k_j\}_{j=1}^M \sim\mu_k$ with MALA, see \eqref{eq:mu_k}
		\STATE Estimate the denoiser by $\tU_{t_k}^\alpha=(1/M)\sum_{j=1}^M X_j^k$
		\STATE Simulate $\tY^\alpha_{t_{k+1}} \sim \densityGaussian(\tY^\alpha_{t_{k}} + w_k \tU_{t_k}^\alpha, \sigma^2 \delta_k\,\Idd)$
		\ENDFOR
		\STATE {\bfseries Output:} $\tY^\alpha_{t_{K}}/\alpha(t_K)$
	\end{algorithmic}
\end{algorithm}
\vspace{-0.3cm}
\begin{algorithm}[h!]
	\caption{Langevin-within-Langevin initialization}
	\label{alg:langevin_within_langevin}
	\begin{algorithmic}
		\STATE {\bfseries Input:} $\alpha(t_0)$, $t_0$, $\sigma$, $N$, $M$
		\STATE Set $Y^{(0)} \sim \densityGaussian(0, \sigma^2 t_0 \, \Idd)$ and $\lambda = \sigma^2 t_0 / 2$
		\FOR{$n = 0$ {\bfseries to} $N-1$}
		\STATE Simulate $\{X^{(n)}_j\}_{j=1}^M \sim q^{\alpha}_{t_0}(\cdot | Y^{(n)})$ with MALA
		\STATE Estimate the denoiser by $U^{(n)}= (1/M) \sum_{j=1}^M X^{(n)}_j$
		\STATE Set $\hat{s}^{\alpha}_{t_0}(Y^{(n)}) = \left(\alpha(t_0) U^{(n)} - Y^{(n)}\right) / (\sigma^2 t_0)$, see \eqref{eq:SDE-gen}
		\STATE Simulate $Y^{(n+1)} \sim \densityGaussian(Y^{(n)} + \lambda \hat{s}^{\alpha}_{t_0}(Y^{(n)}), 2 \lambda \, \Idd)$
		\ENDFOR
		\STATE {\bfseries Output:} $Y^{(N)}$
	\end{algorithmic}
\end{algorithm}

\vspace{-0.4cm}
\section{Related work}

\vspace{-0.1cm}
\paragraph{Score-based sampling with VI.} Building upon score-based generative models and VI, recent works have proposed deep-learning approaches for density-based sampling.
Those sampling schemes amount to discretized versions of denoising processes with a parameterized drift function. Two main settings may be distinguished. On one hand, the VI framework is seen as a stochastic optimal control problem involving the time-reversal of the denoising process, see e.g., \cite{zhang2021path, berner2022optimal, vargas2023bayesian,vargas2023denoising}. We refer to \cite{richter2023improved} for a global overview of this approach and its extensions. On the other hand, another line of work has proposed parameterized extensions of Annealed Importance Sampling (AIS) \cite{doucet2022score,geffner2023langevin,vargas2023transport}. Since $\SLIPS$ is learning-free, these algorithms are not included in our numerical tests as it would be difficult to draw comparisons at equal computational budget.

\paragraph{Score-based sampling with Monte Carlo.} Preliminary study of Monte Carlo score estimation was done by \cite{huang2021schrodinger}, followed by \cite{vargas2023bayesian}, who considered the Föllmer diffusion \cite{follmer2005entropy,follmer2006time} bridging $\updelta_0$ to $\pi$ in a finite-time setting. However, their method relying on IS does not scale well with dimension, and suffers from numerical unstability. Note that it can be recasted as a particular example of our scheme Geom(1,1), via the stochastic interpolant analogy, see \Cref{subsec:guideline}.

Closely related to this work, \cite{huang2023monte} recently proposed \emph{Reverse Diffusion Monte Carlo} (RDMC), a sampling algorithm based on a Monte Carlo estimation of the drift of the time-reversal of a variance preserving diffusion. Many algorithmic choices in RDMC are similar to $\SLIPS$, turning to Langevin for the drift estimation and using Langevin-within-Langevin for initialization. These authors also discuss the choice of the initial integration time, which plays a role similar to our $t_0$. Notably, under Lipschitz assumption on the score, they derive an upper bound on the overall complexity of RDMC, depending on this time. The present work complements the approach of \cite{huang2023monte}. We formalize the crucial trade-off in choosing $t_0$ according to our notion of `duality of log-concavity' (see \Cref{subsection:duality}), which is only briefly mentioned in \cite{huang2023monte}. This focal point naturally arises from our numerical tests in high-dimension that allows us to assess the potential of SL-based sampling in practice, where \cite{huang2023monte} remained mostly theoretical. A precise comparison is provided in \Cref{app:related_work}.

\vspace{-0.2cm}
\paragraph{Multi-Noise Measurements sampling.}
Recently, \cite{saremi2022multimeasurement} exploited a non-Markovian stochastic process to propose a novel data-based sampling scheme. The Multi-Noise Measurements (MNM) process $(Y^m)_{m=1}^M$ is defined by $Y^m = X + \sigma Z^m$ where $\sigma > 0$, $X \sim \pi$ and $Z^m$ are independently sampled from $\densityGaussian(0,\Idd)$.
Here, the \emph{measurement} $Y^m$ plays the same role as the \emph{observation} in standard SL\footnote{Indeed, both denoising processes have the same localization rate when $M=T$, see \Cref{app:related_work} for more details.}.
In \cite{saremi2023chain}, the authors suggest to use this process in the density-based setting by computing $Y^{1:M}$ sequentially through the sampling of $Y^m$ conditioned on $Y^{1:m-1}$. They show that those conditional distributions are increasingly log-concave, see \citep[Theorem 1]{saremi2023chain}, making their \emph{Once-At-a-Time} (OAT) algorithm efficient.
In contrast to our framework, \cite{saremi2023chain}  mainly assume that the score of the measurement process is analytically available. Although a Monte Carlo-based approach is proposed to estimate the score in realistic settings, results significantly degrade compared to the analytical case. As we detail in \Cref{app:related_work}, this can be explained by the fact that OAT faces but does not tackle the challenge of `duality of log-concavity'.

\vspace{-0.1cm}
\section{Numerical experiments} \label{sec:xps}

\vspace{-0.1cm}
In this section, we compare $\SLIPS$ against SMC, AIS, RDMC \cite{huang2023monte} and OAT \cite{saremi2023chain}. Note that we deliberately omit an exhaustive comparison with standard MCMC methods, as they notoriously fail to sample from multi-modal distributions, see \Cref{app:mcmc}. For $\SLIPS$, we consider three different SL schemes : Standard, Geom(1,1) and Geom(2,1). Except for RDMC, all the algorithms are informed by the scalar variance $R_\pi^2$ of the target distribution (or an estimation). We tuned the hyper-parameters of each algorithm with coarse grid searches of similar size and similar computational budgets assuming access to an oracle distance metric to the target distribution
(see details in \Cref{app:numerics}). For OAT, we estimated the intermediate scores with IS.

\vspace{-0.3cm}
\paragraph{Toy target distributions.} We first discuss standard target distributions including the \emph{8-Gaussians} ($d=2$), the \emph{Rings} ($d=2$) and the \emph{Funnel} ($d=10$) distributions. Details about their respective definitions are provided in \Cref{app:numerics}. Apart from Funnel, for which we provide results based on the sliced Kolmogorov-Smirnov (KS) distance as per \cite{grenioux2023onsampling}, we compare the samples obtained with the algorithms against the ground truth using the entropic regularized 2-Wasserstein distance. The first four columns of \Cref{tbl:bench_toys_bays} show that $\SLIPS$ is on par with most of its competitors on those toy distributions.

\vspace{-0.2cm}
\paragraph{Bayesian Logistic Regression.}

Beyond toy distributions, we sample from the posterior of a Bayesian logistic regression model on two popular datasets : \emph{Ionosphere} ($d = 34$) and \emph{Sonar} ($d = 61$). More details on the design of the model are available in \Cref{app:numerics}. We evaluate the quality of sampling by computing the average predictive log-likelihood of the obtained samples. The last two columns of \Cref{tbl:bench_toys_bays} show that, in these higher dimensions, $\SLIPS$ is slightly superior to its counterparts, especially OAT and RDMC.

\vspace{-0.2cm}
\paragraph{High-dimensional Gaussian Mixtures.}
As a challenging task, we seek to estimate the relative weight of a bimodal Gaussian mixture with modes $\densityGaussian\left(x; -(2/3)\mathbf{1}_{d}, \Sigma\right)$ and $\densityGaussian\left(x; (4/3) \mathbf{1}_{d}, \Sigma\right)$, where $\Sigma = 0.05 \, \mathrm{I}_{d}$ and $\mathbf{1}_{d}$ is the $d$-dimensional vector with all components equal to $1$, and respective weights $2/3$ and $1/3$. In our experiments, these weights are computed as $\hat{W}$ and $1-\hat{W}$, where $\hat{W}$ is a Monte Carlo estimation of $\textstyle \int_{\rset^d}\mathbbm{1}_{(-\infty,0)^d}(x)\rmd \pi(x)$. Here, we consider increasing values of $d$. \Cref{fig:bench:two_modes} (\textit{top}) shows that $\SLIPS$ recovers the relative weight of the target at a 1\% accuracy, even in high dimensions, whereas the other methods fail to give accurate estimates in most dimensions. In \Cref{fig:bench:two_modes} (\textit{bottom}), we also illustrate the superiority of $\SLIPS$ to capture the local properties of the distribution by displaying the sliced Wasserstein distance to the target distribution. Due to high estimation error, AIS and SMC results are omitted in \Cref{fig:bench:two_modes} and displayed in \Cref{app:numerics}.

\begin{table*}[t]
	\caption{Metrics when sampling toy distributions and posteriors from Bayesian logistic regression models. Bold font indicates best results. The metric for 8-Gaussians and Rings ($d = 2$) is the entropic regularized 2-Wasserstein distance (with regularization hyper-parameter $0.05$) (\textit{the lowest, the best}), the metric for Funnel ($d=10$) is the sliced Kolmogorov-Smirnov distance (\textit{the lowest, the best}) and the metric for the Bayesian logistic regression on Sonar ($d=34$) and Ionosphere ($d=61$) datasets corresponds to the average predictive posterior log-likelihood (\textit{the highest, the best}).}
	\label{tbl:bench_toys_bays}
	\vspace{-0.3cm}
	\begin{center}
		\begin{small}
			\begin{sc}
				\begin{tabular}{lccccc}
					\toprule
					Algorithm                  & 8-Gaussians ($\downarrow$) & Rings ($\downarrow$)     & Funnel ($\downarrow$)      & Sonar  ($\uparrow$)         & Ionosphere ($\uparrow$)    \\
					\midrule
					AIS                        & $1.10 \pm 0.09$            & $0.19 \pm 0.02$          & $0.037 \pm 0.004$          & $-111.04 \pm 0.08$          & $-87.92 \pm 0.13$          \\
					SMC                        & $0.99 \pm 0.16$            & $0.28 \pm 0.06$          & $0.035 \pm 0.005$          & $-111.02 \pm 0.13$          & $-87.82 \pm 0.16$          \\
					OAT \cite{saremi2023chain} & $0.91 \pm 0.09$            & $0.18 \pm 0.02$          & $0.105 \pm 0.005$          & $-280.92 \pm 1.34$          & $-205.49 \pm 0.61$         \\
					RDMC \cite{huang2023monte} & $1.01 \pm 0.05$            & $0.30 \pm 0.01$          & $0.082 \pm 0.006$          & $-129.88 \pm 0.12$          & $-109.84 \pm 0.10$         \\
					\hline
					$\SLIPS$ Standard          & $0.76 \pm 0.05$            & $\mathbf{0.19 \pm 0.01}$ & $\mathbf{0.024 \pm 0.003}$ & $-109.25 \pm 0.07$          & $-86.65 \pm 0.04$          \\
					$\SLIPS$ Geom(1,1)         & $\mathbf{0.74 \pm 0.12}$   & $0.20 \pm 0.01$          & $0.032 \pm 0.002$          & $\mathbf{-109.14 \pm 0.09}$ & $\mathbf{-86.32 \pm 0.10}$ \\
					$\SLIPS$ Geom(2,1)         & $0.75 \pm 0.10$            & $0.22 \pm 0.02$          & $0.040 \pm 0.007$          & $-110.24 \pm 0.05$          & $-86.78 \pm 0.08$          \\
					\bottomrule
				\end{tabular}
			\end{sc}
		\end{small}
	\end{center}
	\vskip -0.1in
\end{table*}

\begin{figure}[h!]
	\centering
	\includegraphics[width=\linewidth]{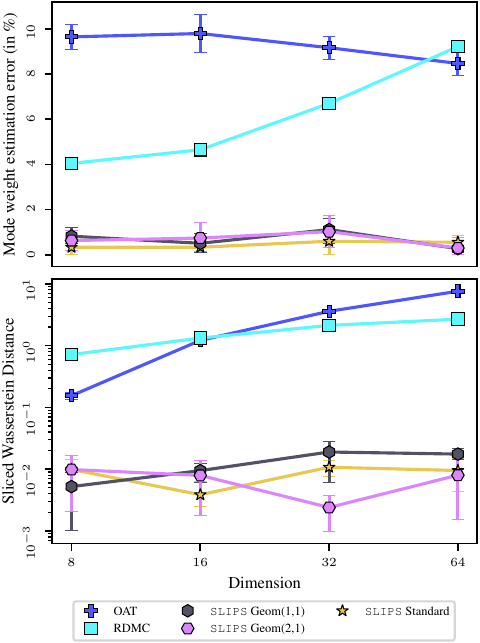}
	\vspace{-0.5cm}
	\caption{Metrics when sampling a bimodal Gaussian mixture with $d$ growing. \textbf{Top}: Relative weight estimation error. \textbf{Bottom}: Sliced Wasserstein distance.}
	\label{fig:bench:two_modes}
	\vspace{-0.5cm}
\end{figure}

\paragraph{Field system $\phi^4$ from statistical mechanics.} Lastly, we sample the 1D $\phi^4$ model, which was recently used as a benchmark in \cite{gabrie2019adaptive}. At the chosen temperature, the distribution has two well distinct modes with relative weight that can be adjusted through a `local-field' parameter $h$. We discretize this continuous model with a grid size of 100 (\ie, $d = 100$).
We estimate the relative weight between the two modes and compare the results with a Laplace approximation.

\begin{figure}[h!]
	\centering
	\includegraphics[width=\linewidth]{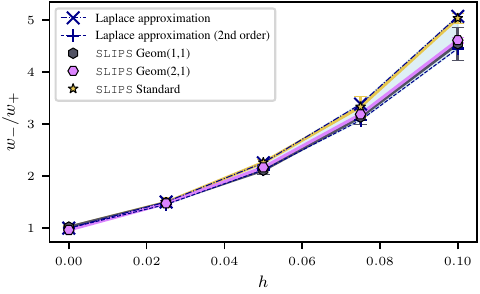}
	\vspace{-0.6cm}
	\caption{Estimation of the mode weight ratio of $\phi^4$ with increasing $h$ - Only $\SLIPS$ produced correct samples.}
	\label{fig:bench:weight_ratio_sto_loc}
	\vspace{-0.5cm}
\end{figure}

\Cref{fig:bench:weight_ratio_sto_loc} shows that the relative weight estimated by $\SLIPS$ lie between the 0-th and 2-nd order Laplace approximations. Due to high estimation error, the results from concurrent algorithms are omitted in \Cref{fig:bench:weight_ratio_sto_loc}. We refer to \Cref{app:numerics} for more details on the setting of this numerical experiment.

\section{Discussion}

In this paper, we introduced a Stochastic Localization (SL) scheme, that features flexible denoising schedule, in order to sample from any unnormalized target density. Relying on this framework, we proposed our algorithm, \emph{Stochastic Localization via Iterative Posterior Sampling} ($\SLIPS$) which leverages Monte Carlo estimation of the SL denoiser. Its design notably reveals the so-called ``duality of log-concavity'', that can be seen as a trade-off between sampling from the observation process and sampling from the SL posterior. For various localization schemes, we illustrate the performance of $\SLIPS$ in high-dimension. In future work, we would like to derive more theoretical guarantees on the phenomenon of duality, investigate improved designs of denoising schedule and consider more challenging target distributions.

\newpage

\section*{Acknowledgements}

The authors would like to thank Saeed Saremi and Francis Bach for insightful discussions during the course of this project. The work of LG and MG is supported by Hi! Paris. Part of the work of AD is funded by the European Union (ERC, Ocean, 101071601). Views and opinions expressed here are however those of the authors only and do not necessarily reflect those of the European Union or the European Research Council Executive Agency. Neither the European Union nor the granting authority can be held responsible for them.

\section*{Impact Statement}

This paper presents work whose goal is to advance the field of sampling from unnormalized densities. There are many potential societal consequences of our work, none which we feel must be specifically highlighted here.

\bibliography{main}
\bibliographystyle{icml2024}

\newpage
\appendix
\onecolumn

\crefalias{subsubsection}{appendix}
\crefalias{subsection}{appendix}
\crefalias{section}{appendix}
\crefalias{chapter}{appendix}

\section*{Organization of the supplementary}

The appendix is organized as follows. \Cref{app:preliminaries} summarizes general facts that will be useful for proofs. \Cref{app:details_sl} provides details on the Stochastic Localization (SL) frameworks introduced in \Cref{sec:background} and \Cref{sec:gen-sto-loc}. In \Cref{app:gaussian_mixture}, we derive detailed computations on the application of SL to Gaussian mixtures. \Cref{app:perfect_score} presents an implementation of SL for sampling in the case where the score of the observation process is available. \Cref{app:theory} dispenses theoretical results on the phenomenon of ``duality of log-concavity" highlighted in \Cref{subsec:estimation} and \Cref{subsection:duality}. In \Cref{app:slips}, we detail the implementation of our sampling algorithm called $\SLIPS$ and provide an ablation study of its hyper-parameters. A detailed comparison between our approach and works from \cite{huang2023monte} and \cite{saremi2023chain} is given in \Cref{app:related_work}. Finally, \Cref{app:numerics} provides more details on the numerical experiments presented in \Cref{sec:xps}.

\paragraph{Notation.} To alleviate the computations derived below, for a probability distribution $\mu$ defined on $\rset^d$, and any measurable function $\varphi:\rset^d \to\rset^d$, we will denote the expectation $\int_{\rset^d}\varphi(x)\rmd \mu(x) \in \rset^d$ by $\mathbb{E}[\varphi(X)]$
and the covariance $\int_{\rset^d}(\varphi(x)-\mu(\varphi))(\varphi(x)-\mu(\varphi))^\top\rmd \mu(x) \in \rset^{d\times d}$ by $\mathrm{Cov}[\varphi(X)]$, where $X$ is a random vector distributed according to $\mu$. For any $T\in (0, +\infty]$, two functions $f$ and $g$ defined on $[0,T)$ are said to be asymptotically equivalent if $f(t)/g(t)\to 1$ and $g(t)/f(t) \to 1$ as $t \to T$. This is denoted by $f(t)\sim g(t)$.

\section{Preliminaries} \label{app:preliminaries}

\paragraph{Metrics.}We recall that the 2-Wasserstein distance between two distributions $\mu$ and $\nu$ is given by
\begin{align}
	\textstyle{ W_{2}(\mu, \nu) = \inf \{\int_{\rset^d \times \rset^d}\|x_1-x_0\|^2\rmd \pi(x_0,x_1): \pi_0=\mu, \pi_1=\nu\}^{1/2} \eqsp ,}
\end{align}
where $\pi_i$ denotes the $i$-th marginal of $\pi$ for $i\in \{0,1\}$. While it is not tractable in general, we have an explicit expression when comparing two Gaussian distributions. Consider $\mu=\densityGaussian(\mathbf{m}_1, \gamma_1^2 \, \Idd)$ and $\nu=\densityGaussian(\mathbf{m}_2, \gamma_2^2 \, \Idd)$, where $(\mathbf{m}_1, \mathbf{m}_2)\in \rset^d \times \rset^d$ and $(\gamma_1, \gamma_2)\in (0, \infty)^2$. Then, following \citep[Equation 2.41]{peyre2019computational}, we have
\begin{align} \label{eq:wass_gaussian}
	W_{2}(\mu, \nu)^2= \norm{\mathbf{m}_1-\mathbf{m}_2}^2 + d(\gamma_1 -\gamma_2)^2 \eqsp .
\end{align}

We also recall that the entropic regularized 2-Wasserstein distance between two distributions $\mu$ and $\nu$ is defined by
\begin{align}
	\textstyle W_{2,\varepsilon}(\mu, \nu) = \inf\{\int_{\rset^d \times \rset^d}\|x_1-x_0\|^2\rmd \pi(x_0,x_1) - \mathscr{H}(\pi): \pi_0=\mu, \pi_1=\nu\}^{1/2} \eqsp ,
\end{align}
where $\varepsilon > 0$ is a regularization hyper-parameter and $\mathscr{H}(\pi) = -\int_{\rset^d \times \rset^d} \log \pi(x_0,x_1) \rmd \pi( x_0, x_1)$ refers to as the entropy of $\pi$.

Finally, we recall that the Kolmogorov-Smirnov distance between two distributions $\mu$ and $\nu$ is defined by
\begin{align}
	\operatorname{KS}(\mu,\nu) = \sup_{x \in \rset^d} \abs{F_{\mu}(x) - F_{\nu}(x)}\eqsp,
\end{align}
where $F_{\mu}$ (respectively $F_{\nu}$) denotes the cumulative distribution function of $\mu$ (respectively $\nu$).

Below, we provide a first result, known as the Tweedie's formula, which gives the expression of the score of any distribution that writes as a marginalization, and its derivative.

\begin{lemma}[Tweedie's formula and extension]\label{lemma:tweedie} Let $p$ be a positive probability density on $\rset^d$ such that for any $y\in \rset^d$ $p(y)=\int_{\rset^d}p(y|x)\rmd \mu(x)$, where $(x,y) \mapsto p(y|x)$ is a positive transition density on $\rset^d\times \rset^d$, and $\mu$ is a probalility mesure on $\rset^d$. Assume that for any $x \in \rset^d$, $y\mapsto \log p(y|x)$ is twice continuously differentiable and that, for any $k\in\{1,2\}$, there exists $\varphi_k: \rset^d \to \rset_+$ such that $\int_{\rset^d} \varphi_k(x) \rmd \mu(x) < \infty$ and for any $(x,y)\in \rset^d \times \rset^d$, we have $\norm{\nabla_y^k p(y|x)} \leq \varphi_k(x)$.

	For any $y\in \rset^d$, define the 'posterior' density $x \mapsto p(x|y)$ by $p(x|y)= p(y|x) p(x)/p(y)$.

	Then, we have for any $y\in \rset^d$
	\begin{align} \label{eq:Tweedie}
		\nabla_y \log p(y)=\int_{\rset^d}\nabla_y \log p(y|x) p(x|y) \rmd x  \eqsp .
	\end{align}
	This is often referred to as the Tweedie's formula \cite{robbins1992empirical}. We also have for any $y \in \rset^d$
	\begin{align} \label{eq:Tweedie_2}
		\nabla_y^2 \log p(y) & = \int_{\rset^d}\nabla_y^2 \log p(y|x) p(x|y) \rmd x  + \int_{\rset^d}\nabla_y \log p(y|x) \{\nabla_y \log p(y|x)\}^\top   p(x|y) \rmd x             \\
		                     & \quad- \left\{\int_{\rset^d}\nabla_y \log p(y|x)  p(x|y) \rmd x\right\}\left\{\int_{\rset^d}\nabla_y \log p(y|x)  p(x|y) \rmd x\right\}^\top \eqsp .
	\end{align}
	For any $y\in \rset^d$, these two results can be reformulated as
	\begin{align}
		\nabla_y \log p(y) = \mathbb{E}[\nabla_y \log p(y|X)] \eqsp , \eqsp \nabla_y^2 \log p(y)= \mathbb{E}[\nabla_y^2 \log p(y|X)] + \mathrm{Cov}[\nabla_y \log p(y|X)] \eqsp ,
	\end{align}
	where $X$ is a random vector distributed according to the posterior distribution defined by $\rmd \mu^y(x)=p(x|y) \rmd x$.
\end{lemma}

\begin{proof} First observe that for any $x\in \rset^d$, $y \mapsto p(y|x)$ is twice continuously differentiable since $p(y|x) = \exp( \log p(y|x))$, where $y \mapsto \log p(y|x)$ is twice continuously differentiable.

	We begin with the proof of \eqref{eq:Tweedie}. By combining the assumptions of the lemma with the dominated convergence theorem, we obtain that $p\in \mathrm{C}^2(\rset^d, (0, \infty))$, and that for any $y\in \rset^d$, we have
	\begin{align}
		\nabla_y \log p(y) & = \frac{\nabla_y p(y) }{p(y)}=\frac{\nabla_y \int_{\rset^d} p(y|x) \rmd \mu(x)}{\int_{\rset^d} p(y|x) \rmd \mu(x)}= \frac{ \int_{\rset^d} \nabla_y \log p(y|x) p(y|x) \rmd \mu(x)}{\int_{\rset^d} p(y|x) \rmd \mu(x)} = \int_{\rset^d}\nabla_y \log p(y|x) p(x|y) \rmd x \eqsp ,
	\end{align}
	We now turn to the proof of \eqref{eq:Tweedie_2}. Note that we have $\nabla_y \log p(y)=p(y)^{-1}\int_{\rset^d} \nabla_y \log p(y|x) p(y|x) \rmd \mu(x)$. Then, using once again the dominated convergence theorem, we obtain that for any $y\in \rset^d$,
	\begin{align}
		\nabla_y^2 \log p(y) & = \frac{1}{p(y)}\int_{\rset^d} \nabla_y^2 \log p(y|x) p(y|x) \rmd \mu(x) + \frac{1}{p(y)}\int_{\rset^d} \nabla_y \log p(y|x) \{\nabla_y \log p(y|x)\}^\top p(y|x) \rmd \mu(x) \\
		                     & \quad - \frac{\nabla_y p(y)}{p(y)^2}\left\{\int_{\rset^d} \nabla_y \log p(y|x) p(y|x) \rmd \mu(x)\right\}^\top                                                                \\
		                     & = \int_{\rset^d}\nabla_y^2 \log p(y|x) p(x|y) \rmd x + \int_{\rset^d}\nabla_y \log p(y|x) \{\nabla_y \log p(y|x)\}^\top   p(x|y) \rmd x                                       \\
		                     & \quad- \nabla_y \log p(y)\left\{\frac{1}{p(y)}\int_{\rset^d} \nabla_y \log p(y|x) p(y|x) \rmd \mu(x)\right\}^\top \eqsp,
	\end{align}
	which gives the result.
\end{proof}

We now dispense a useful lemma to compute exact integration in SDEs with linear drift.
\begin{lemma} \label{lemma:ito} Let $T>0$. Consider the SDE defined on $[0,T]$ by $\rmd Y_t= \beta_t (Y_t +b) \rmd t +\sigma \rmd B_t$, where $\beta$ is integrable on $[0,T]$, $b\in\rset^d$, and $\sigma>0$. For any $T\geq t>s\geq 0$, the conditional density of $Y_t$ given $Y_s=y_s$, denoted by $p_{t|s}(\cdot|y_s)$, verifies
	\begin{align}
		\textstyle{p_{t|s}(y_t|y_s)= \densityGaussian(y_t;\exp(\int_{s}^{t}\beta_u \rmd u) y_s + (\exp(\int_{s}^{t}\beta_u\rmd u)-1)b , \sigma^2 \int_s^{t}\exp(2 \int_{u}^{t}\beta_r \rmd r)\rmd u \,\Idd)} \eqsp .
	\end{align}
\end{lemma}
\begin{proof} Define $\gamma(t)= \exp(-\int_{0}^{t}\beta_u \rmd u)$. Consider the stochastic process $Z_t= \gamma(t) Y_t$. By Îto's formula, we have $\rmd Z_t=\beta_t\gamma(t) b \rmd t +\gamma(t) \sigma\rmd B_t= -\dot{\gamma}(t) b \rmd t + \gamma(t) \sigma \rmd B_t$. By integrating this SDE between $s$ and $t$, we have
	\begin{align}
		\textstyle \gamma(t)Y_{t}-\gamma(s)Y_{s}= \{\gamma(s) -\gamma(t)\}b+  \sigma \int_{s}^{t}\gamma(u) \rmd B_u \eqsp ,
	\end{align}
	and then
	\begin{align}
		\textstyle Y_{t}= \exp(\int_{s}^{t}\beta_u \rmd u) Y_s + ( \exp(\int_{s}^{t}\beta_u\rmd u)-1)b + \sigma \int_{s}^{t}\exp( \int_{u}^{t}\beta_r \rmd r)\rmd B_u \eqsp,
	\end{align}
	which gives the result using Îto's isometry and the fact that $Y_s$ is independent from $(B_t-B_s)_{t\in [s,T]}$.
\end{proof}
\newpage
\section{Details on Stochastic Localization and our extension}\label{app:details_sl}

\paragraph{Vocabulary of Stochastic Localization (SL).} For sake of clarity, we precise below some terms employed in our paper:
\begin{itemize}
	\item $Y^\alpha_t$: observation process,
	\item $g$: denoising schedule,
	\item $u_t^\alpha$ : denoiser function,
	\item $u_t^\alpha(Y^\alpha_t)$: denoiser or Bayes estimator,
	\item $q^\alpha_t(\cdot|Y^\alpha_t)$ : random posterior density. In the geometric measure theory literature, this is referred to as the SL process.
\end{itemize}

\subsection{Reminders and intuition on generalized stochastic localization}

\paragraph{Connection between standard stochastic localization and diffusion models.} Stochastic localization as defined in \eqref{eq:def_y_t} is equivalent to Variance-Preserving (VP) diffusion models \cite{song2020score} under change of time. Indeed, if we define $\Phi(t)=\log(1+\sigma^2 t^{-1}) / 2$ for any $t\geq0$, with the convention that $\Phi(0)=\infty$, and consider the Ornstein-Uhlenbeck process $(X_t)_{t\geq 0}$ solving the SDE
\begin{align} \label{eq:OU}
	\rmd X_s = -X_s \rmd s + \sqrt{2}\rmd B_s \eqsp, \quad X_0 \sim \pi\eqsp ,
\end{align}
then, $(Y_t)_{t\geq 0}$ and $(\sqrt{t(t+\sigma^2)}X_{\Phi(t)})_{t\geq 0}$ have the same distribution by application of Dubins-Schwartz theorem \cite{montanari2023sampling}. In other words, the localization process \eqref{eq:def_y_t} can be identified as a \emph{non-linear} time-reversal transformation of the noising VP process \eqref{eq:OU}.

\paragraph{Comments on our framework.} We recall that $\alpha(t)=t^{1/2} g(t)$. We now justify the assumptions given on g in \Cref{subsec:gen-framework}.
\begin{itemize}
	\item Rationale behind \ref{item:gsl-fun} and \ref{item:gsl-0}: together, these requirements ensure that (i) $\alpha$ is continuously derivable at $t=0$, without imposing this same condition on $g$ and that (ii) we have $\alpha(0)=0$. With (i), the SDE defined in \eqref{eq:SDE-gen} does not have a singularity at time $t=0$. With (ii), it allows to fix $Y_0^\alpha$ to a deterministic value. This is equivalent to full independence with $X$, \ie, full noise looking at the expression of the SNR in \eqref{eq:SNR}, which is natural to start the denoising process. Note that taking $-g$ instead of $g$ would be equivalent in the denoising procedure, since the SNR considers $g(t)^2$. Nonetheless, considering $g$ with positive values is arbitrary and allows us to simplify our framework.
	\item Rationale behind \ref{item:gsl-T}: as $t\to T_{\text{gen}}$, we would like to denoise increasingly the observation process, \ie,  gaining progressively information about $X$, such that we obtain complete denoising at the end of the process. Looking at the expression of the SNR in \eqref{eq:SNR}, this requires to naturally take $g(t)\to \infty$ as $t\to T_{\text{gen}}$ and $g$ strictly increasing, recalling that $g$ takes non negative values.
\end{itemize}

Note that it includes the case of standard stochastic localization by taking $g(t)=t^{1/2}$. In this case, all properties are verified: in particular, consider $C=1$ and $\beta=1$ in \ref{item:gsl-0}.

\subsection{Theoretical results.} \label{subapp:sto-loc-th}
Here, we provide details on the theoretical claims made in \Cref{subsec:gen-framework} on our extended framework of stochastic localization.

\paragraph{Link between score and denoiser function.} The formula given in \eqref{eq:link_score_general} is an immediate corollary of \Cref{lemma:tweedie} applied to the observation process defined in \eqref{eq:def_y_t-gen}. Indeed, the conditional distribution of $Y_t^\alpha$ given $X=x\in \rset^d$ is $p_t^\alpha(\cdot|x)=\densityGaussian(\alpha(t)x, \sigma^2 t \, \Idd)$. Then, it comes that $\nabla_y\log p_t^\alpha(y|x)=\{\alpha(t)x -y\}/\sigma^2 t$, and we obtain the result using \Cref{lemma:tweedie}-\eqref{eq:Tweedie}.

\paragraph{Localization rate.}We begin with the localization rate given in \Cref{prop:standard_gen_sto_loc_cv}.  We define the following quantity $\locr(t)=\sigma \sqrt{d} /g(t)$. We emphasize that the assumption of finite second order moment on the target distribution, that is used below, guarantees the proper definition of the 2-Wasserstein distance.
\begin{proposition} \label{prop:conv_sto_loc} Assume that $\pi$ has finite second order moment. Denote by $\pi^{\alpha}_t$ the probability distribution of $Y_t^\alpha/\alpha(t)$ where $(Y^\alpha_t)_{t \in [0,T_{\text{gen}})}$ is the stochastic observation process defined in \eqref{eq:def_y_t-gen}. Then, for any $t\in (0,T_{\text{gen}})$, we have
	\begin{align}
		W_2(\pi, \pi^\alpha_t)\leq \locr(t) \eqsp .
	\end{align}
\end{proposition}

\begin{proof} Let $t\in (0,T_{\text{gen}})$. Consider the following coupling $(X,\hat{Y}^\alpha_t)$: (i) $X \sim \pi$, (ii) $\hat{Y}^\alpha_t=\alpha(t)X+\sigma \sqrt{t} Z$, where $Z\sim \densityGaussian(0,\Idd)$. Denote $X^\alpha_t= \hat{Y}^\alpha_t/\alpha(t)$. We have $X^\alpha_t= X + \{\sigma/g(t)\}\,Z$ and $\hat{X}^\alpha_t\sim \pi^\alpha_t$. Therefore, it comes that
	\begin{align}
		W_2(\pi, \pi^{\alpha}_t)^2\leq
		\mathbb{E}[\|\sigma/g(t) \, Z\|^2]=  \sigma^2 d/g(t)^2 \eqsp ,
	\end{align}
	which gives the result.
\end{proof}

In particular, we recover the localization rate of the standard setting given in \Cref{sec:background} with $g(t)=t^{1/2}$ and $T_{\text{gen}}=\infty$. Note that this upper bound applies on general distributions that may be non log-concave, as considered in \Cref{ass:target}, and may be improved by assuming further regularity. We notably show a strong refinement in the Gaussian case by a factor $O(g(t))$.
\begin{proposition} \label{prop:Gaussian} Consider the target distribution given by $\pi=\densityGaussian(\mathbf{m}, \gamma^2 \, \Idd)$ where $\mathbf{m}\in \rset^d$ and $\gamma>0$. Then, as $t\to T_{\text{gen}}$, we have
	\begin{align}
		W_2(\pi, \pi^{\alpha}_t)=\gamma\abs{ 1-\left(1+\frac{\sigma^2}{\gamma^2g(t)^2}\right)^{1/2}}\sqrt{d} \sim \frac{\sigma}{2\gamma g(t)}\locr(t) \eqsp .
	\end{align}
\end{proposition}
\begin{proof}
	In this setting, the distribution of the observation process at time $t$ is tractable, and we have
	\begin{align}
		p_t^\alpha=\densityGaussian(\alpha(t)\mathbf{m}, \{\alpha(t)^2 \gamma^2 + \sigma^2 t\}\, \, \Idd) \eqsp .
	\end{align}
	Then, we get that
	\begin{align}
		\pi^{\alpha}_t=\densityGaussian(\mathbf{m}, \{\gamma^2 + \sigma^2/g(t)^2\}\, \Idd) \eqsp .
	\end{align}
	In this particular, using \eqref{eq:wass_gaussian}, the 2-Wasserstein distance is given by
	\begin{align}
		W_2(\pi, \pi^{\alpha}_t)=\gamma\abs{ 1-\left(1+\frac{\sigma^2}{\gamma^2g(t)^2}\right)^{1/2}}\sqrt{d}  \eqsp ,
	\end{align}
	from which we deduce the result by a simple asymptotic equivalent.
\end{proof}

Interestingly, the same bound applies on the denoiser $u_t^\alpha(Y^\alpha_t)=\int_{\rset^d}xq_t^\alpha(x|Y_t^\alpha)$. The proof relies on the same structure.

\begin{proposition} \label{prop:conv_sto_loc-denoiser} Assume that $\pi$ has finite second order moment. Denote by $\tilde{\pi}^{\alpha}_t$ the probability distribution of $u_t^\alpha(Y^\alpha_t)$ where $(Y^\alpha_t)_{t \in [0,T_{\text{gen}})}$ is the stochastic observation process defined in \eqref{eq:def_y_t-gen}. Then, for any $t\in (0,T_{\text{gen}})$, we have
	\begin{align}
		W_2(\pi, \tilde{\pi}^\alpha_t)\leq \locr(t) \eqsp .
	\end{align}
\end{proposition}

\begin{proof} Let $t\in (0,T_{\text{gen}})$. Consider the following coupling $(X,\hat{Y}^\alpha_t)$: (i) $X \sim \pi$, (ii) $\hat{Y}^\alpha_t=\alpha(t)X+\sigma \sqrt{t} Z$, where $Z\sim \densityGaussian(0,\Idd)$. Denote $X^\alpha_t= u_t^\alpha(Y^\alpha_t)$. We have $X^\alpha_t\sim \tilde{\pi}^\alpha_t$. Therefore, it comes that
	\begin{align}
		W_2(\pi, \tilde{\pi}^{\alpha}_t)^2\leq \mathbb{E}_{X,X^\alpha_t}[\|X^\alpha_t-X\|^2]=  \mathbb{E}_{X,\hat{Y}^\alpha_t}[\|\mathbb{E}[X|\hat{Y}^\alpha_t]-X\|^2]=\mathbb{E}_{X,\hat{Y}^\alpha_t}[\|\mathbb{E}[X|\hat{Y}^\alpha_t/\alpha(t)]-X\|^2] \eqsp .
	\end{align}
	Since conditional expectations are orthogonal projections in $\mathrm{L}^2$, we have
	\begin{align}
		W_2(\pi, \tilde{\pi}^{\alpha}_t)^2\leq  \mathbb{E}_{X,\hat{Y}^\alpha_t}[\| \hat{Y}^\alpha_t/\alpha(t)-X\|^2]= \mathbb{E}[\| \sigma /g(t) \, Z\|^2]=\sigma^2 d/g(t)^2 \eqsp ,
	\end{align}
	which gives the result.
\end{proof}

Similarly, we also show a strong refinement of \Cref{prop:conv_sto_loc-denoiser} by a factor $O(g(t))$ in the Gaussian case.

\begin{proposition} \label{prop:Gaussian-denoiser} Consider the target distribution given by $\pi=\densityGaussian(\mathbf{m}, \gamma^2 \Idd)$ where $\mathbf{m}\in \rset^d$ and $\gamma>0$. Then, as $t\to T_{\text{gen}}$, we have
	\begin{align}
		W_2(\pi, \tilde{\pi}^{\alpha}_t)=\gamma\abs{ 1-\left(1+\frac{\sigma^2}{\gamma^2g(t)^2}\right)^{-1/2}}\sqrt{d} \sim \frac{\sigma}{2\gamma g(t)}\locr(t) \eqsp .
	\end{align}
\end{proposition}
\begin{proof}
	Here again, the distribution of the observation process at time $t$ is tractable, we have
	\begin{align} \label{eq:prop_1}
		p_t^\alpha=\densityGaussian(\alpha(t)\mathbf{m}, \{\alpha(t)^2 \gamma^2 + \sigma^2 t\}\, \Idd) \eqsp .
	\end{align}
	In particular, the score is tractable and given by
	\begin{align}
		\nabla_y \log p_t^\alpha(y)=\frac{\alpha(t)\mathbf{m} - y}{\alpha(t)^2 \gamma^2 + \sigma^2 t} \eqsp .
	\end{align}
	Using relation \eqref{eq:link_score_general}, we thus have
	\begin{align}\label{eq:prop_2}
		u_t^\alpha(y)=\frac{\alpha(t)\gamma^2}{\alpha(t)^2\gamma^2 + \sigma^2 t}y + \frac{\sigma^2t}{\alpha(t)^2\gamma^2 + \sigma^2 t}\mathbf{m} \eqsp .
	\end{align}
	Recalling that $\tilde{\pi}^{\alpha}_t$ is the distribution of $u_t^\alpha(Y^\alpha_t)$, by combining \eqref{eq:prop_1} and \eqref{eq:prop_2}, it comes that
	\begin{align}
		\textstyle \tilde{\pi}^{\alpha}_t=\densityGaussian\left(\mathbf{m}, \gamma^2\left\{1+ \frac{\sigma^2}{\gamma^2 g(t)^2}\right\}^{-1}\, \Idd\right) \eqsp .
	\end{align}
	In this particular, using \eqref{eq:wass_gaussian}, the 2-Wasserstein distance is given by
	\begin{align}
		W_2(\pi, \tilde{\pi}^{\alpha}_t)=\gamma\abs{ 1-\left(1+\frac{\sigma^2}{\gamma^2g(t)^2}\right)^{-1/2}}\sqrt{d}
	\end{align}
	from which we deduce the result with a simple asymptotic equivalent.
\end{proof}

\paragraph{General remark on denoising in Stochastic Localization.}

In our presentation of standard stochastic localization where the observation process is defined in \eqref{eq:def_y_t}, we take $Y_T/T$ as approximate sample from $\pi$, where $Y_T$ is obtained by running the (discretized) SDE \eqref{eq:sde_sto_loc_with_denoiser} up to time $T$. Yet, in standard stochastic localization, \cite{montanari2023sampling} proposed to take the last denoiser of the procedure $u_{T}(Y_T)=\int_{\rset^d} x q_T(x|Y_T) \rmd x$. We highlight here that this is equivalent asymptotically.
\begin{itemize}
	\item The two approaches have the same upper bound on their rate of convergence, see \Cref{prop:conv_sto_loc} and \Cref{prop:conv_sto_loc-denoiser} in the standard setting. Note that we also obtain the same rate asymptotically in the Gaussian case, see \Cref{prop:Gaussian} and \Cref{prop:Gaussian-denoiser}.
	\item When $T$ is large, the Gaussian term dominates the target term in the expression of the posterior $q_t$ given in \eqref{eq:standard_posterior}. Noting that $\sigma^2/T\ll 1$, it comes that $q_T(x|Y_T)\approx \updelta_{Y_T/T}$, and therefore $u_{T}(Y_T)\approx Y_T/T$.
\end{itemize}

As shown above, this equivalence also occurs in generalized stochastic localization. In addition to the similarity of results in \Cref{prop:conv_sto_loc} and \Cref{prop:conv_sto_loc-denoiser}, we observe that, as $T$ is large if $T_{\text{gen}}=\infty$ or as $T$ is close to $T_{\text{gen}}$ otherwise, the Gaussian term dominates the target term in the expression of the posterior $q_t$ given in \eqref{eq:general_posterior}. Since $\sigma^2/g(T)^2\ll 1$, it comes that $q_T^\alpha(x|Y^\alpha_T)\approx \updelta_{Y_T^\alpha/\alpha(T)}$, and therefore, we have again $u^\alpha_{T}(Y^\alpha_T)\approx Y^\alpha_T/\alpha(T)$.

In the main document of our paper, we choose to consider the distribution of $Y_T^\alpha/\alpha(T)$ rather than the the distribution of the denoiser $u^\alpha_{T}(Y^\alpha_T)$ as an approximation of the target distribution. This choice seemed to us to be easier to understand for non-expert readers. In practice, we however align with the original methodology and compute the denoiser $u_{T}^\alpha(Y_T^\alpha)$ instead of $Y^\alpha_T/\alpha(T)$. This allows to us to have a fair comparison with the approach from \cite{saremi2023chain}, who also consider such Bayes estimator for sampling.

\paragraph{Markovian projection of the observation process.}


Following \cite{liptser1977statistics} and \cite{Brunick2013mimicking}, we now turn to existence and uniqueness results on the SDE \eqref{eq:SDE-gen}, under assumptions on $\pi$. These are corollaries of the general results: \citep[Corollary 3.7]{Brunick2013mimicking}, restated in \Cref{prop:SDE-existence}, and \citep[Theorem 7.6]{liptser1977statistics}, restated in \Cref{prop:SDE-unique}.

\begin{proposition}[Corollary 3.7 in \cite{Brunick2013mimicking}] \label{prop:SDE-existence}
	Let $(\Omega, \mathcal{F}, \mathbb{P})$ be a complete probability space, and let $W$ be an $\rset^r$-valued Brownian motion defined on this space. Fix a time horizon $T > 0$, and consider the $\rset^d$-valued stochastic process $(X_t)_{t \in [0, T]}$ defined by
	$$
		X_t = X_0 + \int_0^t b_s \rmd s + \int_0^t \sigma_s \rmd W_s, \quad 0 \leq t \leq T\eqsp,
	$$
	where $b = (b_t)_{t \in [0,T]}$ is an $\rset^d$-valued process adapted to a filtration with respect to which $W$ is a Brownian motion, $\sigma = (\sigma_t)_{t \in [0,T]}$ is a $d \times r$-valued adapted process (with respect to the same filtration), and the integrability condition
	$$
		\mathbb{E}\left[\int_0^t \left(\norm{b_s} + \norm{\sigma_s \sigma_s^\top}\right) \rmd s\right] < \infty, \quad \text{for all } t \in [0,T]\eqsp,
	$$
	is satisfied. Then there exist measurable functions $\hat{b} : [0,T) \times \rset^d \to \rset^d$, and $\hat{\sigma} : [0,T) \times \rset^d \to \rset^{d \times d}$,	such that for each $t \in [0,T)$ and $x_t \in \rset^d$,
	$$
		\hat{b}(t, x_t) = \mathbb{E}\left[b_t \mid X_t = x_t\right], \quad \hat{\sigma}(t, x_t)\hat{\sigma}(t, x_t)^\top = \mathbb{E}\left[\sigma_t \sigma_t^\top \mid X_t = x_t\right]\eqsp.
	$$
	Moreover, there exists a complete probability space $(\hat{\Omega}, \hat{\mathcal{F}}, \hat{\mathbb{P}})$ supporting a continuous, adapted $\rset^d$-valued process $(\hat{X}_t)_{t \in [0,T]}$, and a $d$-dimensional Brownian motion $(\hat{W}_t)_{t \in [0,T]}$, such that
	$$
		\hat{X}_t = \hat{X}_0 + \int_0^t \hat{b}(s, \hat{X}_s) \rmd s + \int_0^t \hat{\sigma}(s, \hat{X}_s) \rmd\hat{W}_s, \quad 0 \leq t \leq T\eqsp,
	$$
	and for each $t \in [0,T]$, the law of $\hat{X}_t$ under $\hat{\mathbb{P}}$ coincides with the law of $X_t$ under $\mathbb{P}$.
\end{proposition}

\begin{corollary}\label{cor:SDE-existence}
	Assume that $\pi$ has a finite first moment. Let $Y^\alpha = (Y_t^\alpha)_{t \in [0, T_{\text{gen}}]}$ be the observation process defined on a complete probability space $(\Omega, \mathcal{F}, \mathbb{P})$ by
	\begin{align}
		Y_t^\alpha = \alpha(t) X + \sigma B_t, \quad Y_0^\alpha = 0\eqsp,
	\end{align}
	where $\alpha : [0, T_{\text{gen}}] \to \rset_+$ is a differentiable, strictly increasing function, $\sigma > 0$, $X \sim \pi$, and $B$ is a standard Brownian motion defined on this space and independent of $X$. 

	Then, there exists a complete probability space and a process $\hat{Y}^\alpha = (\hat{Y}_t^\alpha)_{t \in [0, T_{\text{gen}}]}$ defined on it, which satisfies the stochastic differential equation
	\begin{align}
		\rmd \hat{Y}_t^\alpha = \dot{\alpha}(t) u_t^\alpha(\hat{Y}_t^\alpha) \rmd t + \sigma \rmd B_t, \quad \hat{Y}_0^\alpha = 0\eqsp,
	\end{align}
	where $u_t^\alpha(y) = \mathbb{E}[X | Y_t^\alpha = y]$, and such that for every $t \in [0, T_{\text{gen}}]$, the marginal law of $\hat{Y}_t^\alpha$ coincides with that of $Y_t^\alpha$.
\end{corollary}
\begin{proof}
	By construction,
	\begin{align}
		Y_t^\alpha = \alpha(t) X + \sigma B_t\eqsp,
	\end{align}
	so $Y^\alpha$ is an Itô process with drift $b_t = \dot{\alpha}(t) X$ and constant diffusion $\sigma$. It satisfies
	\begin{align}
		Y_t^\alpha = \int_0^t b_s \rmd s + \sigma B_t\eqsp.
	\end{align}
	To apply \Cref{prop:SDE-existence}, we verify the integrability condition
	\begin{align}
		\mathbb{E}\left[\int_0^t \left(\norm{b_s} + \sigma^2\right) \rmd s\right]
		&= \mathbb{E}[\norm{X}] \int_0^t \dot{\alpha}(s) \rmd s + \sigma^2 t \\
		&= \mathbb{E}[\norm{X}] (\alpha(t) - \alpha(0)) + \sigma^2 t < \infty \eqsp,
	\end{align}
	for all $t \in [0, T_{\text{gen}}]$, since $\mathbb{E}[\norm{X}] < \infty$ by assumption.
	By \Cref{prop:SDE-existence}, there exists a measurable function $\hat{b} : [0,T) \times \rset^d \to \rset^d$ such that
	\begin{align}
		\hat{b}(t, y) = \mathbb{E}[b_t \mid Y_t^\alpha = y] = \dot{\alpha}(t) \mathbb{E}[X \mid Y_t^\alpha = y] = \dot{\alpha}(t) u_t^\alpha(y)\eqsp.
	\end{align}
	Thus, there exists a process $\hat{Y}^\alpha$, defined on a possibly different probability space, satisfying the SDE
	\begin{align}
		\rmd \hat{Y}_t^\alpha = \dot{\alpha}(t) u_t^\alpha(\hat{Y}_t^\alpha) \rmd t + \sigma \rmd \hat{B}_t, \quad \hat{Y}_0^\alpha = 0\eqsp,
	\end{align}
	whose law at each time $t$ matches that of $Y_t^\alpha$.
\end{proof}

\begin{proposition}[Theorem 7.6 \& Remark 7.2.7. in \cite{liptser1977statistics}]\label{prop:SDE-unique}
	Let $T>0$, $\sigma>0$. Consider the following SDE
	\begin{align}\label{eq:SDE-unique}
		\rmd Y_t = h_t(Y_t) \rmd t + \sigma \rmd W_t, \quad Y_0 = 0 \eqsp,
	\end{align}
	where $\mathbb{P}(\int_{0}^T \norm{h_t(Y_t)}^2\rmd t < \infty )=1$. Then, \eqref{eq:SDE-unique} admits at most one weak solution.
\end{proposition}
The proof of this result lies in the fact that the distributions of solutions to the SDE \eqref{eq:SDE-unique} have the same density with respect to the distribution of the Brownian motion. We refer to \cite{liptser1977statistics} for the complete proof.

\begin{corollary} \label{cor:SDE-unique} Assume that $\pi$ has finite second order moment and that $t\to \dot{\alpha}(t)^2$ is integrable at time $t=0$. Then, the SDE defined in \eqref{eq:SDE-gen} by
	\begin{align}
		\rmd Y^\alpha_t = \dot{\alpha}(t) u^\alpha_t(Y^\alpha_t) \rmd t + \sigma \rmd B_t \eqsp, \eqsp Y^\alpha_0 = 0 \eqsp ,
	\end{align}
	has a unique weak solution.
\end{corollary}
\begin{proof} Consider $T\in (0, T_{\text{gen}})$. Since $\pi$ has its first order moment that is finite, we have existence of solutions by \Cref{cor:SDE-existence}. Consider a solution $Y^\alpha$ and denote $h_t=\dot{\alpha}(t)u_t^\alpha$. We have
	$\norm{h_t(Y^\alpha_t)}^2= \dot{\alpha}(t)^2 \norm{\mathbb{E}[X|Y^\alpha_t]}^2\leq \dot{\alpha}(t)^2 \mathbb{E}[\norm{X}^2|Y^\alpha_t]$ by Jensen's inequality. Then, using this inequality, we obtain that
	\begin{align}
		\textstyle\mathbb{E}[\int_0^T \norm{h_t(Y_t^\alpha)}^2\rmd t ]= \int_0^T \mathbb{E}[\norm{h_t(Y_t^\alpha)}^2]\rmd t\leq \int_0^T\dot{\alpha}(t)^2\mathbb{E}[\mathbb{E}[\norm{X}^2|Y^\alpha_t] ]\rmd t = \mathbb{E}[\norm{X}^2]\int_0^T \dot{\alpha}(t)^2 \rmd t
	\end{align}

	Combining the assumptions on $\pi$ and $\alpha$,
	it comes that $\mathbb{E}[\int_0^T \norm{h_t(Y^\alpha_t)}^2\rmd t ] < \infty$ using the inequality above, and therefore
	$\mathbb{P}(\int_{0}^T \norm{h_t(Y^\alpha_t)}^2\rmd t < \infty )=1$. We finally obtain the result by applying \Cref{prop:SDE-unique}.
\end{proof}

We emphasize that the extra assumption made on $\alpha$ in \Cref{cor:SDE-unique} is verified for the localization schemes Geom and Geom-$\infty$ presented in \Cref{subsec:guideline}.

\subsection{Alternative denoising approach to $\SLIPS$} \label{subsec:alternative_SLIPS}
We note here the score of the observation process can be expressed via an expectation over the posterior of the model, in a different manner than \eqref{eq:link_score_general}. Define $v_t^\alpha(y)=\int_{\rset^d}\nabla_x \log \pi(x) q_t^\alpha(x|y)\rmd x$, where $q_t^\alpha$ is the posterior density given in \eqref{eq:general_posterior}.

\begin{lemma} Consider the observation process defined in \eqref{eq:def_y_t-gen} with marginal distribution at time $t$ given by $p_t^\alpha$. Assume that $\log\pi$ is continuously differentiable on $\rset^d$ and that there exists $\varphi: \rset^d \to \rset_+$ such that $\int_{\rset^d} \varphi(z) \densityGaussian(z; 0, \sigma^2 t \Idd) \rmd z < \infty$ and for any $(z,y)\in \rset^d \times \rset^d$, we have $\norm{\nabla_y \pi (\{z+ y\}/\alpha(t))} \leq \varphi(z)$. Then, we have for any $y\in \rset^d$
	\begin{align}
		\nabla_y \log p_t^\alpha(y)=v_t^\alpha(y)/\alpha(t) \eqsp .
	\end{align}
\end{lemma}
\begin{proof}
	We recall that $\textstyle p^\alpha_t(y)=\int_{\rset^d}\densityGaussian(y;\alpha(t)x, \sigma^2 t\, \Idd)  \rmd \pi(x)$. Then, by change of variable $z=\alpha(t)x-y$, we have
	\begin{align}
		p^\alpha_t(y)\propto\int_{\rset^d} \pi(\{z+y\}/\alpha(t)) \densityGaussian(z; 0, \sigma^2 t \, \Idd )  \rmd z \eqsp ,
	\end{align}
	where the multiplicative constant does not depend on $y$. For any $y\in \rset^d$, denote by $\tilde{q}^\alpha_t(\cdot|y)$ the density defined up to a normalizing constant by $\tilde{q}^\alpha_t(z|y)\propto \pi(\{z+y\}/\alpha(t)) \densityGaussian(z; 0, \sigma^2 t \, \Idd )$. By combining the assumptions of the lemma with the result of \Cref{lemma:tweedie}-\eqref{eq:Tweedie} for this new expression of $p^\alpha_t$, we obtain that
	\begin{align}
		\nabla_y \log p_t^\alpha(y) & = \int_{\rset^d } \nabla_y \log \pi(\{z+y\}/\alpha(t)) \tilde{q}^\alpha_t(z|y)\rmd z                  \\
		                            & = \frac{1}{\alpha(t)}\int_{\rset^d } \nabla \log \pi(\{z+y\}/\alpha(t)) \tilde{q}^\alpha_t(z|y)\rmd z \\
		                            & = \frac{1}{\alpha(t)}\int_{\rset^d } \nabla_x \log \pi(x) q^\alpha_t(x|y)\rmd x \eqsp,
	\end{align}

	where we re-applied the change of variable $z=\alpha(t)x-y$ in the last equality.
\end{proof}

Therefore, the SDE \eqref{eq:SDE-gen} is strictly equivalent to the SDE
\begin{align}\label{eq:SDE-gen-2}
	\textstyle\rmd Y^\alpha_t = \frac{\dot{\alpha}(t)}{\alpha(t)}\{Y^\alpha_t + \frac{\sigma^2t}{\alpha(t)}v_t^\alpha(Y_t^\alpha)\} \rmd t + \sigma \rmd B_t \eqsp, \eqsp Y^\alpha_0 = 0 \eqsp .
\end{align}
Hence, to simulate from the observation process in a learning-free fashion, we can adopt a similar strategy to $\SLIPS$, that rather involves the SDE derived above. Given a SNR-adapted discretization $(t_k)_{k=1}^K$ of a time interval $[t_0,T]$, where $t_0>0$ and $K\geq 1$, it amounts to consider a sequence $\{(\tY^\alpha_{t_k}, V_{t_k}^\alpha)\}_{k=1}^K$, where $V_{t_k}^\alpha$ is a Monte Carlo estimator of  $v_{t_k}^\alpha(\tY_{t_k}^\alpha)$ and $\tY^\alpha_{t_k}$ is obtained by solving the SDE
\begin{align}\label{eq:SDE-gen-2-disc}
	\textstyle\rmd \tY^\alpha_t = \frac{\dot{\alpha}(t)}{\alpha(t)}\{\tY^\alpha_t + \frac{\sigma^2t_k}{\alpha(t_k)}V_{t_k}\} \rmd t + \sigma \rmd B_t \eqsp, t\in [t_k, t_{k+1}]  \eqsp .
\end{align}

In theory, this approach leads to the same level difficulty in the Monte Carlo estimation as $\SLIPS$, since the involved random poster densities are the same. In particular, this formulation does not bypass the ``duality of log-concavity'' explained in \Cref{subsection:duality}. The main difference with the approach presented in the main of our paper however lies in the integration of the SDE \eqref{eq:SDE-gen-2-disc}, which is not tractable at first sight. By combining the result of \Cref{lemma:ito} and the use of the stochastic Exponential Integrator scheme \cite{durmus2015quantitative}, we show in \Cref{subapp:integration-EI} how to solve this SDE for our localization schemes. We emphasize that this implementation is also available in our code.

Unfortunately, we found in our early experiments that this approach suffered from numerical unstability and showed higher variance than $\SLIPS$. We think that this is due to the evaluation of $\nabla \log \pi$ in our MC estimation for the early steps of the algorithm, and could be overcome by using the $\SLIPS$ recursion given in \eqref{eq:recursion-EM-est} instead. We leave this study for future work.

\section{Detailed computations for Gaussian mixtures}\label{app:gaussian_mixture}

In this section, we consider the special case where $\pi$ is a mixture of $N$ Gaussians with weights $(w_i)_{i=1}^N$, means $(\mathbf{m}_i)_{i=1}^N$ and covariance matrices $(\gamma_i^2 \, \Idd)_{i=1}^N$. Under this assumption, $p^{\alpha}_t$ can be explicitly written for any $y \in \rset^d$ and $t \in (0, T_{\text{gen}})$ as
\begin{align}
	p^{\alpha}_t(y) = \sum_{i=1}^N w_i \densityGaussian(y; \alpha(t) \mathbf{m}_i, t (g^2(t) \gamma_i^2 + \sigma^2) \, \Idd)\eqsp.
\end{align}
This means that the distribution of the observation process is itself a mixture of Gaussians with the same weights as $\pi$ but with means $(\alpha(t) \mathbf{m}_i)_{i=1}^N$ and covariance matrices $(t (g^2(t) \gamma_i^2 + \sigma^2) \, \Idd)_{i=1}^N$. This elementary result is obtained by applying the rule of linear combination of independent Gaussian random variables. The score of the observation process can also be computed by noticing that for all $y \in \rset^d$
\begin{align}
	\nabla \log p^{\alpha}_t(y) = \frac{\nabla \log p^{\alpha}_t(y)}{p^{\alpha}_t(y)} = -\frac{\sum_{i=1}^N w_i t^{-1} (g^2(t) \gamma_i^2 + \sigma^2)^{-1} (y - \alpha(t) \mathbf{m}_i) \densityGaussian(y; \alpha(t) \mathbf{m}_i, t (g^2(t) \gamma_i^2 + \sigma^2) \, \Idd)}{\sum_{i=1}^N w_i \densityGaussian(y; \alpha(t) \mathbf{m}_i (g^2(t) \gamma_i^2 + \sigma^2) \, \Idd)}\eqsp.
\end{align}

One can also compute the posterior distribution $q^{\alpha}_t$ for any $t \in (0, T_{\text{gen}})$ and $x,y \in \rset^d$ as
\begin{align}
	q^{\alpha}_t(x | y) & \propto \sum_{i=1}^N w_i \densityGaussian(x; \mathbf{m}_i, \gamma_i^2 \, \Idd) \densityGaussian(y; \sqrt{t} g(t) x, \sigma^2 t\,  \Idd)                                                                                                                                                                                                                                                                                                                                                                                                                                \\
	                    & = \sum_{i=1}^N w_i (\sqrt{t} g(t))^{-d} \densityGaussian(x; \mathbf{m}_i, \gamma_i^2 \, \Idd) \densityGaussian\left(x; \frac{y}{\sqrt{t} g(t)}, \frac{\sigma^2}{g^2(t)} \Idd\right)                                                                                                                                                                                                                                                                                                                                                                                    \\
	                    & = \sum_{i=1}^N \underbrace{w_i \alpha(t)^{-d} \densityGaussian\left(\mathbf{m}_i; \frac{y}{\alpha(t)}, \left(\gamma_i^2 + \frac{\sigma^2}{g^2(t)}\right) \Idd\right)}_{= \tilde{w}^{\alpha}_{t,y,i}} \densityGaussian\left(x; \underbrace{\left(\frac{\gamma_i^2 \sigma^2}{\sigma^2 + g^2(t) \gamma_i^2}\right) \left(\frac{\mathbf{m}_i}{\gamma_i^2} + \frac{y g(t)}{\sqrt{t} \sigma^2}\right)}_{= \mathbf{m}^{\alpha}_{t,y,i}}, \underbrace{\left(\frac{\gamma_i^2 \sigma^2}{\sigma^2 + g^2(t) \gamma_i^2}\right)}_{= (\gamma^{\alpha}_{t,y,i})^2} \Idd\right)\eqsp.
\end{align}
This shows that the posterior is itself a mixture of Gaussian distributions with weights $(w^{\alpha}_{t,y,i})_{i=1}^N$ where $w^{\alpha}_{t,y,i} = \tilde{w}^{\alpha}_{t,y,i} / \sum_{j=1}^N \tilde{w}^{\alpha}_{t,y,j}$, means $(\mathbf{m}^{\alpha}_{t,y,i})_{i=1}^N$ and covariance matrices $((\gamma^{\alpha}_{t,y,i})^2 \, \Idd)_{i=1}^N$.

Additionally, we can derive tight expressions for the constants $R$ and $\tau$ introduced in \Cref{ass:target} in the case where $\pi$ is a Gaussian mixture parameterized by $N = 2$, $w_1 = 1 - w_2 = w$, with $w\in (0,1)$,  $\mathbf{m}_2 = -\mathbf{m}_1 = a \, \mathbf{1}_d$\footnote{ We recall that $\mathbf{1}_d$ stands for the $d$-dimensional vector with all components equal to $1$.}, with $a>0$, and $\gamma_1 = \gamma_2 = \gamma$, with $\gamma>0$. In this case, $\pi$ verifies \Cref{ass:target}, where (i) $\mu$ is a mixture of two Dirac masses at $-a \, \mathbf{1}_d$ and $+a \, \mathbf{1}_d$ with respective weights $w$ and $1-w$ and (ii) $\tau = \gamma$. Moreover, for any random vector $U \sim \mu$, it holds
\begin{align}
	\mathbb{E}[\norm{U - \mathbf{m}_{\pi}}^2] & = \mathbb{E}[\norm{U - a(1 - 2w) \, \mathbf{1}_d}^2]                                                                       \\
	                                                 & = w \norm{-a \, \mathbf{1}_d - a(1 - 2w) \, \mathbf{1}_d}^2 + (a-w) \norm{a \, \mathbf{1}_d - a(1 - 2w) \, \mathbf{1}_d}^2 \\
	                                                 & = 4 w (1-w) a^2 d
\end{align}
Therefore, we obtain that $R = 2 \sqrt{w(1-w)} a$ in \Cref{ass:target}, where the inequality holds tightly.
Note also that $\Bar{\pi}$, defined as the distribution of $X-\mathbf{m}_\pi$ where $X\sim \pi$, is still a Gaussian mixture that verifies \Cref{ass:target} with same constants $R$ and $\tau$.

In the experiments conducted in \Cref{sec:xps}, \Cref{app:perfect_score} and \Cref{app:mcmc_ablation_study}, we will consider a rolling example given by the target distribution $\Bar{\pi}$ where $w=1/3$, $a=1.0$ and $\gamma^2=0.05$. In this case, we have $R^2=8/9$ and $\tau^2=0.05$, and the corresponding density is defined as
\begin{equation}\label{eq:target_mixture}
	\textstyle x\in \rset^d \mapsto \frac{2}{3} \densityGaussian\left(x; -\frac{2}{3}\, \mathbf{1}_{d}, 0.05 \, \mathrm{I}_{d}\right) + \frac{1}{3} \densityGaussian\left(x; \frac{4}{3}\, \mathbf{1}_{d}, 0.05 \,\mathrm{I}_{d}\right) \eqsp.
\end{equation}
Due to the tightness of $R$, the scalar variance of this target distribution verifies $R_\pi^2= d(R^2+\tau^2)$.

\section{Sampling via Stochastic Localization in an ideal setting} \label{app:perfect_score}

The goal of this section is to validate the claims from \Cref{subsec:discretization} about the minimization of the integration error. In this section, we work under the assumption that the score $\nabla \log p^{\alpha}_t$ is a known function and do not consider any MCMC method at all (even for the initialization). This setting removes entirely the estimation error that we deal with in \Cref{subsec:estimation} to solely focus on the integration error. For the initialization, we now consider an arbitrary $t_0 > 0$ with $Y^\alpha_{t_0}$ distributed as $\densityGaussian(0, \sigma^2 t_0 \, \Idd)$ (the best estimation that we have at the beginning of the SDE).  We will clarify 4 different points in this section.
\begin{enumerate}[wide, labelindent=0pt, label=(\alph*)]
	\item Exploring \emph{Exponential Integration} (EI) based schemes for SDE discretization;
	\item Analyzing the impact of the SNR-adapted discretization on the integration error;
	\item Exploring the limits of the Gaussian approximation $Y^\alpha_{t_0} \sim \densityGaussian(0, \sigma^2 t_0 \, \Idd)$;
	\item Exploring the impact of the computational budget.
\end{enumerate}
In the numerical examples presented below, we consider the target distribution defined in \eqref{eq:target_mixture} where $d=10$. We compute the exact score using the analytical formulas from \Cref{app:gaussian_mixture}. We choose to display two complementary results based on (i) the empirical Sliced Wasserstein distance \cite{bonneel2015sliced,nadjahi2019asymptotic}, which tells how local information on $\pi$ is recovered, and (ii) the error in estimating the weight of the first mode, which tells about global properties of the estimation. Moreover, for any SL scheme $\alpha$, the values of $t_0$ and $T$ will be taken so that the log-SNR evaluated at times $t_0$ and $t_K$ has the same value (the initial log-SNR is taken as $-4.0$ and the last log-SNR is taken as $5.0$). This enables fair comparison across different schedules.

\subsection{Exploring new integration schemes}\label{subapp:integration-EI}

By exploiting the relation given in \eqref{eq:link_score_general} between the denoiser function and the score of the observation process, which is now assumed to be tractable, the SDE \eqref{eq:SDE-gen} is equivalently defined on $[t_0,T]$ by
\begin{align}\label{eq:SDE-gen-perfect}
	\textstyle\rmd Y^\alpha_t = \frac{\dot{\alpha}(t)}{\alpha(t)}\{Y^\alpha_t +\sigma^2t \nabla \log p_t^\alpha(Y_t^\alpha)\} \rmd t + \sigma \rmd B_t \eqsp, \eqsp Y^\alpha_0 = 0 \eqsp .
\end{align}

Through this formulation, we have fully removed the problem of score estimation, but the issue of discretization error still remains to sample from $(Y^\alpha_t)_{t\in[t_0,T]}$ with \eqref{eq:SDE-gen-perfect}. Due to the divergence of the coefficient $t \mapsto \dot{\alpha}(t)/\alpha(t)$ at time $0$ (and time $T_{\text{gen}}$ in the finite-time setting), we propose to use the stochastic \emph{Exponential Integrator} (EI) scheme \cite{durmus2015quantitative} in this setting. Consider a time discretization of the interval $[t_0,T]$ defined by an increasing sequence of timesteps $(t_k)_{k=0}^K$ where $t_K=T$ and $K\geq 1$. Then, the EI scheme applied on the SDE \eqref{eq:SDE-gen-perfect} amounts to define a sequence of random variables $\{\tY^\alpha_{t_k}\}_{k=0}^K$ obtained by \emph{exactly} integrating the SDE defined for any $k \in \{0, \hdots, K-1\}$ by
\begin{align}\label{eq:SDE-gen-perfect-disc}
	\textstyle\rmd \tY^\alpha_t = \frac{\dot{\alpha}(t)}{\alpha(t)}\{\tY^\alpha_t + \sigma^2t_k \nabla \log p_{t_k}^\alpha(\tY_{t_k}^\alpha)\} \rmd t + \sigma \rmd B_t \eqsp, \eqsp t \in[t_k, t_{k+1}] \eqsp .
\end{align}
Although the exact integration is not guaranteed for a general schedule $\alpha$ as defined in \Cref{subsec:gen-framework}, our specific design of the denoising schedule $g(t)=t^{-1/2}\alpha(t)$ in the schemes Geom and Geom-$\infty$ provides tractable computations following the result of \Cref{lemma:ito}. We treat these cases separately below.

\paragraph{EI scheme combined with Geom-$\infty$ localization scheme.} In this setting, we recall that $g(t)=t^{\alpha_1 /2}$. Then, we have
\begin{align}
	\frac{\dot{\alpha}(t)}{\alpha(t)}=\frac{\alpha_1+1}{2t} \eqsp .
\end{align}
Note that $\dot{\alpha}(t)/\alpha(t)\to 0$ as $t\to 0$. Following \Cref{lemma:ito}, the sequence $\{\tY^\alpha_{t_k}\}_{k=0}^K$ obtained by the EI scheme is defined by the recursion
\begin{align}
	\tY^\alpha_{t_{k+1}}= \left(\frac{t_{k+1}}{t_k}\right)^{\frac{\alpha_1+1}{2}}\tY^\alpha_{t_k} + \left\{\left(\frac{t_{k+1}}{t_k}\right)^{\frac{\alpha_1+1}{2}} -1\right\}\sigma^2t_k \nabla \log p_{t_k}^\alpha(\tY_{t_k}^\alpha) + \sigma\sqrt{\frac{t_{k+1}}{\alpha_1 t_{k}^{\alpha_1}}(t_{k+1}^{\alpha_1}-t_{k}^{\alpha_1})}Z_{k+1} \eqsp .
\end{align}
where $(Z_k)_{k=1}^K$ is distributed according to the standard centered Gaussian distribution. In the \emph{standard} case, \ie, $\alpha_1=1$, this simplifies as
\begin{align}
	\textstyle\tY^\alpha_{t_{k+1}}= \frac{t_{k+1}}{t_k}\tY^\alpha_{t_k} + \{\frac{t_{k+1}}{t_k} -1\}\sigma^2t_k \nabla \log p_{t_k}^\alpha(\tY_{t_k}^\alpha) + \sigma\sqrt{\frac{t_{k+1}}{ t_{k}}(t_{k+1}-t_{k})}Z_{k+1} \eqsp .
\end{align}

\paragraph{EI scheme combined with Geom localization scheme.}In this second setting, we recall that $g(t)=t^{\alpha_1 /2}(1-t)^{-\alpha_2/2}$ ($T_{\text{gen}}=1$). Then, we have
\begin{align}
	\frac{\dot{\alpha}(t)}{\alpha(t)}=\frac{\alpha_1+1}{2t} + \frac{\alpha_2}{2(1-t)} \eqsp .
\end{align}
Note that $\dot{\alpha}(t)/\alpha(t)\to 0$ as $t\to 0$ and $t\to 1$. Following \Cref{lemma:ito}, the sequence $\{\tY^\alpha_{t_k}\}_{k=0}^K$ obtained by the EI scheme is defined by the recursion
\begin{align}
	\tY^\alpha_{t_{k+1}} & = \left(\frac{t_{k+1}}{t_k}\right)^{\frac{\alpha_1+1}{2}}\left(\frac{1-t_{k}}{1-t_{k+1}}\right)^{\frac{\alpha_2}{2}}\tY^\alpha_{t_k}                                                                 \\
	                     & \quad + \left\{\left(\frac{t_{k+1}}{t_k}\right)^{\frac{\alpha_1+1}{2}}\left(\frac{1-t_{k}}{1-t_{k+1}}\right)^{\frac{\alpha_2}{2}} -1\right\}\sigma^2t_k \nabla \log p_{t_k}^\alpha(\tY_{t_k}^\alpha) \\
	                     & \quad + \frac{\sigma}{\sqrt{\alpha_1}}\frac{t_{k+1}^{\frac{\alpha_1+1}{2}}}{(1-t_{k+1})^{\frac{\alpha_2}{2}}}
	\sqrt{ {}_2 \mathrm{F}_1(-\alpha_1, -\alpha_2, 1-\alpha_1, t_k)t_k^{-\alpha_1} - {}_2 \mathrm{F}_1(-\alpha_1, -\alpha_2, 1-\alpha_1, t_{k+1})t_{k+1}^{-\alpha_1} }Z_{k+1} \eqsp ,
\end{align}
where ${}_2 \mathrm{F}_1$ denotes the hypergeometric function, see \citep[Chapter 15]{olver2010nist}, and $(Z_k)_{k=1}^K$ is distributed according to the standard centered Gaussian distribution.

When $(\alpha_1,\alpha_2)=(1,1)$, this simplifies as
\begin{align}
	\tY^\alpha_{t_{k+1}} & = \textstyle (\frac{t_{k+1}}{t_k})(\frac{1-t_{k}}{1-t_{k+1}})^{\frac{1}{2}}\tY^\alpha_{t_k}                                                      \\
	                     & \textstyle\quad + \{(\frac{t_{k+1}}{t_k})(\frac{1-t_{k}}{1-t_{k+1}})^{\frac{1}{2}} -1\}\sigma^2t_k \nabla \log p_{t_k}^\alpha(\tY_{t_k}^\alpha)  \\
	                     & \textstyle\quad + \sigma\sqrt{\frac{t_{k+1}^2}{1-t_{k+1}}\log(\frac{t_k}{t_{k+1}}) + \frac{t_{k+1}(t_{k+1}-t_k)}{t_k(1-t_{k+1})}}Z_{k+1} \eqsp .
\end{align}

When $(\alpha_1,\alpha_2)=(2,1)$, this simplifies as
\begin{align}
	\tY^\alpha_{t_{k+1}} & \textstyle= (\frac{t_{k+1}}{t_k})^{\frac{3}{2}}(\frac{1-t_{k}}{1-t_{k+1}})^{\frac{1}{2}}\tY^\alpha_{t_k}                                                      \\
	                     & \textstyle\quad + \{(\frac{t_{k+1}}{t_k})^{\frac{3}{2}}(\frac{1-t_{k}}{1-t_{k+1}})^{\frac{1}{2}} -1\}\sigma^2t_k \nabla \log p_{t_k}^\alpha(\tY_{t_k}^\alpha) \\
	                     & \textstyle\quad + \sigma t_{k+1}(\frac{t_{k+1}/t_k -1}{1-t_{k+1}})^{\frac{1}{2}}
	\sqrt{ \frac{t_{k}+t_{k+1}}{2t_k t_{k+1}}-1 }Z_{k+1} \eqsp .
\end{align}

When $(\alpha_1,\alpha_2)=(1,2)$, this simplifies as
\begin{align}
	\tY^\alpha_{t_{k+1}} & \textstyle= (\frac{t_{k+1}}{t_k})^{\frac{1}{2}}(\frac{1-t_{k}}{1-t_{k+1}})^{\frac{3}{2}}\tY^\alpha_{t_k}                                                      \\
	                     & \textstyle\quad + \{(\frac{t_{k+1}}{t_k})^{\frac{1}{2}}(\frac{1-t_{k}}{1-t_{k+1}})^{\frac{3}{2}} -1\}\sigma^2t_k \nabla \log p_{t_k}^\alpha(\tY_{t_k}^\alpha) \\
	                     & \textstyle\quad + \frac{\sigma}{1-t_{k+1}}
	\sqrt{ (t_{k+1}-t_k)\{t_{k+1}^2+ \frac{t_{k+1}}{t_k}\} + 2 t_{k+1}^2\log(\frac{t_{k}}{t_{k+1}})\} }Z_{k+1} \eqsp .
\end{align}

\subsection{Studying the impact of the SNR-adapted time discretization}

\begin{figure}[t!]
	\centering
	\includegraphics[width=\linewidth]{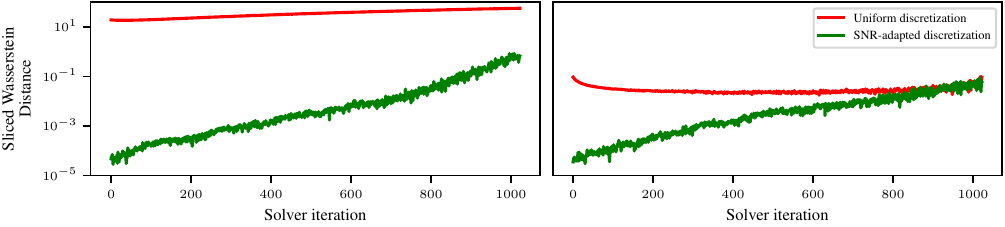}
	\caption{Sliced Wasserstein distance between the true samples from $p^{\alpha}_{t_k}$ and the empirical distribution of $\tY^{\alpha}_{t_k}$ obtained by EI scheme combined with uniform time discretization (\tcmv{green}) and SNR-adapted discretization (\tcmr{red}) for the Standard (\textit{left}) and Geom(1,1) (\textit{right}) schemes.}
	\label{fig:app:adaptive_disc_both_wasserstein}
\end{figure}

In \Cref{subsec:discretization}, we suggested to take a SNR-adapted time discretization by $(t_k)_{k=0}^K$ such that
\begin{align}
	\lsnr(t_k) = \lsnr(t_0) + \Delta_{\text{SNR}} k\eqsp,
\end{align}
with $t_0 > 0$ and $\Delta_{\text{SNR}} = (\lsnr(T) - \lsnr(t_0)) / K$. \Cref{fig:app:adaptive_disc_both_wasserstein} shows that the SNR-adapted discretization efficiently reduces the integration error for both schedules Geom and Geom-$\infty$. Note, the impact on the schedule Geom(1,1) is more moderated as the uniform initialization already splits the curve into moderated log-SNR increments (see \Cref{fig:adap_dist}) because of a slower rate near $t_0$.

\subsection{Studying the impact of the Gaussian approximation in the initialization} \label{subsec:gaussian_init}

\begin{figure}[t]
	\centering
	\includegraphics[width=0.33\linewidth]{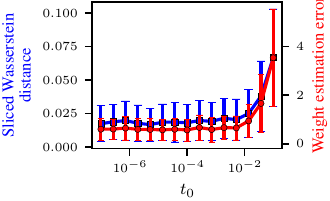}
	\hfill
	\includegraphics[width=0.66\linewidth]{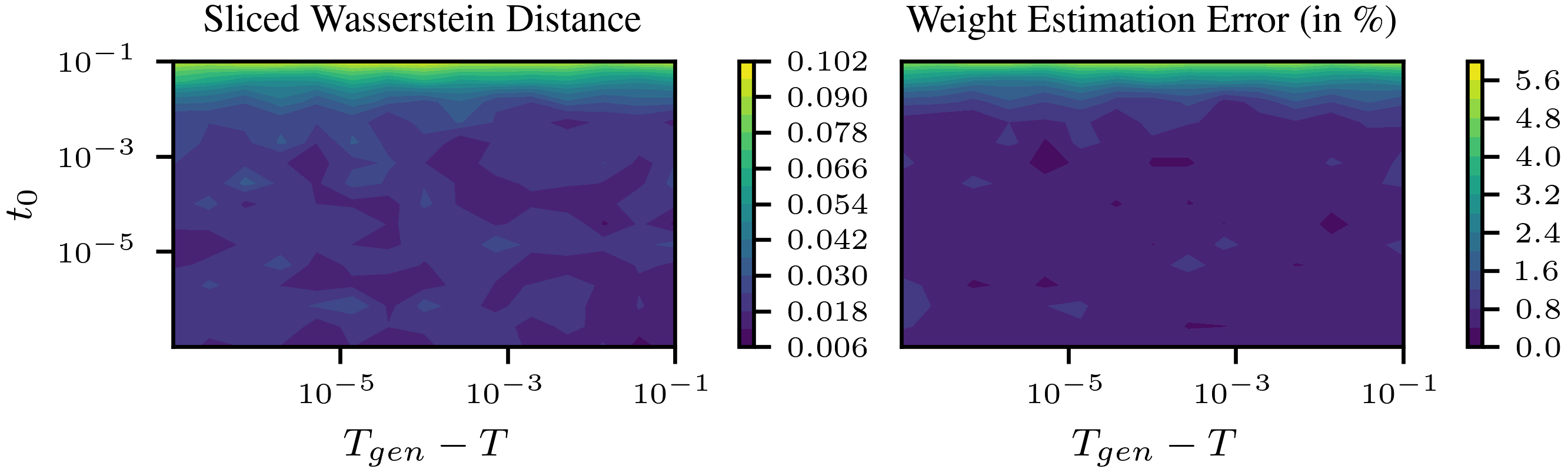}
	\caption{Impact of $t_0$ with different schedules measured with the relative weight estimation error and the sliced Wasserstein distance. \textbf{Left}: Impact of $t_0$ in the standard scheme. \textbf{Right}: Impact of $t_0$ and $(T_{\text{gen}} - T)$ in the Geom(1,1) scheme.}
	\label{fig:app:impact_t_0}
\end{figure}

In this section only, we replace the target distribution considered in \eqref{eq:target_mixture} by the mixture of two Gaussian distributions $\densityGaussian(-3\, \mathbf{1}_{10}, \Sigma)$ and $\densityGaussian(3\, \mathbf{1}_{10}, \Sigma)$, where $\Sigma=0.05\,\mathrm{I}_{10}$, with weights respectively given by $2/3$ and $1/3$. Therefore, $\pi$ has non-zero mean and high variance, which provides a challenging setting for our Gaussian approximation at initialization given by $Y^\alpha_{t_0}\sim\densityGaussian(0, \sigma^2 t_0 \, \Idd)$. Without information on the target, this is the best estimation that we have when $t_0$ is close to $0$. \Cref{fig:app:impact_t_0} shows that only high values of $t_0$ (\ie~above $10^{-2}$) degrades the performance in our localization schemes. This underlines the need of running the Langevin-within-Langevin correction procedure of this approximation as explained in \Cref{subsection:duality}. Additionally, note that \Cref{fig:app:impact_t_0} (\textit{right}) shows that taking $(T_{\text{gen}} - T)$ low in the finite time setting does not degrade performance, which was not obvious from the log-SNR shape near $T_{\text{gen}}$ in \Cref{fig:log-SNR}.

\subsection{Studying the impact of the number of discretization steps}

\begin{figure}[t]
	\centering
	\includegraphics[width=0.90\linewidth]{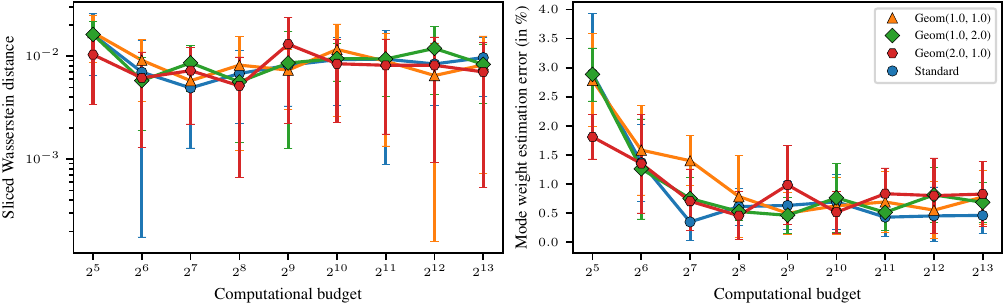}
	\caption{Impact of the computational budget $K$ with different schemes - The sampling error is computed either with the Sliced Wasserstein distance (\textit{left}) or the relative weight error (\textit{right}).}
	\label{fig:app:impact_T_K}
\end{figure}

The bottom row of \Cref{fig:app:impact_T_K} shows that, in the ideal case where the score is known, using the SNR-adapted discretization makes all the different schemes equivalent\footnote{We recall that we choose $t_0$ and $T$ so that the starting and ending levels of SNR are the same across different schemes.}. Moreover, we see that the sampling error quickly stabilizes with medium budgets which highlights that the integration error was successfully minimized.

\section{Theoretical results on the duality of log-concavity} \label{app:theory}

In this section, we detail the results on log-concavity provided in \Cref{subsec:estimation} and \Cref{subsection:duality}. For sake of readability, those are stated with a general $\sigma>0$. Our first result provides uniform upper bounds on the Hessian of the log-density of the observation process and its corresponding log-posterior density.

\begin{lemma}\label{lemma:concavity} Assume \Cref{ass:target}. Let $t\in (0, T_{\text{gen}})$. We recall that $p_t^\alpha$ stands for the marginal distribution at time $t$ of the observation process defined by \eqref{eq:def_y_t-gen}, while $q_t^\alpha$ stands for the corresponding posterior density. Under regularity assumptions on $\pi$ detailed in the proof of the lemma, we have for any $(x,y)\in \rset^d\times \rset^d$ that
	\begin{align}
		\nabla^2_y \log p_t^\alpha(y)   & \preccurlyeq \zeta_p(t) \, \Idd \eqsp, \quad \text{where } \zeta_p(t)=\textstyle\frac{\alpha(t)^2 d R^2}{(\alpha(t)^2\tau^2 + \sigma^2 t)^2} -\frac{1}{\alpha(t)^2\tau^2 + \sigma^2 t} \eqsp ,\label{eq:zeta_p_t} \\
		\nabla^2_x \log q_t^\alpha(x|y) & \preccurlyeq \zeta_q(t) \, \Idd \eqsp , \quad \text{where } \zeta_q(t)=\textstyle\frac{d R^2}{\tau^4} -\frac{1}{\tau^2}-\frac{g(t)^2}{\sigma^2} \eqsp .\label{eq:zeta_q_t}
	\end{align}
	In particular, $\zeta_q$ is a strictly decreasing function on $(0, T_{\text{gen}})$.
\end{lemma}

\begin{proof} We begin by proving \eqref{eq:zeta_p_t}. Consider the stochastic process $(Y_t^\alpha)_{t\in [0, T_{\text{gen}})}$, given in \eqref{eq:def_y_t-gen}, and defined by $Y^\alpha_t= \alpha(t)X + \sigma W_t$ where $X\sim \pi$ and $(W_t)_{t\geq 0}$ is a standard Brownian motion.
	Due to \Cref{ass:target}, $X$ can be written as $X= U + G$ where $U\sim \mu$, with $\mathbb{E}[\norm{U-\mathbf{m}_\pi}^2]\leq dR^2$,
	and $G\sim \densityGaussian(0, \tau^2\, \Idd)$.
	We thus have the following decomposition
	\begin{align}
		\textstyle p_t^\alpha(y)=\int_{\rset^d}p_t^\alpha(y|u)\rmd \mu(u) \eqsp,
	\end{align}
	where $p_t^\alpha(\cdot|u)$ is the conditional density of $Y_t^\alpha$ given $U=u\in \rset^d$ defined as
	\begin{align}\label{eq:p_t_y_given_u}
		p_t^\alpha(y|u)=\densityGaussian(y; \alpha(t) u, \{\alpha(t)^2\tau^2 + \sigma^2 t\}\, \Idd) \eqsp .
	\end{align}

	Then, it comes that
	\begin{align}
		\nabla_y \log p_t^\alpha(y|u)=\frac{-y+\alpha(t)u}{\alpha(t)^2\tau^2 + \sigma^2 t} \eqsp, \eqsp \nabla_y^2\log p_t^\alpha(y|u)=-\frac{1}{\alpha(t)^2\tau^2 + \sigma^2 t}\Idd \eqsp .
	\end{align}

	We now make the \underline{technical assumption} that for any $k\in \{1,2\}$, there exists $\varphi_{p,k}:\rset^d \to \rset_+$ such that $\int_{\rset^d}\varphi_{p,k}(u)\rmd \mu(u) < \infty$ and for any $(u, y)\in \rset^d \times \rset^d$, we have $\norm{\nabla_y^k p_t^\alpha(y|u)}\leq \varphi_{p,k}(u)$. By combining this assumption with the result of \Cref{lemma:tweedie}-\eqref{eq:Tweedie_2}, we obtain that
	\begin{align}
		\nabla^2_y \log p_t^\alpha(Y_t^\alpha) & = -\frac{1}{\alpha(t)^2\tau^2 + \sigma^2 t}\Idd + \mathrm{Cov}\left[\frac{-Y_t^\alpha+\alpha(t)U}{\alpha(t)^2\tau^2 + \sigma^2 t}\middle | Y_t^\alpha \right] \\
		                                       & = -\frac{1}{\alpha(t)^2\tau^2 + \sigma^2 t}\Idd + \frac{\alpha(t)^2}{(\alpha(t)^2\tau^2 + \sigma^2 t)^2}\mathrm{Cov}[U | Y_t^\alpha] \eqsp .
	\end{align}
	Since $\mathbb{E}[\norm{U-\mathbf{m}_\pi}^2]\leq dR^2$, we have by Cauchy-Schwartz inequality that $\mathrm{Cov}[U | Y_t^\alpha] \preccurlyeq dR^2\,\Idd$\footnote{Note that this upper bound may be loose, since it does not depend on $Y_t^\alpha$.}. Then, we obtain the bound \eqref{eq:zeta_p_t} as $\nabla^2 \log p_t^\alpha$ is continuous and $p_t^\alpha$ is positive on $\rset^d$.

	We now prove \eqref{eq:zeta_q_t}. We recall from \eqref{eq:general_posterior} that $q_t^\alpha (x|y) \propto \pi(x) \densityGaussian(x; y/\alpha(t), \sigma^2/g(t)^2\, \Idd)$. Therefore
	\begin{align}
		\nabla^2_x \log q_t^\alpha (X|Y_t^\alpha)= \nabla^2_x \log \pi(X) -\frac{g(t)^2}{\sigma^2}\Idd \eqsp .
	\end{align}
	In particular, we have $\pi(x)=\int_{\rset^d}\densityGaussian(x;u, \tau^2\, \Idd)\rmd \mu(u)$.

	We now make the \underline{technical assumption} that for any $k\in \{1,2\}$, there exists $\varphi_{q,k}:\rset^d \to \rset_+$ such that $\int_{\rset^d}\varphi_{q,k}(u)\rmd \mu(u) < \infty$ and for any $(u, x)\in \rset^d \times \rset^d$, we have $\norm{\nabla_x^k \densityGaussian(x;u, \tau^2\, \Idd)}\leq \varphi_{p,k}(u)$. Then, by using again the result of \Cref{lemma:tweedie}-\eqref{eq:Tweedie_2}, we obtain that
	\begin{align}
		\nabla^2_x \log \pi(X)= -\frac{1}{\tau^2}\Idd + \frac{1}{\tau^4} \mathrm{Cov}[U|X] \eqsp ,
	\end{align}
	Since $\mathbb{E}[\norm{U-\mathbf{m}_\pi}^2]\leq dR^2$, we also have by Cauchy-Schwartz inequality that $\mathrm{Cov}[U|X] \preccurlyeq dR^2\,\Idd$\footnote{Note that this upper bound may be loose, since it does not depend on $X$.}. We obtain \eqref{eq:zeta_q_t} with similar reasoning as before.
\end{proof}

Note that the upper bounds obtained in \Cref{lemma:concavity} can be made tighter in the case where the target distribution is a Gaussian mixture, as shown in \Cref{lemma:gaussian_concavity}.

\begin{lemma} \label{lemma:gaussian_concavity} Let $a>0$, $\gamma>0$ and $w\in (0,1)$. Consider the target distribution $\pi$ defined as the mixture of the Gaussian distributions $\densityGaussian(-a \, \mathbf{1}_d, \gamma^2\, \Idd)$ and $\densityGaussian(+a \, \mathbf{1}_d, \gamma^2\, \Idd)$, with respective weights $w$ and $1-w$. Then for any $(x,y)\in \rset^d\times \rset^d$, we have
	\begin{align}
		\nabla^2_y \log p_t^\alpha(y)   & = \textstyle\frac{2\alpha(t)^2 a^2}{(\alpha(t)^2\gamma^2 + \sigma^2 t)^2(1+ \cosh{g(t,y)})}\,\mathbf{1}_d \mathbf{1}_d^\top -\frac{1}{\alpha(t)^2\gamma^2 + \sigma^2 t}\,\Idd \eqsp, \label{eq:hessian_gaussian_p}\\
		& \eqsp \text{with }g(t,y)=\frac{2 \alpha(t) a y^\top \mathbf{1}_d}{\alpha(t)^2\gamma^2 + \sigma^2 t} + \log(\frac{1}{w}-1) \nonumber \\
		\nabla^2_x \log q_t^\alpha(x|y) & = \textstyle\frac{2 a^2 }{\gamma^4(1+ \cosh{f(t,x)})} \,\mathbf{1}_d \mathbf{1}_d^\top -\frac{1}{\gamma^2}\, \Idd-\frac{g(t)^2}{\sigma^2} \, \Idd \eqsp ,\label{eq:hessian_gaussian_q}\\
		& \eqsp \text{with }f(t,x)=\frac{2  a x^\top \mathbf{1}_d}{\gamma^2}+ \log(\frac{1}{w}-1) \nonumber
	\end{align}
\end{lemma}

The proof of this result is inspired from the derivation of \citep[Eq. (4.6)]{saremi2023chain}.

\begin{proof} As explained in \Cref{app:gaussian_mixture}, $\pi$ verifies \Cref{ass:target}, where $\tau=\gamma$ and $\mu$ is a mixture of two Dirac masses at $-a\, \mathbf{1}_d$ and $a\, \mathbf{1}_d$ with respective weights $w$ and $1-w$, whose density is given by
	$p(u)=w\mathbbm{1}_{-a\, \mathbf{1}_d}(u)+ (1-w)\mathbbm{1}_{a\, \mathbf{1}_d}(u)$.

	We recall that the stochastic process $(Y_t^\alpha)_{t\in [0, T_{\text{gen}})}$, given in \eqref{eq:def_y_t-gen}, is defined by $Y^\alpha_t= \alpha(t)X + \sigma W_t$ where $X\sim \pi$ and $(W_t)_{t\geq 0}$ is a standard Brownian motion. Due to \Cref{ass:target}, $X$ can be written as $X= U + G$ where $U\sim \mu$ and $G\sim \densityGaussian(0, \gamma^2\, \Idd)$.

	We begin by proving \eqref{eq:hessian_gaussian_p}. Following the first part of the proof of \Cref{lemma:concavity}, we have
	\begin{align}
		\nabla^2_y \log p_t^\alpha(y)= -\frac{1}{\alpha(t)^2\gamma^2 + \sigma^2 t}\Idd + \frac{\alpha(t)^2}{(\alpha(t)^2\gamma^2 + \sigma^2 t)^2}\mathrm{Cov}[U_y] \eqsp ,
	\end{align}
	where $U_y$ is a random vector distributed according to $p_t^\alpha(\cdot|y)$, the conditional distribution of $U$ given $Y_t^\alpha=y\in \rset^d$, whose density is given by
	\begin{align}
		p_t^\alpha(u|y)\propto p(u) p_t^\alpha(y|u) \eqsp ,
	\end{align}
	where $p_t^\alpha(y|u)$ is defined in \eqref{eq:p_t_y_given_u}. Therefore, we obtain that
	\begin{align}
		p_t^\alpha(u|y)\propto \underbrace{w \exp\left(-\frac{\norm{y+\alpha(t)a \mathbf{1}_d}^2}{2 \{\alpha(t)^2\gamma^2 + \sigma^2 t\}}\right)}_{w_1(t,y)}\mathbf{1}_{-a\, \mathbf{1}_d}(u) + \underbrace{(1-w)\exp\left(-\frac{\norm{y-\alpha(t)a \mathbf{1}_d}^2}{2\{\alpha(t)^2\gamma^2 + \sigma^2 t\}}\right)}_{w_2(t,y)}\mathbf{1}_{a\, \mathbf{1}_d}(u) \eqsp .
	\end{align}
	Hence, $p_t^\alpha(\cdot|y)$ is a mixture of two Dirac masses at $-a\, \mathbf{1}_d$ and $a\, \mathbf{1}_d$ with respective weights $\tilde{w}_1(t,y)=w_1(t,y)/\{w_1(t,y)+ w_2(t,y)\}$ and $\tilde{w}_2(t,y)=w_2(t,y)/\{w_1(t,y)+ w_2(t,y)\}$. Therefore, we get that
	\begin{align}
		\mathrm{Cov}[U_y] & = \PE[U_y U_y^\top]- \PE[U_y] \PE[U_y]^\top                                                                                                                                             \\
		                  & = \tilde{w}_1(t,y) a^2 \mathbf{1}_d \mathbf{1}_d^\top + \tilde{w}_2(t,y) a^2 \mathbf{1}_d \mathbf{1}_d^\top - (\tilde{w}_2(t,y)-\tilde{w}_1(t,y))^2  a^2 \mathbf{1}_d \mathbf{1}_d^\top \\
		                  & = \{1-(\tilde{w}_2(t,y)-\tilde{w}_1(t,y))^2\}a^2 \mathbf{1}_d \mathbf{1}_d^\top \eqsp .
	\end{align}
	In particular, we have
	\begin{align}
		\tilde{w}_2(t,y)-\tilde{w}_1(t,y)= \frac{w_2(t,y)-w_1(t,y)}{w_2(t,y)+w_1(t,y)}= \frac{\exp(\log w_2(t,y))-\exp(\log w_1(t,y))}{\exp(\log w_2(t,y))-\exp(\log w_1(t,y))} \eqsp ,
	\end{align}
	where $\log w_1(t,y)=-\frac{\norm{y+\alpha(t)a \mathbf{1}_d}^2}{2\{\alpha(t)^2\gamma^2 + \sigma^2 t\}} + \log(w)$ and $\log w_2(t,y)=-\frac{\norm{y-\alpha(t)a \mathbf{1}_d}^2}{2 \{\alpha(t)^2\gamma^2 + \sigma^2 t\}} +\log(1-w) $.
	Then, it comes that
	\begin{align}
		1-(\tilde{w}_2(t,y)-\tilde{w}_1(t,y))^2 & = 1- \tanh\left(\frac{\log w_2(t,y)-\log w_1(t,y)}{2}\right) \\
		                                        & = \frac{2}{1+ \cosh(\log w_2(t,y)-\log w_1(t,y))} \eqsp .
	\end{align}
	Moreover, we get that
	\begin{align}
		\log w_2(t,y)-\log w_1(t,y)= \frac{2 \alpha(t)a y^\top \mathbf{1}_d}{\alpha(t)^2\gamma^2 + \sigma^2 t} + \log \left(\frac{1}{w}-1\right) \eqsp .
	\end{align}
	Denote $g(t,y)= \log w_2(t,y)-\log w_1(t,y)$. We finally obtain \eqref{eq:hessian_gaussian_p} by combining previous computations.

	We now prove \eqref{eq:hessian_gaussian_q}.  Following the second part of the proof of \Cref{lemma:concavity}, we have
	\begin{align}
		\nabla^2_x \log q_t^\alpha(x|y)= -\frac{1}{\gamma^2}\Idd -\frac{g(t)^2}{\gamma^2}\Idd + \frac{1}{\gamma^4} \mathrm{Cov}[U_x] \eqsp ,
	\end{align}
	where $U_x$ is a random vector distributed according to $\mu(\cdot|x)$, the conditional distribution of $U$ given $X=x\in \rset^d$, whose density is given by
	\begin{align}
		\mu(u|x)\propto \densityGaussian(x;u, \gamma^2 \, \Idd) \mu(u) \eqsp .
	\end{align}
	Therefore, we obtain that
	\begin{align}
		\mu(u|x)\propto \underbrace{w \exp\left(-\frac{\norm{x+a \mathbf{1}_d}^2}{2 \gamma^2 }\right)}_{w_1(t,x)}\mathbf{1}_{-a\, \mathbf{1}_d}(u) + \underbrace{(1-w)\exp\left(-\frac{\norm{x-a \mathbf{1}_d}^2}{2\gamma^2 }\right)}_{w_2(t,x)}\mathbf{1}_{a\, \mathbf{1}_d}(u) \eqsp .
	\end{align}
	Hence, $\mu(\cdot|x)$ is a mixture of two Dirac masses at $-a\, \mathbf{1}_d$ and $a\, \mathbf{1}_d$ with respective weights $\tilde{w}_1(t,x)=w_1(t,x)/\{w_1(t,x)+ w_2(t,x)\}$ and $\tilde{w}_2(t,x)=w_2(t,x)/\{w_1(t,x)+ w_2(t,x)\}$. Therefore, we get that
	\begin{align}
		\mathrm{Cov}[U_x] & = \PE[U_x U_x^\top]- \PE[U_x] \PE[U_x]^\top                                                                                                                                             \\
		                  & = \tilde{w}_1(t,x) a^2 \mathbf{1}_d \mathbf{1}_d^\top + \tilde{w}_2(t,x) a^2 \mathbf{1}_d \mathbf{1}_d^\top - (\tilde{w}_2(t,x)-\tilde{w}_1(t,x))^2  a^2 \mathbf{1}_d \mathbf{1}_d^\top \\
		                  & = \{1-(\tilde{w}_2(t,x)-\tilde{w}_1(t,x))^2\}a^2 \mathbf{1}_d \mathbf{1}_d^\top \eqsp .
	\end{align}
	Similarly to the computations derived above, we obtain that
	\begin{align}
		1-(\tilde{w}_2(t,x)-\tilde{w}_1(t,x))^2= \frac{2}{1+ \cosh f(t,x)} \eqsp ,
	\end{align}
	where
	\begin{align}
		f(t,x)= \log w_2(t,x) - \log w_1(t,x)= \frac{2  a x^\top \mathbf{1}_d}{\gamma^2}+ \log\left(\frac{1}{w}-1\right) \eqsp ,
	\end{align}
	which finally leads to \eqref{eq:hessian_gaussian_q}.

\end{proof}
We provide below a result, that combines formal versions of \Cref{th:q-log-concave} and \Cref{th:p-log-concave}.
\begin{proposition} \label{prop:log_concavity}  Assume \Cref{ass:target}. Denote by $g^{-1}$ the inverse function of the denoising schedule $g$.

	If $dR^2 >\tau^2$, define
	\begin{align}
		\textstyle t_{\mathrm{p}} = g^{-1}\left(\frac{\sigma}{\sqrt{dR^2-\tau^2}}\right) \quad \text{and} \quad
		t_{\mathrm{q}} = g^{-1}\left(\frac{\sigma\sqrt{dR^2-\tau^2}}{\tau^2}\right) \eqsp ,
	\end{align}
	otherwise, define $t_{\mathrm{p}}=T_{\text{gen}}$ and $t_{\mathrm{q}}=0$. Then,
	\begin{enumerate}[wide, labelindent=0pt, label=(\alph*)]
		\item \label{item:q_concave} for any $y\in \rset^d$, $q^\alpha_t(\cdot|y)$ is strongly log-concave for $t\in (t_{\mathrm{q}}, T_{\text{gen}})$ and gets more log-concave as $t$ increases,
		\item \label{item:p_concave} $p_t^\alpha$ is strongly log-concave for $t\in (0,t_{\mathrm{p}})$.
	\end{enumerate}
\end{proposition}

\begin{proof} Consider the expressions of $t_{\mathrm{p}}$ and $t_{\mathrm{q}}$ given above. Note that they are well defined in the case where $dR^2 >\tau^2$ since $g$ is a bijection from $[0,T_{\text{gen}} )$ to $\rset_+$. We begin with the proof of the result \ref{item:q_concave}. Let $t\in (0, T_{\text{gen}})$. Following \Cref{lemma:concavity}-\eqref{eq:zeta_q_t}, for any $y\in \rset^d$, $q_t^\alpha(\cdot|y)$ is strongly log concave if
	\begin{align}
		\zeta_q(t)<0 \iff g(t)^2 > \frac{\sigma^2 (dR^2-\tau^2)}{\tau^4} \iff t > t_{\mathrm{q}} \eqsp .
	\end{align}
	Moreover, $q_t^\alpha$ gets more log-concave as $t$ increases as $\zeta_q$ is a strictly decreasing function, see \Cref{lemma:concavity}. We now give the proof of the result \ref{item:p_concave}. Let $t\in (0, T_{\text{gen}})$. Following \Cref{lemma:concavity}-\eqref{eq:zeta_p_t}, $p_t^\alpha$ is strongly log concave if
	\begin{align}
		\zeta_p(t)<0 \iff \frac{\alpha(t)^2 d R^2}{\alpha(t)^2\tau^2 +\sigma^2t } <1 \iff g(t)^2 (dR^2-\tau^2)<\sigma^2 \iff t <  t_{\mathrm{p}} \eqsp .
	\end{align}
\end{proof}

Hence, this result shows that the condition $dR^2 >\tau^2$ is restrictive on the log-concavity of the marginal distribution and the posterior of the localization model. In terms of Gaussian mixtures, this condition can be interpreted as having the distance between the modes that is larger than the variance of the modes. We now restate \Cref{th:duality} and give its proof.

\begin{proposition} \label{prop:tq_tp} Assume \Cref{ass:target}, where $dR^2 < 2\tau^2$. Then, $t_{\mathrm{q}}<t_{\mathrm{p}}$, where $t_{\mathrm{q}}$ and $t_{\mathrm{p}}$ are defined in \Cref{prop:log_concavity}.
\end{proposition}
\begin{proof} If $dR^2 \leq \tau^2$, this result is directly obtained since $t_{\mathrm{q}}=0$ and $t_{\mathrm{p}}=T_{\text{gen}}$. Assume that $dR^2 >\tau^2$. We have
	\begin{align}
		t_{\mathrm{q}}<t_{\mathrm{p}} \iff \frac{dR^2-\tau^2}{\tau^4}<\frac{1}{dR^2-\tau^2} \iff (dR^2-\tau^2)^2 < \tau^4 \iff dR^2 <2 \tau^2 \eqsp ,
	\end{align}
	which gives the result.
\end{proof}

Hence, \Cref{prop:tq_tp} shows that a sweet spot for the hyper-parameter $t_0$ in $\SLIPS$ exists under a restrictive condition on $R$ and $\tau$ in Assumption \Cref{ass:target}, enabling the duality of log-concavity.



\section{Details on $\SLIPS$ algorithm} \label{app:slips}

\subsection{Details on the implementation of $\SLIPS$}

\paragraph{Langevin-within-Langevin initialization.} Consider a timestep $t_0$ well chosen such that $p_{t_0}^\alpha$ and $q_{t_0}^\alpha$ are both approximately log-concave. The goal of our initialization is to sample from $p_{t_0}^\alpha$ at lowest cost. Thanks to log-concavity of $p_{t_0}^\alpha$, we may consider to apply ULA to sample from this distribution. Given $N\geq 1$ and a step-size $\lambda>0$, this amounts to consider a sequence of random variables $\{Y^{(n)}\}_{n=0}^N$, defined by the following recursion
\begin{align}
	Y^{(n+1)}= Y^{(n)} + \lambda \nabla \log p^\alpha_{t_0}(Y^{(n)}) + \sqrt{2\lambda}Z^{(n+1)} \eqsp , \label{eq:ula}
\end{align}
where $\{Z^{(n)}\}_{n=1}^N$ are independently distributed according to the standard centered Gaussian distribution. We recall that the score $\nabla \log p^\alpha_{t_0}$ is however not tractable but can be re-expressed with \eqref{eq:link_score_general} such that for any $y\in \rset^d$, we have
\begin{align}
	\nabla \log p_{t_0}^{\alpha}(y)=\{\alpha(t_0)u^\alpha_{t_0}(y)-y\}/\sigma^2 t_0 \eqsp ,
\end{align}
where $u^\alpha_{t_0}(y)$ is the expectation of the posterior $q^\alpha_{t_0}(\cdot|y)$ given in \eqref{eq:general_posterior}. Although this term is intractable too, it can be estimated by approximately sampling from $q^\alpha_{t_0}$ via MALA, since $q^\alpha_{t_0}$ is also ensured to be approximately log-concave.

Building upon this relation, at each step $n\in \{0, \hdots, N-1\}$, we propose to approximate $\nabla \log p^\alpha_{t_0}(Y^{(n)})$ in the recursion \eqref{eq:ula} with a MCMC-based estimator of $u^\alpha_{t_0}(Y^{(n)})$ obtained by sampling from $q^\alpha_{t_0}(\cdot|Y^{(n)})$. This results in the Langevin-within-Langevin procedure summarized in \Cref{alg:langevin_within_langevin}. We finally highlight three main choices in the design of this algorithm: (i) we choose the Gaussian approximation $Y^{(0)} \sim \densityGaussian(0, \sigma^2 t_0 \, \Idd)$ for ULA initialization, relying on the study conducted in \Cref{subsec:gaussian_init}; (ii) consequently, we set the step-size $\lambda$ to be slightly smaller than the variance of this distribution, \ie, $\lambda=\sigma^2 t_0 / 2$; (iii) the MALA-based posterior sampling is initialized with the best estimation of $u^\alpha_{t_0}(Y^{(n)})$ that we have, \ie, $Y^{(n)} / \alpha(t_0)$, recalling that $Y^\alpha_t / \alpha(t)$ and $u^\alpha_t(Y^\alpha_t)$ have the same localization behaviour, see \Cref{subapp:sto-loc-th}.
\vspace{-0.1cm}
\paragraph{Algorithmic technicalities.}We now present three algorithmic subtleties featured in $\SLIPS$.
\begin{enumerate}[wide, labelindent=0pt, label=(\alph*)]
	\vspace{-0.2cm}
	\item The step-size of MALA used to sample from the posterior is adapted. Using the acceptance ratio of the Metropolis-Hasting filter, we geometrically decrease (respectively increase) the step-size when the acceptance ratio is below (respectively above) a target ratio of 75\%;
	      \vspace{-0.1cm}
	\item The last MCMC samples obtained when sampling the posterior at step $k$ will be the first samples when sampling at step $k+1$. This persistent trick is motivated by the fact that the posterior is expected to change little from one iteration to the next. This behaviour is notably illustrated in \Cref{fig:fig1} \textit{(bottom row)}.
	      \vspace{-0.2cm}
	\item Lastly, as mentioned and justified in \Cref{subapp:sto-loc-th},  we use the estimated denoiser $\tU^\alpha_T$ at the final integration timestep $T$ as an approximate sample from $\pi$ rather than $Y^\alpha_T/\alpha(T)$ to align with related work.
\end{enumerate}
\subsection{Ablation study}\label{app:mcmc_ablation_study}

In this section, we investigate the impact of the different hyper-parameters of $\SLIPS$. Similarly to \Cref{app:perfect_score}, we run our study by applying our algorithm on the unbalanced bimodal Gaussian mixture defined in \eqref{eq:target_mixture} with $d=10$. This section will clarify three points:
\begin{enumerate}[wide, labelindent=0pt, label=(\alph*)]
	\vspace{-0.1cm}
	\item The behaviour of $\SLIPS$ within the assumptions of \Cref{th:duality};
	      \vspace{-0.1cm}
	\item The existence of a ``sweet spot" for $t_0$ outside the assumptions of \Cref{th:duality} and the impact of the estimation of $R_{\pi}$;
	      \vspace{-0.1cm}
	\item The tuning of hyper-parameters in the Langevin-within-Langevin initialization (see \Cref{alg:langevin_within_langevin}).
\end{enumerate}

\paragraph{Behaviour of $\SLIPS$ within the restrictive assumptions.} In this part only, we consider the case of a mixture of two Gaussian distributions in dimension 5 given by $\densityGaussian(-(2/3) a \, \mathbf{1}_5, \tau^2 \mathrm{I}_{5})$ and $\densityGaussian(+(4/3) a \, \mathbf{1}_5, \tau^2 \mathrm{I}_{5})$, where $\tau^2=0.1$ and $a>0$ is \emph{not fixed}, with respective weights $2/3$ and $1/3$. By varying $a$, we vary $R$, recalling that $R = 2a \sqrt{2}/3$, see \Cref{app:gaussian_mixture}. In \Cref{fig:app:within_ass}, we run $\SLIPS$ with different schemes while setting $t_0$ as the mean value between $t_p$ and $t_q$ derived in \Cref{prop:log_concavity}, in accordance with \Cref{th:duality}. When $p^{\alpha}_t$ and $q^{\alpha}_t$ are both log-concave independently of $t$, \ie, when $dR^2/\tau^2< 1$ (see \Cref{prop:log_concavity}), we always set $t_0 = 10^{-2}$. This choice ensures that the Gaussian approximation which corresponds to the initialization of the Langevin-within-Langevin algorithm is accurate (see \Cref{fig:app:impact_t_0} for details). \Cref{fig:app:within_ass} shows that choosing $t_0$ according to \Cref{prop:log_concavity} leads to accurate sampling below the threshold and degraded (or equal) performance above this threshold.

\paragraph{Existence of a sweet spot for duality outside of the restrictive assumptions.}

\begin{figure}[t!]
	\centering
	\includegraphics[width=0.9\linewidth]{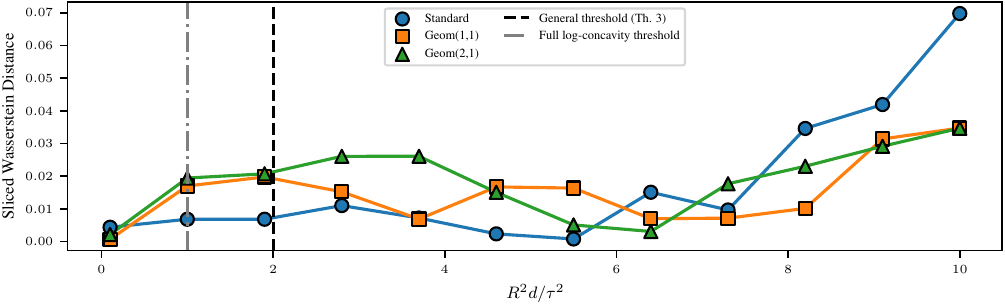}
	\caption{Sliced Wasserstein distance depending on the ratio $R^2 d / \tau^2$. \textit{left}: $R^2 d / \tau^2 = 2$ is the threshold of duality of log-concavity in \Cref{th:duality}. In both cases, the grey line corresponds to the threshold where the log-concavity conditions start to be restrictive.} 
	\label{fig:app:within_ass}
	\vspace{-0.3cm}
\end{figure}

Given the target distribution described in \eqref{eq:target_mixture} with $d=10$, we have that $R^2 = 8/9$ and $\tau^2 = 0.05$, which means that $R^2 d / \tau^2 \approx 178$.  Hence, the distribution at stake does not fit the additional assumption of \Cref{th:duality}. However, \Cref{fig:app:w2_with_t_0} shows that there still exists a sweet spot for $t_0$ in this case. Moreover, we observe in \Cref{fig:app:w2_with_t_0_sigmas} that a poor estimation of $R_{\pi}$ shifts the sweet spot. This makes sense as the sweet spot likely corresponds to a SNR level. Recalling that $\lsnr(t) = 2 \log(g(t)) + \log(R_{\pi}^2 / (\sigma^2 d))$, if $\sigma > R_{\pi} / \sqrt{d}$ (\ie, we overestimate $R_{\pi}$) the log-SNR is shifted downwards and the resulting $t_0$ should be larger (as we see in the right column of \Cref{fig:app:w2_with_t_0_sigmas}); if $\sigma < R_{\pi} / \sqrt{d}$ (\ie, we underestimate $R_{\pi}$), the resulting $t_0$ should be smaller (as we see in the left column of \Cref{fig:app:w2_with_t_0_sigmas}).

\paragraph{Choice of hyper-parameters in the Langevin-within-Langevin initialization.}
\Cref{fig:app:lwl_steps} shows that only a few steps in the Langevin-within-Langevin algorithm (see \Cref{alg:langevin_within_langevin}) are needed to reach a stationary sampling error.

\begin{figure}[t!]
	\centering
	\includegraphics{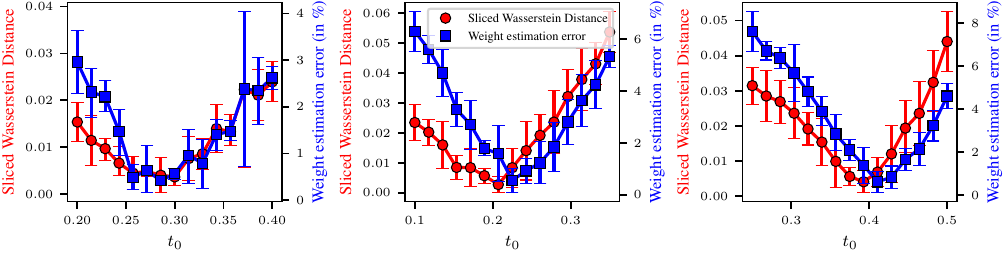}
	\vspace{-0.4cm}
	\caption{Sliced Wasserstein distance and estimation error on the relative weight depending on $t_0$ with different schemes. \textbf{Left}: Standard. \textbf{Middle}: Geom(1,1). \textbf{Right}: Geom(2,1).}
	\label{fig:app:w2_with_t_0}
\end{figure}

\begin{figure}[t!]
	\centering
	\includegraphics[width=\linewidth]{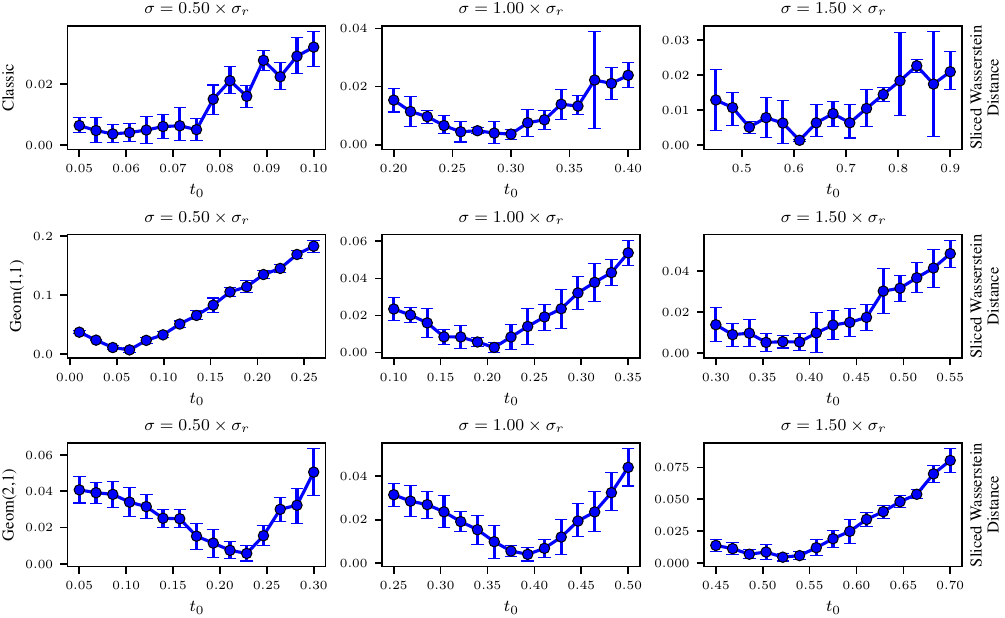}
	\vspace{-0.4cm}
	\caption{Sliced Wasserstein distance depending on $t_0$ for different values of $\sigma$. We denote $\sigma_r = R_{\pi} / \sqrt{d}$.}
	\label{fig:app:w2_with_t_0_sigmas}
\end{figure}

\begin{figure}[t!]
	\centering
	\includegraphics[width=\linewidth]{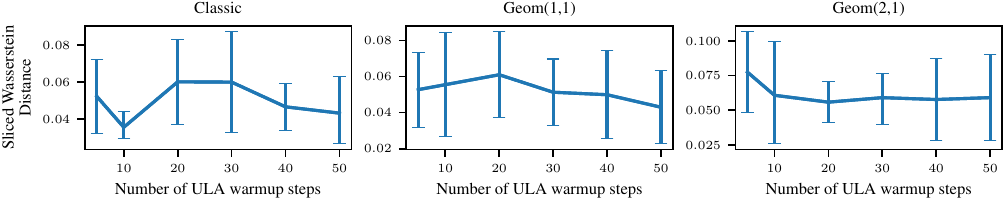}
	\caption{Sliced Wasserstein distance depending on the number of Langevin-within-Langevin steps.}
	\label{fig:app:lwl_steps}
\end{figure}

\section{Details on related work} \label{app:related_work}

In this appendix, we provide further details on related works, highlighting similarities and identifying the main limitations.

\subsection{Reverse Diffusion Monte Carlo \cite{huang2023monte}}

Given $T>0$, the authors consider the denoising process $(Y_t)_{t\in [0,T]}$ that is solution to the following SDE
\begin{align} \label{eq:SDE-diffusion}
	\rmd Y_t = \{Y_t+2 \nabla \log p_{T-t}(Y_t)\} +\sqrt{2}\rmd B_t \eqsp , \eqsp  Y_0\sim p_T \eqsp ,
\end{align}
where $(B_t)_{t\geq 0}$ is a Brownian motion in $\rset^d$ and $p_s$ is the \emph{intractable} marginal distribution defined for any $s>0$ and any $y\in \rset^d$ by $p_s(y)=\int_{\rset^d} \densityGaussian(y; \mathrm{e}^{-s}x, (1-\mathrm{e}^{-2s})\,\Idd)\rmd \pi(x)$. Under mild conditions on $\pi$ \cite{cattiaux2021time}, $(Y_t)_{t\in [0,T]}$ is the time-reversed process of the standard Ornstein-Uhlenbeck process defined in \eqref{eq:OU} and corresponds to a Variance-Preserving diffusion model \cite{song2020score}. For any $t\in [0,T]$, given $X\sim \pi$, we therefore have
\begin{align}
	Y_t=  \mathrm{e}^{-(T-t)}X + \sqrt{1-\mathrm{e}^{-2(T-t)}} Z \eqsp ,
\end{align}
where $Z$ is distributed according to the standard centered Gaussian distribution. In this case, in a similar fashion to the SL posterior given in \eqref{eq:general_posterior}, the conditional density of $X$ given $Y_t=y\in \rset^d$ can be defined as
\begin{align}\label{eq:posterior-RDMC}
	q_t(x|y) \propto \pi(x)\densityGaussian(x; \mathrm{e}^{T-t} y, (\mathrm{e}^{2(T-t)}-1)\, \Idd) , \eqsp t\in [0,T) \eqsp ,
\end{align}
and the denoiser function as $u_t(y)=\int_{\rset^d}x q_t(x|y)\rmd x$. Similarly to our framework, the random vector $u_t(Y_t)$ can be seen as the denoiser of $Y_t$. Following Tweedie's formula given in \Cref{lemma:tweedie}, the denoiser function is related to the score of $p_{T-t}$ for any $t\in[0,T)$ by
\begin{align}
	\nabla \log p_{T-t}(y)=-\frac{y}{1-\mathrm{e}^{-2(T-t)}}+ \frac{\mathrm{e}^{-(T-t)}}{1-\mathrm{e}^{-2(T-t)}} u_{t}(y) \eqsp .
\end{align}
Therefore, the SDE \eqref{eq:SDE-diffusion} describing the denoising process is strictly equivalent to the following SDE
\begin{align}\label{eq:SDE-diffusion-2}
	\rmd Y_t = \left\{\frac{\mathrm{e}^{-2(T-t)}+1}{\mathrm{e}^{-2(T-t)}-1}Y_t + \frac{2\mathrm{e}^{-(T-t)}}{1-\mathrm{e}^{-2(T-t)}}u_{t}(Y_t)\right\}\rmd t + \sqrt{2} \rmd B_t, \eqsp Y_0\sim p_T \eqsp .
\end{align}
This last SDE has to be compared with the SDE describing our observation process in \eqref{eq:SDE-gen}: in both cases, the drift involves the denoiser function, which is not tractable in practice. Here, the target distribution is approximated by the distribution of $Y_T$, while it is approximated by the distribution of $Y^\alpha_T/\alpha(T)$ in our framework.

The approach of \cite{huang2023monte} to sample from the SDE \eqref{eq:SDE-diffusion-2}, while handling the estimation of the denoiser $u_t$, is close to ours. Indeed, after discretizing this SDE with a certain time discretization $(t_k)_{k=1}^K$ of $[0,T]$, they propose to estimate the denoiser at time $t_k$ with a Markov Chain Monte Carlo estimation. To do so, at each step $t_k$, they approximately sample from the random posterior $q_{t_k}(\cdot|\tY_{t_k})$ with ULA (while we use MALA), where $\tY_{t_k}$ is the $k$-th realization of their discretized process. Besides this, they also propose a Langevin-within-Langevin initialization to approximately sample from $p_T$, that involves the corresponding posterior $q_0$. Although they briefly discuss the difficulty of this initialization by considering the log-Sobolev constants of $p_T$ and $q_0$, their analysis lacks readability and does not emphasize any crucial trade-off on $T$ (which corresponds to our $t_0$).

Their scheme can nonetheless be analyzed as a localization scheme under the scope of what we call ``duality of log-concavity", see \Cref{subsection:duality}. Assume that $\pi$ verifies \Cref{ass:target} and is not log-concave. Since we have $q_0(x|y) \propto \pi(x)\densityGaussian(x; \mathrm{e}^{T} y, (\mathrm{e}^{2T}-1)\, \Idd)$, one can show that
\begin{enumerate}[wide, labelindent=0pt, label=(\alph*)]
	\item if $T$ is large: $q_0$ will be close to $\pi$ (not log-concave), while $p_T$ will be close to $\densityGaussian(0, \Idd)$ (log-concave),
	\item if $T$ is small: $q_0$ will localize as a Dirac mass (log-concave), while $p_T$ will be close to $\pi$ (not log-concave).
\end{enumerate}

Therefore, the setting of $T$ is crucial in RDMC to ensure a proper initialization via Langevin-within-Langevin algorithm. However, the authors claim that setting $T$ large would only incur computational waste and that RDMC shows insensitivity to $T$, aside from the case where $T$ is too close to 0\footnote{See \citep[Appendix F.4]{huang2023monte}.}. As explained above, these claims are not realistic, since sampling from the posterior $q_0$ becomes as hard as sampling from $\pi$ when $T$ is large. We illustrate this fundamental limitation in \Cref{fig:app:rdmc_T}.

\begin{figure}[t!]
	\centering
	\includegraphics[width=\linewidth]{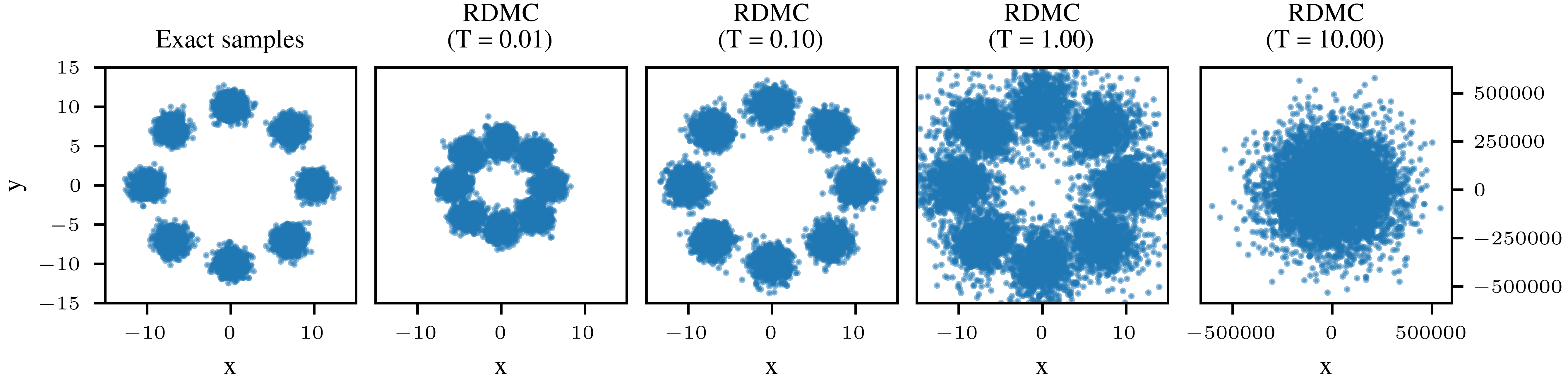}
	\caption{Sampling the 8-Gaussians distribution via RDMC with different values of $T$.}
	\label{fig:app:rdmc_T}
\end{figure}

In contrast, we pay a particular attention to point out the importance of our hyper-parameter $t_0$ throughout our analysis, with practical and theoretical perspectives. Finally, we highlight that RDMC relies on a uniform time discretization (while we propose a SNR-adapted discretization), and do not use persistent initialization of their Markov chains in posterior sampling.

\subsection{Multi-Noise Measurements sampling \cite{saremi2022multimeasurement, saremi2023universal, saremi2023chain}}

Given $M\geq 1$, the Multi-Noise Measurements (MNM) model introduced by \cite{saremi2022multimeasurement} defines a sequence of random variables $(Y^m)_{m=1}^M$ as
\begin{align}
	Y^m = X + \sigma Z^m \eqsp ,
\end{align}
where $X \sim \pi$, $\sigma > 0$ and $(Z^m)_{m=1}^M$ are independently distributed according to the standard centered Gaussian distribution. This \emph{non-Markovian} stochastic process can be seen as a denoising process. The rationale behind this formulation is that simulating an increasing number of \emph{equally} noised measurements of a sample from $\pi$ helps to obtain more information on this sample, and finally approximate it. We now explain how it can be interpreted as a non-Markovian analog to the standard localization scheme given in \eqref{eq:def_y_t}.

For any $m\in \{1, \hdots, M\}$, the conditional density of $X$ given the $m$-tuple $Y^{1:m}=y_{1:m}\in (\rset^d)^m$ is defined as
\begin{align}
	q_m(x|y_{1:m})\propto \pi(x) \densityGaussian(x; \bar{y}_{1:m}, \sigma^2/m \, \Idd) \eqsp ,
\end{align}
where $\bar{y}_{1:m}$ denotes the empirical mean of the noised measurements, \ie, $\bar{y}_{1:m}=(1/m)\sum_{i=1}^m y_i$. By defining $u_m(y_{1:m})=\int_{\rset^d}x q_m(x|y_{1:m})\rmd x$, one can derive the Bayes estimator of $X$ given $Y^{1:m}$ as the random vector $u_m(Y^{1:m})$ \cite{saremi2022multimeasurement}. Interestingly, this denoiser shares the same rate of convergence as the denoiser of standard stochastic localization given in \Cref{prop:conv_sto_loc-denoiser} when $t=m$. Indeed, \citep[Proposition 2]{saremi2023chain} states that for any target distribution $\pi$, we have $W_2(\pi, \tilde{\pi}_m)\leq \sigma \sqrt{d/m}$, where $\tilde{\pi}_m$ denotes the distribution of $u_m(Y^{1:m})$. In other words, the MNM denoiser localizes to $X$ with the same localization rate as the standard SL denoiser, defined in \Cref{sec:background}, by taking $T=M$.

To sample from $\pi$, the MNM approach consists in first sampling $Y^{1:M}$, and then computing the denoiser $u_M(Y^{1:M})$, in the same fashion as in the SL framework. Recently, \cite{saremi2023chain} proposed to tackle the sampling of the $M$-tuple $Y^{1:M}$ by first simulating $Y^1$ and then sequentially sampling $Y^m$ given $Y^{1:m-1}$ for $m\in\{2, \hdots, M\}$ using a Monte Carlo Markov Chain method - in this case, Underdamped Langevin Algorithm \cite{sachs2017langevin}. At step $m$, the Langevin procedure then involves pointwise evaluations of the conditional score of the distribution of $Y^m$ given $Y^{1:m-1}$ (or simply the score of the distribution of $Y_1$ when $m=1$). This introduces the \emph{Once-At-a-Time} (OAT) algorithm. The interest of such method lies in the fact that, under Assumption \Cref{ass:target} that may include multi-modal distributions, the distributions to sample from are increasingly log-concave with $m$, see \citep[Theorem 1]{saremi2023chain}. Hence, the most challenging step of this approach lies in the sampling of $Y^1$, in the same spirit as in $\SLIPS$.

In their work, \cite{saremi2023chain} mainly consider the case where the scores are analytically available. To handle realistic settings, they implement an IS estimator, see \citep[Section 4.2.1]{saremi2023chain}, but their numerical results show that its performance significantly degrades compared to setting of perfect knowledge of the score, see \citep[Appendix H]{saremi2023chain}. We include this approach in our numerical benchmark. Alternatively, for high-dimensional settings, they propose an estimator based on posterior sampling, see \citep[Appendix F.2.]{saremi2023chain}, in the same fashion as in $\SLIPS$. However, they do not pay attention to the limitation of this approach, in particular at the initialization of their algorithm. Indeed, while \citep[Theorem 1]{saremi2023chain} suggests to take $\sigma$ very large to ensure that the distribution of $Y^1$ is approximately log-concave (\ie, taking $t_0$ small in $\SLIPS$), the corresponding posterior becomes unfortunately closer to $\pi$, and then hard to sample with standard MCMC methods when $\pi$ is not log-concave. This fundamental constraint is well illustrated in \citep[Figure 7, left]{saremi2023chain}, where the posterior sampling approach is shown to systematically fail for an arbitrary choice of $\sigma$, independently of the computational budget allocated to the MCMC sampling. Therefore, the combination of OAT with posterior sampling requires a trade-off on the hyper-parameter $\sigma$, similarly to $t_0$ in $\SLIPS$, that reflects once again the ``duality of log-concavity".

\section{Details on numerical experiments} \label{app:numerics}

\subsection{The failure of MCMC methods when targeting multi-modal distributions} \label{app:mcmc}

As we show in \Cref{fig:local_mcmc}, \emph{local} MCMC samplers such as MALA, HMC, NUTS or ESS tend to produce Markov chains that get trapped in modes while our methodology $\SLIPS$ generates samples reaching both modes.

\begin{figure}[t!]
	\centering
	\includegraphics[width=0.9\linewidth]{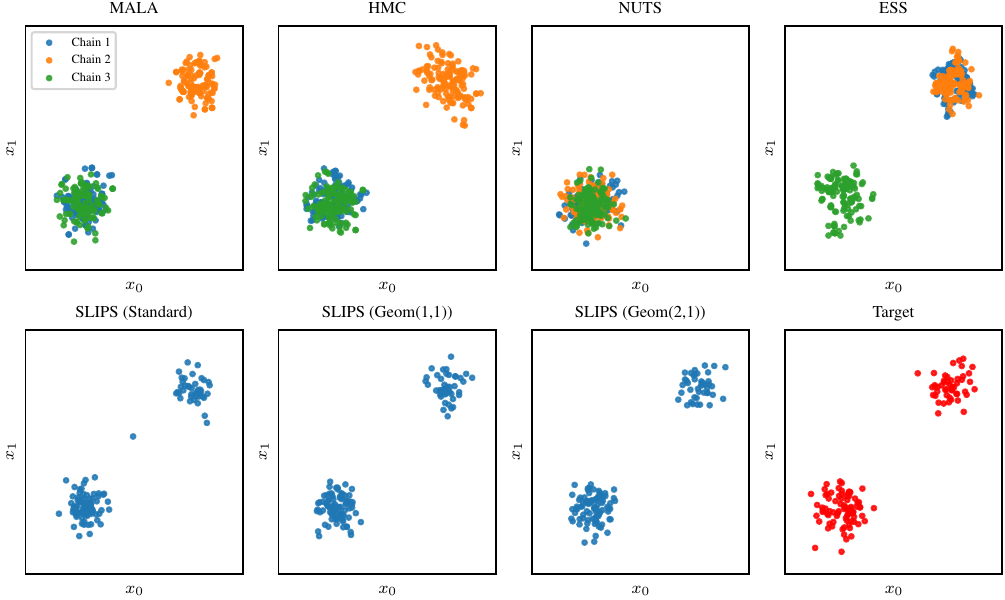}
	\caption{Samples from the Gaussian mixture defined in Section 6 ($d=8$) obtained using different algorithms. Here, we display the first two coordinates of $n = 128$ samples. For $\SLIPS$, we used the values of $t_0$ and $\eta$ given in Table 5, $K = 10$ discretization steps and $n_{\text{MCMC}} = 16$ MCMC steps. On the other hand, the standard MCMC algorithms produce 3 chains, each one being of size $K \times n_{\text{MCMC}} \times n$ for fairness of comparison on the number of target density evaluations; we only display the $n$ last MCMC samples of each chain. The chains were initialized in $\densityGaussian(0,\sigma^2 \Idd)$. For HMC, we grid-searched an optimal trajectory length in $[4,8,12,16]$ which ended up being $8$ (the step-size was automatically tuned). For ESS, we used $\densityGaussian(0,\sigma^2\, \Idd)$ as a prior.}
	\label{fig:local_mcmc}
	\vspace{-0.75cm}
\end{figure}

\subsection{Algorithms and hyper-parameters}

\paragraph{Using the information from $R_{\pi}$ and adapting the algorithms.}

As our definition of $\sigma$ in $\SLIPS$ requires to have an estimate of $R_{\pi}$ (see \Cref{subsec:guideline}), we informed other algorithms with such knowledge when possible. More precisely, when it is possible, we determine $R_\pi$ analytically; otherwise, we use ground truth samples to obtain a Monte Carlo estimation of $R_\pi$. Besides this, we assume that $R$ and $\tau$ are assumed to be known for each target distribution verifying \Cref{ass:target}.
\begin{itemize}[labelindent=0pt]
	\item For SMC and AIS, we use a starting distribution $\rho_0$ as the centered Gaussian with variance $R^2 d  \,\Idd$;
	\item For OAT, we use $\sigma$ as per \citep[Equation 4.3]{saremi2023chain} by setting $\sigma^2 = R^2 d - \tau^2$ (ensuring log-concavity);
	\item For RDMC, we do not include this information as the authors do not give any heuristic on $T$ with respect to the variance of the target distribution.
\end{itemize}
Moreover, the implementations of the algorithms were slightly adjusted from their general definition. SMC and AIS use a MALA kernel (which leaves the target distribution invariant) as transition kernel. Additionally, RDMC also uses a MALA kernel for posterior sampling instead of ULA. This reduces the estimation bias and also enables automatic tuning of the Langevin step-size by leveraging the acceptance ratio (here we adapt the step-size to maintain the acceptance ratio at 75\%). Still on RDMC, we drop the first 50\% of the MCMC samples to ignore the warm-up period in the estimation. We also ignore the warm-up period with the same proportion in $\SLIPS$. In $\SLIPS$, we reuse the step-size from the MCMC sampling of $q_{t_k}$ when sampling $q_{t_{k+1}}$ and we also initialize the chains of the later with the last samples of the chain from the former. This is an intuitive initialization when looking at the bottom row of \Cref{fig:fig1} as the modes of the posterior seem to be stable and $q_{t_k}$ is expected to be close to $q_{t_{k+1}}$.

\paragraph{Estimating the scores in OAT.}

In OAT, denoting $p(y)$ the likelihood of the measurement process $Y = X + \sigma Z$ with $X \sim \pi$ and $Z \sim \densityGaussian(0, \Idd)$, the score can be written using the following identity from \citep[Appendix F]{saremi2023chain} : $\nabla \log p(y) = \sigma^{-1} \PE_{Z \sim q(\cdot | y; \sigma)}[Z]$ where the posterior $q(\cdot | y)$ is defined by $q(z | y) \propto \pi(y + \sigma z) \densityGaussian(z; 0, \Idd)$. This means that the score $\nabla \log p(y)$ can be estimated with the following IS estimator $\nabla \log p(y) \approx \sigma^{-1} \sum_{i=1}^N w_i Z_i$ where $Z_i \overset{\iid}{\sim} \densityGaussian(0, \Idd)$ and $w_i = \pi(y + \sigma Z_i) / \sum_{j=1}^N \pi(y + \sigma Z_j)$. We reduce the variance of this estimator by applying the antithetic trick.

\paragraph{Selecting the hyper-parameters of the algorithms.}

\begin{table}[t]
	\caption{Hyper-parameter grids used for $\SLIPS$ on different targets.}
	\label{tbl:grids_slips}
	\vskip 0.15in
	\begin{center}
		\begin{small}
			\begin{sc}
				\begin{tabular}{lcccc}
					\toprule
					                           & \multicolumn{1}{l}{} & Mixture of Gaussians and Bayesian & $\phi^4$ model                & Others                        \\
					\midrule
					\multirow{2}{*}{Standard}  & $\eta$               & $\{5.0\}$                         & $\{5.7, 6.1\}$                & $\{5.0, 5.7\}$                \\
					                           & $t_0$                & $\{0.03, 0.05, 0.1, 0.2, 0.4\}$   & $\{0.8, 1.0, 1.2, 1.4, 1.8\}$ & $\{0.1, 0.2, 0.4, 1.0, 1.2\}$ \\
					\midrule
					\multirow{2}{*}{Geom(1,1)} & $\eta$               & $\{5.0\}$                         & $\{5.7, 6.1\}$                & $\{4.6, 5.0\}$                \\
					                           & $t_0$                & $\{0.03, 0.05, 0.1, 0.15, 0.25\}$ & $\{0.30, 0.35, 0.40, 0.45\}$  & $\{0.1, 0.15, 0.20\}$         \\
					\midrule
					\multirow{2}{*}{Geom(2,1)} & $\eta$               & $\{5.0\}$                         & $\{5.7, 6.1\}$                & $\{4.6, 5.0\}$                \\
					                           & $t_0$                & $\{0.15, 0.20, 0.25,0.35, 0.45\}$ & $\{0.40, 0.45, 0.50, 0.55\}$  & $\{0.30, 0.35, 0.45\}$        \\
					\bottomrule
				\end{tabular}
			\end{sc}
		\end{small}
	\end{center}
	\vskip -0.1in
\end{table}

For each algorithm, we search its hyper-parameters within a predetermined grid. The selection is based on the metrics which will be later detailed. The metrics were computed by comparing $4096$ samples against true samples. We globally fixed the computational budget by setting the SDE discretization of $\SLIPS$ as $K = 1024$. Moreover, we define the number of MCMC steps of $\SLIPS$, denoted by $L$, to be equal to $32$ except with the mixture of Gaussian in high dimensions where it is equal to $48$, $64$ and $96$ in dimensions $32$, $64$ and $128$ respectively and in the $\phi^4$ experiments were it is set to $64$. The selected hyper-parameters for each algorithm are summarized in \Cref{tbl:hyper-params}. Below, we detail how the grids were built for each one.
\begin{itemize}
	\item The SMC and AIS algorithms define a sequence of annealed distributions $\rho_k$ for $k \in \iint{0}{K}$ from $\rho_0$ (defined above) to $\rho_K = \pi$ as $\rho_k \propto \exp((1-\beta_k) \log \rho_0 + \beta_k \log \pi))$ where $(\beta_k)_{k=0}^K$ is a linear schedule of size $K$ between $0$ and $1$. Both algorithms used $N = 4096$ particles;
	\item The RDMC algorithm has a single hyper-parameter $T$ which we search within the grid $T \in \{-\log 0.99, -\log 0.95, -\log 0.9, -\log 0.8, -\log 0.7\}$ given in \citep[Appendix F.1]{huang2023monte}.
	      However, we also find that large $T$ may not work systematically (see \Cref{fig:app:rdmc_T}). We use $16$ steps of \citep[Algorithm 3]{huang2023monte} with $4$ MCMC chains. The chains are initialized with an IS approximation of the posterior powered by $128$ particles. The SDE is discretized over $K$ steps and the expectation leading to the drift is estimated with $L$ Langevin steps. The initial sample is distributed according to $\densityGaussian(0, (1 - \exp(-2T)) \, \Idd)$, the best estimation of $p_T$ that we have, and the initial step-size is taken according to its variance
	      ;
	\item The OAT algorithm has its parameter $M$ set as $\floor{K / 2}$ and uses as many MCMC steps per noise level as $\SLIPS$ or RDMC. This choice of $K$ ensures that the computational complexity of OAT is on par with the other algorithms. The Langevin steps are done using the underdamped Langevin algorithm as suggested by the authors. Its step-size is searched in $\{0.03, 1.0\}$ and the efficient friction is searched in $\{0.0625, 0.05\}$ as recommended by the authors. These prescriptions were extracted from \citep[Appendix G]{saremi2023chain}. The OAT algorithm has slightly shorter grid sizes because of its prohibitive computational cost;
	\item The $\SLIPS$ algorithm is declined in three flavors depending on the choice of the schedule $\alpha$. Since the choice of $t_0$ is sensitive to the accuracy of the estimation of $R_{\pi}$ (different values of $R_{\pi}$ will shift the log-SNR upwards or downwards), we decided to search the hyper-parameters in different areas depending on the target distribution. Those grids can be found in \Cref{tbl:grids_slips}. The values for $t_0$ were chosen by approximate equal log-SNR spacing in $[-3.5,-1.0]$ for Gaussian mixtures and Bayesian logistic regression, in $[-1.0, 0.2]$ for $\phi^4$ and $[-2.0, 0.0]$ for the others. The values for $\eta$ were chosen to be around $5.0$.
\end{itemize}

\begin{table}[t!]
	\caption{Hyper-parameters selected for the experiments for each algorithm and target density.}
	\vspace{-0.5cm}
	\label{tbl:hyper-params}
	\vskip 0.15in
	\begin{center}
		\begin{small}
			\begin{sc}
				\begin{tabular}{cccccccccc}
					\toprule
					Target                 & RDMC          & \multicolumn{2}{c}{OAT} & \multicolumn{2}{c}{$\SLIPS$ Standard} & \multicolumn{2}{c}{$\SLIPS$ Geom(1,1)} & \multicolumn{2}{c}{$\SLIPS$ Geom(2,1)}                                   \\
					\midrule
					                       & Time $T$      & step-size $\delta$      & Frict. $\delta \gamma$                & $\eta$                                 & $t_0$                                  & $\eta$ & $t_0$ & $\eta$ & $t_0$ \\
					\midrule
					8-Gaussians            & $-\log(0.80)$ & 0.05                    & 0.03                                  & 5.7                                    & 0.60                                   & 5.7    & 0.35  & 5.0    & 0.35  \\
					Rings                  & $-\log(0.80)$ & 0.0625                  & 1.0                                   & 4.6                                    & 1.20                                   & 4.6    & 0.10  & 4.6    & 0.30  \\
					Funnel                 & $-\log(0.90)$ & 0.05                    & 0.03                                  & 5.0                                    & 1.00                                   & 4.6    & 0.30  & 4.6    & 0.40  \\
					Mixture ($d = 8$)      & $-\log(0.70)$ & 0.0625                  & 0.03                                  & 5.0                                    & 0.40                                   & 5.0    & 0.25  & 5.0    & 0.45  \\
					Mixture ($d = 16$)     & $-\log(0.70)$ & 0.0625                  & 0.03                                  & 5.0                                    & 0.20                                   & 5.0    & 0.15  & 5.0    & 0.35  \\
					Mixture ($d = 32$)     & $-\log(0.70)$ & 0.05                    & 0.03                                  & 5.0                                    & 0.10                                   & 5.0    & 0.10  & 5.0    & 0.25  \\
					Mixture ($d = 64$)     & $-\log(0.70)$ & 0.0625                  & 0.03                                  & 5.0                                    & 0.05                                   & 5.0    & 0.05  & 5.0    & 0.20  \\
					Ionosphere             & $-\log(0.95)$ & 0.0625                  & 0.03                                  & 5.0                                    & 0.03                                   & 5.0    & 0.03  & 5.0    & 0.15  \\
					Sonar                  & $-\log(0.95)$ & 0.0625                  & 0.03                                  & 5.0                                    & 0.03                                   & 5.0    & 0.03  & 5.0    & 0.15  \\
					$\phi^4$ ($b = 0$)     & $-\log(0.95)$ & 0.0625                  & 0.03                                  & 5.7                                    & 0.80                                   & 5.7    & 0.30  & 5.7    & 0.40  \\
					$\phi^4$ ($b = 0.025$) & $-\log(0.95)$ & 0.0625                  & 1.0                                   & 5.7                                    & 1.80                                   & 5.7    & 0.35  & 6.1    & 0.45  \\
					$\phi^4$ ($b = 0.05$)  & $-\log(0.70)$ & 0.05                    & 1.0                                   & 6.1                                    & 1.00                                   & 5.7    & 0.30  & 5.7    & 0.40  \\
					$\phi^4$ ($b = 0.075$) & $-\log(0.70)$ & 0.0625                  & 0.03                                  & 5.7                                    & 1.80                                   & 5.7    & 0.35  & 5.7    & 0.40  \\
					$\phi^4$ ($b = 0.1$)   & $-\log(0.90)$ & 0.05                    & 0.03                                  & 5.7                                    & 1.40                                   & 5.7    & 0.45  & 5.7    & 0.40  \\
					\bottomrule
				\end{tabular}
			\end{sc}
		\end{small}
	\end{center}
	\vskip -0.1in
\end{table}

\subsection{Target distributions and metrics}

\paragraph{8 Gaussians, Rings and Funnel distributions and their respective metrics.}
\begin{enumerate}[wide, labelindent=0pt, label=(\alph*)]
	\item The 8 Gaussians distribution consists of 8 equally weighted Gaussian distributions with mean $\mathbf{m}_i = 10 \times (\cos(2 \pi i / 8), \sin(2 \pi i / 8))$ for $i \in \{0,\ldots,7\}$ and covariance $0.7\, \mathrm{I}_2$. This distribution satisfies \Cref{ass:target} with $R=10 / \sqrt{2}$ and $\tau^2 = 0.7$.
	\item The Rings distribution is the inverse polar reparameterization of a distribution $p_z$ which has itself a decomposition into two univariate marginals $p_r$ and $p_{\theta}$: $p_r$ is a mixture of 4 Gaussian distributions $\densityGaussian(i+1, 0.15^2)$ with $i \in \{0,\ldots,3\}$ describing the radial position and $p_{\theta}$ is a uniform distribution over $[0, 2\pi]$, which describes the angular position of the samples. This distribution satisfies \Cref{ass:target} with $R = 4 / 2$ and $\tau = 0.15$.
	\item The Funnel distribution \cite{neal20023sliced} has its density given by
	$$
		x_{1:10} \mapsto \densityGaussian(x_1; 0, \sigma^2 \mathrm{I}_{10}) \densityGaussian(x_{2:10}; 0, \exp(x_1) \mathrm{I}_{10})\eqsp,
	$$
	where $\sigma^2 = 9$. We approximate $R$ by the maximum scalar standard deviation of the distribution, \ie, $R = 2.1$, and set $\tau = 0$.
\end{enumerate}
For 8 Gaussians and Rings, we evaluate the quality of sampling by approximating the entropy-regularized 2-Wasserstein distance defined in \Cref{app:preliminaries}, with regularization $\varepsilon=0.05$ via POT library \cite{flamary2021pot}. Regarding the Funnel, we evaluate the quality of sampling using the Kolmogorov-Smirnov distance, see \Cref{app:preliminaries}, and adopt the sliced version from \citep[Appendix D.1]{grenioux2023onsampling}.

\paragraph{Bayesian Logistic Regression models.}

Consider a training dataset $\mathcal{D} = \{(x_j, y_j)\}_{j=1}^M$ where $x_j \in \rset^d$ and $y_j \in \{0, 1\}$ for all $j \in \{1,\ldots, M\}$. We evaluate the likelihood of a pair $(x,y)$ as given by $p(y | x; w, b) = \operatorname{Bernoulli}(y; \sigma(x^T x + b))$ where $w \in \rset^d$ is a weight vector, $b \in \rset$ is an intercept and $\sigma$ is the sigmoid function. Given a prior distribution $p(w,b)$, we sample from the posterior distribution $p(w, b | \mathcal{D}) \propto p(\mathcal{D} | w, b) p(w,b) = \prod_{j=1}^M p(y_j | x_j; w, b) p(w,b)$. The prior is built as $p(w,b) = \densityGaussian(w; 0, \Idd) \densityGaussian(b; 0, (2.5)^2)$. In our experiments, we approximate $R$ by the maximum scalar standard deviation of the prior distribution, \ie, $R = 1.1$ and set $\tau = 0$. The quality of the samples are obtained by computing the mean predictive log-likelihood (\ie, computing $p(w, b | \mathcal{D}_{\text{test}})$ with $\mathcal{D}_{\text{test}}$ a separate test dataset).

\paragraph{High dimension Gaussian mixture.}

We consider the mixture of two Gaussian distributions defined in \eqref{eq:target_mixture}. Following the computations of \Cref{app:gaussian_mixture}, we set $R = 2\sqrt{2}/3$ and $\tau^2 = 0.05$. Complementary to the estimation error on the relative weight given in the main paper, we also report the Sliced Wasserstein distance \cite{bonneel2015sliced,nadjahi2019asymptotic} on \Cref{fig:app:two_modes_w2} (\textit{left}). This figure shows that $\SLIPS$ also recovers the local structure of the target distribution quite well. Moreover, \Cref{fig:app:two_modes_w2} (\textit{right}) shows that SMC and AIS are unable to scale with dimension and collapse into a single mode.
\begin{figure}[t!]
	\centering
	\includegraphics[width=0.49\linewidth]{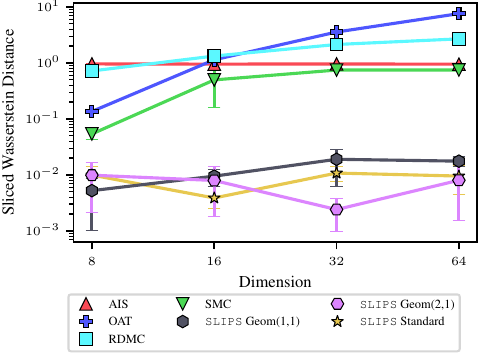}
	\hfill
	\includegraphics[width=0.49\linewidth]{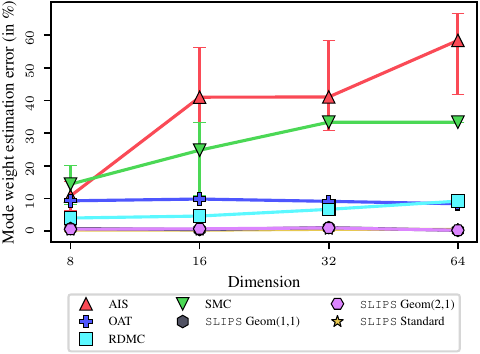}
	\caption{\textbf{Left}: Sliced Wasserstein distance computed on bimodal Gaussian mixtures with increasing dimension. \textbf{Right}: Estimation error when computing the relative weight of the corresponding two modes.}
	\label{fig:app:two_modes_w2}
\end{figure}

\paragraph{The $\phi^4$ distribution.}

The $\phi^4$ model is a toy model defined as a continuous relaxation of the Ising model that serves the study of phase transitions in statistical mechanics. Following \cite{gabrie2022adaptive},  we consider a version of the model discretized on a 1-dimensional grid of size $d=100$. One configuration is therefore a $d$-dimensional vector $(\phi_i )_{i=1}^d$. We additionally clip the field to 0 at both extremities by defining the extra $\phi_0 = \phi_{d+1} = 0$. The negative log-density of the distribution writes
\begin{align}\label{app:eq:phi4}
	\ln \pi_h(\phi) = - \beta \left( \frac{ad}{2}  \sum_{i=1}^{d+1} (\phi_i-\phi_{i-1})^2 + \frac{1}{4ad} \sum_{i=1}^d (1-\phi_i^2)^2 + h\phi_i \right) \eqsp .
\end{align}
We chose parameter values for which the system is bimodal, $a=0.1$ and inverse temperature $\beta=20$, and vary the value of $h$. We denote by $w_{+}$ the statistical occurrence of configurations such that $\phi_{d/2} > 0$ and $w_{-}$  the statistical occurrence of configurations such that $\phi_{d/2} < 0$. At $h=0$, the measure is invariant under the symmetry $\phi \rightarrow - \phi$, such that we expect $w_{+}=w_{-}$. For $h>0$, the negative mode dominates. We plot the two modes on \Cref{fig:phi_four_modes}.

When $d$ is large, the relative probability of the modes can be estimated by a Laplace approximations at 0-th and 2-nd order. Denoting by $\phi_+^h$ and $\phi_-^h$ the local maxima of \eqref{app:eq:phi4}, these approximations yield respectively,
\begin{align}
	\frac{w_{-}}{w_{+}} \approx \frac{\pi_h(\phi_{-}^h)}{\pi_h(\phi_{+}^h)}\eqsp , \quad \frac{w_{-}}{w_{+}} \approx \frac{\pi_h(\phi_{-}^h)\times |\det H_h(\phi_{-}^h)|^{-1/2}}{\pi_h(\phi_{+}^h)\times |\det H_h(\phi_{+}^h)|^{-1/2}} \eqsp ,
\end{align}
where $H_h$ is the Hessian of the function $\phi \rightarrow \ln \pi_h(\phi)$. In our experiments, we considered $R = 0.85$ and $\tau = 0.15$ by running MALA chains started in each mode : $\tau^2$ is set as the variance within the modes while $R$ corresponds to the distance between the modes and $0_{100}$.
In this setting, we observed that SMC and AIS suffered from mode collapse, while OAT and RDMC produced degenerate samples; in contrast, our methodology $\SLIPS$ recovers accurate samples from the target distribution, with correct relative weight, as shown in \Cref{fig:bench:weight_ratio_sto_loc}.

\begin{figure}[t!]
	\centering
	\includegraphics[width=0.6\linewidth]{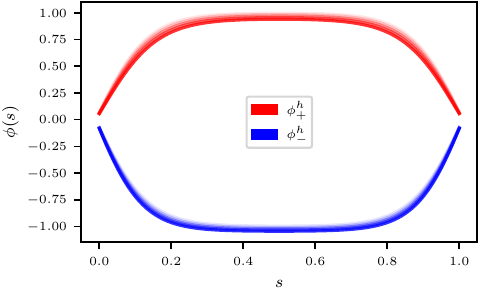}
	\caption{Modes of the $\phi^4$ distribution ($\phi_{+}^h$ in shades of {\color{red} red} and $\phi_{-}^h$ in shades of {\color{blue} blue}) for different values of $h$ (corresponding to the transparency levels). The time interval $[0,1]$ on the $x$-axis is discretized into a grid of size $d=100$, which corresponds to the dimensionality of the samples.}
	\label{fig:phi_four_modes}
\end{figure}

\subsection{Empirical complexity of $\SLIPS$} \label{app:tde}

The results of \Cref{fig:tde_dim_32} ($d=32$) show that under a much lower computational budget ($K = 20$ here) than in \Cref{sec:xps} ($K=1024$), $\SLIPS$ maintains the same performance, while still outperforming its competitors. Moreover, these results demonstrate that $\SLIPS$ has very reasonable execution times (below 1 minute to obtain $8192$ samples) making it completely practical. Lastly, \Cref{fig:tde_dim_32} emphasizes that, even using high computational budgets, competitors cannot solve such multi-modal tasks in high-dimension. Note that those observations stay valid in problems with lower dimension (see \Cref{fig:tde_dim_4} where $d=4$).

\begin{figure}[t!]
	\centering
	\includegraphics[width=0.87\linewidth]{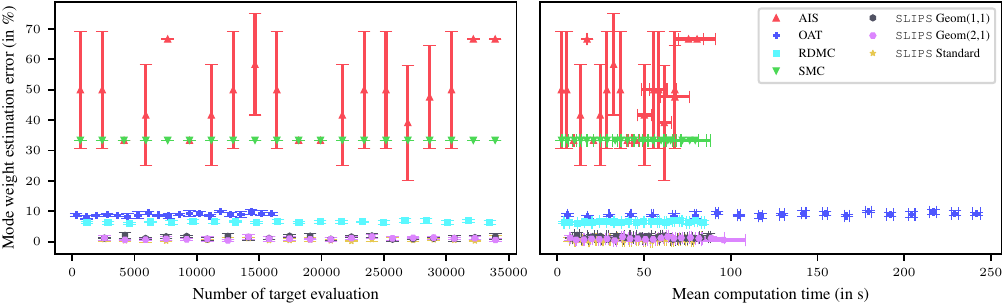}
	\hfill
	\includegraphics[width=0.87\linewidth]{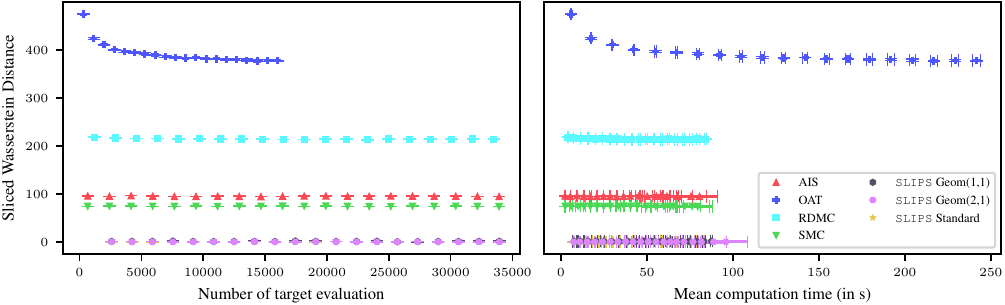}
	\vspace{-0.2cm}
	\caption{Metrics when sampling from the Gaussian mixture defined in \Cref{sec:xps} ($d=32$) using different algorithms. \textbf{Top}: Weight estimation error. \textbf{Bottom}: Sliced Wasserstein Distance. \textbf{Left}: Metric depending on the number of target density evaluations. \textbf{Right}: Metric depending on wall time. The computational budgets are computed by evolving $K$ linearly from $20$ to $90$. The number of MCMC steps is fixed at 32. The computations were run on the same Nvidia V100 GPU.}
	\label{fig:tde_dim_32}
\end{figure}

\begin{figure}[t!]
	\centering
	\includegraphics[width=0.87\linewidth]{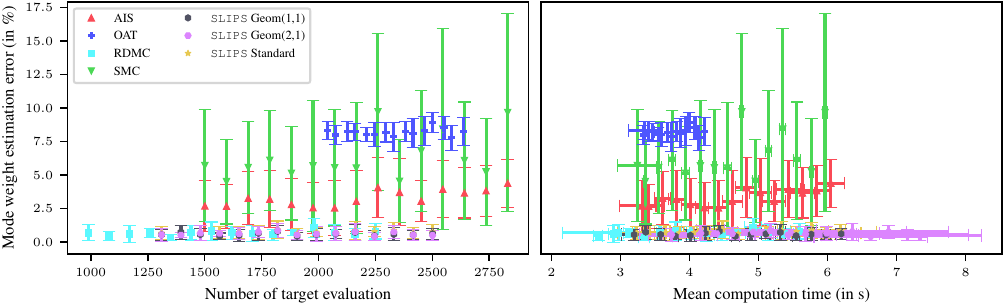}
	\hfill
	\includegraphics[width=0.87\linewidth]{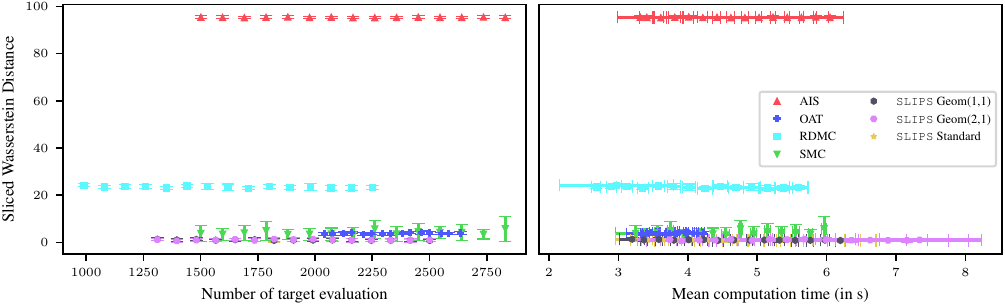}
	\vspace{-0.2cm}
	\caption{Metrics when sampling from the Gaussian mixture defined in \Cref{sec:xps} ($d=4$) using different algorithms. \textbf{Top}: Weight estimation error. \textbf{Bottom}: Sliced Wasserstein Distance. \textbf{Left}: Metric depending on the number of target density evaluations. \textbf{Right}: Metric depending on wall time. The computational budgets are computed by evolving $K$ linearly from $20$ to $90$. The number of MCMC steps is fixed at 16. The computations were run on the same Nvidia V100 GPU. The computational budgets were aligned to match $\SLIPS$'s.}
	\label{fig:tde_dim_4}
\end{figure}

\end{document}